\documentclass[journal]{IEEEtran}

\usepackage{xspace,amsmath,amssymb,amsthm,amsfonts,dsfont,pifont}
\usepackage{bm, url, textcomp, enumitem}
\usepackage{algorithmicx,algorithm,algpseudocode}
\usepackage{stmaryrd}
\usepackage{empheq,float}
\usepackage{graphicx,subfigure,balance}
\usepackage{multirow}
\usepackage{float,tcolorbox}
\usepackage{xcolor}
\usepackage{yhmath}
\usepackage{tabulary}
\usepackage{colortbl}
\usepackage{multicol,multirow}
\usepackage{lipsum}

\usepackage{minted}
\usepackage{listings}
\lstdefinestyle{py}{
  language=Python,
  basicstyle=\ttfamily\small,
  keywordstyle=\bfseries,
  commentstyle=\itshape,
  stringstyle=\ttfamily,
  showstringspaces=false,
  numbers=left,
  numberstyle=\tiny,
  stepnumber=1,
  numbersep=8pt,
  frame=single,
  breaklines=true,
  tabsize=4
}

\newtheorem{proposition}{Proposition}

\newtheorem{example}{Example}

\newtheorem{theorem}{Theorem}

\theoremstyle{definition}

\DeclareMathOperator*{\argmin}{arg\,min}
\graphicspath{{./figs/}}

\newcommand{\x}{\mathbf{x}}
\newcommand{\y}{\mathbf{y}}

\newcommand{\I}{\mathbf{I}}
\newcommand{\X}{\mathbf{X}}
\newcommand{\Y}{\mathbf{Y}}
\newcommand{\A}{\mathbf{A}}
\newcommand{\U}{\mathbf{U}}
\newcommand{\V}{\mathbf{V}}

\newcommand{\Q}{\mathbf{Q}}
\newcommand{\B}{\mathbf{B}}
\newcommand{\M}{\mathbf{M}}

\newcommand{\W}{\mathbf{W}}

\newcommand{\bfP}{\mathbf{P}}
\newcommand{\E}{\mathbf{E}}
\newcommand{\F}{\mathbf{F}}

\newcommand{\bR}{\mathbb{R}}

\newcommand{\bfS}{\mathbf{S}}
\newcommand{\bfPhi}{\mathbf{\Phi}}
\newcommand{\bfPsi}{\mathbf{\Psi}}
\newcommand{\bfSigma}{\mathbf{\Sigma}}
\newcommand{\bfTheta}{\mathbf{\Theta}}

\newcommand{\bfa}{\mathbf{a}}
\newcommand{\bfb}{\mathbf{b}}

\newcommand{\G}{\mathbf{G}}
\newcommand{\m}{\boldsymbol{m}}
\DeclareMathOperator{\diag}{diag}
\newcommand{\rank}{\mathsf{rank}}

\newcommand{\tr}{\mathsf{Tr}}
\newcommand{\fro}{\mathsf{F}}
\newcommand{\st}{\mathsf{St}}
\newcommand{\Rgrad}{\mathsf{grad}}

\usepackage{titletoc}

\usepackage{makecell}
\usepackage{microtype}
\usepackage{graphicx}
\usepackage{subfigure}
\usepackage{booktabs} 
\usepackage[colorlinks,
            linkcolor=blue,
            anchorcolor=blue,
            citecolor=teal,
            ]{hyperref}

\title{Low-Rank Adaptation Redux for Large Models
}

\author{Bingcong Li, \IEEEmembership{Member,~IEEE}, Yilang Zhang, Georgios B. Giannakis,~\IEEEmembership{Fellow,~IEEE}
\thanks{Research in this paper was supported in part by the NSF grants 2212318, 2220292, and 2312547; B. Li is now with the Dept. of CS, ETH Z{\"u}rich, Switzerland; G. B. Giannakis and Y. Zhang are with the Dept. of ECE, U. of Minnesota, USA; Y. Zhang is now with Morgan Stanley, USA. Emails: bingcong.li@inf.ethz.ch, yilang.zhang@morganstanley.com, georgios@umn.edu}
\thanks{\today}}

\vspace{-0.1cm}

\begin{document}

\markboth{IEEE TRANSACTIONS ON SIGNAL PROCESSING (submitted)}{}
\maketitle

\begin{abstract}
Low-rank adaptation (LoRA) has emerged as the de facto standard for parameter-efficient fine-tuning (PEFT) of foundation models, enabling the adaptation of billion-parameter networks with minimal computational and memory overhead. Despite its empirical success and rapid proliferation of variants, it remains elusive which architectural choices, optimization techniques, and deployment constraints should guide practical method selection.
This overview revisits LoRA through the lens of signal processing (SP), bridging modern adapter designs with classical low-rank modeling tools and inverse problems, as well as highlighting how SP principles can inform principled advances of fine-tuning approaches. 
Rather than providing a comprehensive enumeration and empirical comparisons of LoRA variants, emphasis is placed on the technical mechanisms underpinning these approaches to justify their effectiveness. 
These advances are categorized into three complementary axes: architectural design, efficient optimization, and pertinent  applications. The first axis builds on singular value decomposition (SVD)-based factorization, rank-augmentation constructions, and cross-layer tensorization, while the second axis deals with initialization, alternating solvers, gauge-invariant optimization, and parameterization-aware methods. 
Beyond fine-tuning, emerging applications of LoRA are accounted across the entire lifecycle of large models, ranging from pre- and post-training to serving/deployment. 
Finally, open research directions are outlined at the confluence of SP and deep learning to catalyze a bidirectional frontier: classical SP tools provide a principled vocabulary for designing principled PEFT methods, while the unique challenges facing modern deep learning, especially the overwhelming scale and prohibitive overhead, also offer new research lines benefiting the SP community in return.
\end{abstract}

\begin{IEEEkeywords}
	low rank, LLMs, fine-tuning, optimization
\end{IEEEkeywords}

\vspace{-0.2cm}
\section{Introduction}
Recent advances of large language models (LLMs) in artificial intelligence (AI) have unlocked unprecedented capabilities across a broad spectrum of fields. 
State-of-the-art models such as GPT~\cite{GPT1,GPT2,GPT3}, LLaMA~\cite{llama, llama2, llama3}, Gemma~\cite{gemma,team2024gemma,team2025gemma}, and QWen~\cite{qwen,qwen3,qwen3max} involve billions of parameters. There are even larger ones such as Qwen3-Max~\cite{qwen3max} that have surpassed the trillion-parameter milestone, enabling complicated tasks including code generation, reasoning, and autonomous agents. 
Pre-training to learn model parameters at this scale typically requires sustained large-cluster computation and vast Internet-scale corpora, which remain accessible only to a small number of organizations.

A more accessible way to leverage LLMs is to start from an open-weight pre-trained model (e.g., via Hugging Face) and fine-tune it on user- or domain-specific data \cite{peft}. This approach preserves the broad capabilities of LLMs learned during pre-training, while adapting the model to user data with altered distributions.
However, even on moderate-size datasets, full fine-tuning of LLMs with billions of parameters is still formidable, given the considerable resources required.
In practice, fine-tuning a large model by updating all model parameters entails exaFLOPs ($\times 10^{18}$) of \emph{computation} and hundreds of GPU hours, incurring non-negligible financial and energy costs. This makes full fine-tuning infeasible for individual users or small institutes. 
\emph{Memory} is an equally critical bottleneck. A widely used rule of thumb is about 16 GB of GPU RAM per billion parameters in half precision~\cite{modal2024_vram_finetuning}. 
Using this estimate, fine-tuning an 8B-parameter model demands $\sim$128 GB of memory, far exceeding the capacity of standard commodity hardware. 
Scaling to models with over 100B parameters pushes the requirement to the terabyte scale, necessitating distributed training across large GPU clusters.

\emph{Parameter-efficient fine-tuning} (PEFT)~\cite{houlsby2019,li2021prefix,hu2021lora,peft} offers a compute- and memory-frugal alternative to full fine-tuning. Instead of updating all pre-trained weights, PEFT freezes the backbone and optimizes only a small set of additional parameters. Among these methods, \emph{low-rank adaptation} (LoRA)~\cite{hu2021lora} has gained widespread prominence due to its simplicity and strong empirical performance. Concretely, LoRA augments the parameters over layer with an additive low-rank update $\X\Y^\top$, where $\X$ and $\Y$ are small trainable (so-termed ``adapter'') matrices. This can reduce the number of trainable parameters by over $100\times$ relative to full fine-tuning, substantially lowering both compute and memory costs. Adapter matrices are lightweight, often accounting for $\sim 1\% - 2\%$ of the entire model, and can be stored, shared, and swapped across tasks and users with minimal overhead. This modularity has further accelerated adoption in open-source communities; large models such as LLaMA-7B can now be fine-tuned by individual users on a single consumer GPU of 24 GB memory~\cite{zhao2024galore}.
Consequently, LoRA has become the standard approach for fine-tuning large models~\cite{peft}, spawning dozens of variants that further impel efficiency or alleviate limitations. 

LoRA and its variants have been adopted across a wide range of foundation-model applications, from LLM-based coding and mathematical reasoning~\cite{hu2021lora} to diffusion-based image generation~\cite{gu2023mix}, as well as multi-modal systems that jointly handle language, images, video, and audio~\cite{abouelenin2025phi}. This ubiquity has in turn catalyzed LoRA’s expansion beyond standard downstream fine-tuning. In particular, recent advances have designed efficient LoRA variants across the \emph{full lifecycle} of foundation models, that includes pre-training, post-training, and deployment from the providers to serve the users~\cite{lialin2023relora,schulman2025lora,dettmers2024qlora,sheng2024slora}.

However, the rapid proliferation of the literature has made it difficult to discern what first-principles guidelines should inform the design and application of fast, efficient, and practical LoRA algorithms across diverse real-world settings, especially for researchers new to this field.
One natural bridge is to revisit these questions through the lens where low-rank methodologies have been systematically studied and applied for decades. Classical tools in signal processing (SP) such as principal component analysis (PCA)~\cite{pearson1901,Robust-PCA}, subspace tracking~\cite{low-rank-subspace}, and matrix/tensor decompositions~\cite{tensor-decomp} demonstrate the power of leveraging low-dimensional structure hidden within high-dimensional data. Not only offering computational savings and statistical robustness, these tools also underpin applications ranging from communications, imaging sciences to graph learning; see e.g.,~\cite{mardani2015subspace,tensor-decomp,sidiropoulos2017tensor}. Cross-fertilizing the SP community’s expertise in low-rank modeling to develop more principled LoRA methods offers an impactful direction.

This survey recaps recent LoRA advances through an SP lens.
Similar to SP, where progress is often driven by coupling structured parameterizations with algorithmic insights, we group LoRA variants in two broad categories depending on: \textit{architecture (model) designs}, and \textit{tailored optimization algorithms}. We will start with a brief recapitulation of low-rank modeling and relevant SP tools, which provide a principled vocabulary for interpreting LoRA variants and suggesting novel designs. Next, we review representative LoRA methods,  emphasizing methodological insights over numerical results, since results are often difficult to compare fairly across heterogeneous models, datasets, and evaluation protocols. Our goal is to highlight underlying ideas transferred across settings without being overly anchored to specific benchmarks. Finally, we will outline emerging applications of LoRA beyond standard fine-tuning, that span serving, pre- and post-training. Our survey will be concluded with open directions in cross-fertilizing SP and LoRA research.
Beyond the unique SP angle, our survey also differs from existing reviews in several aspects. While~\cite{han2024parameter,wang2025parameter} overview PEFT, we focus on LoRA to enable detailed treatment of its design space. Compared to~\cite{AI-SPedu}, which bridges AI and SP through LLM-facilitated education, we connect the two communities from a low-rank perspective. Unlike~\cite{mao2025survey} which catalogs LoRA variants, we emphasize efficient mechanisms, especially from optimization and geometry perspectives, so that our discussion leans more toward “why” rather than “what.”

The remainder of this survey is organized as follows. Section~\ref{Sec.recap-SP} recaps low-rank approaches in SP. Section~\ref{Sec.LoRA} provides an overview for LoRA, and the link between standard LoRA and low-rank matrix sensing. Subsequently, Section~\ref{Sec.architectures} discusses efficient architectural model parameterizations beyond the Burer-Monteiro factorization used in vanilla LoRA, followed by efficient optimization approaches in Section~\ref{Sec.opt}. The joint designs of architecture and optimization schemes will lead to high-performance fine-tuning approaches. Section~\ref{Sec.applications} deals with applications of LoRA, ranging from fine-tuning, pre-training, and serving, to broader use cases in diffusion and multi-modal models. The survey is concluded with future directions that not only bring classical SP tools to LoRA, but also bounce LoRA ideas back to the SP community. 

\textbf{Notation}. Bold lowercase, uppercase, and calligraphic letters (e.g., $\mathbf{a}$, $\A$, $\bm{\mathcal{A}}$) denote column vectors, matrices, and tensors, respectively; 
$(\cdot)^\top$ and $\| \cdot \|_\fro$ refer to transposition and Frobenius norm of a matrix; $\| \cdot \|$ is the $\ell_2$ (spectrum) norm of a vector (matrix); $\rank()$ denotes the rank of a matrix, while $\odot$ and $\otimes$ stand for Hadamard and Kronecker products.

\section{Low-rank in SP}
\label{Sec.recap-SP}

Low-rank learning has been a cornerstone of the SP research for decades, long predating the rise of modern deep learning~\cite{pearson1901,lanczos1950}. Given the ubiquity of low-rank structure in real-world problems, SP has documented success leveraging the low-rank feature across a broad spectrum of applications, underpinned by interpretable modeling assumptions, mature statistical theory, and efficient implementations with clear computational benefits. This success has also produced a rich algorithmic toolbox. For convex objective functions, nuclear-norm minimization and its variants provide principled formulations with rigorous recovery guarantees~\cite{recht2010guaranteed}. For nonconvex objectives, factor models based on Burer–Monteiro (BM) parameterizations~\cite{burer2003}, combined with alternating  schemes~\cite{jain2013low}, (projected) gradient descent~\cite{chen2015fast,hardt2014fast,ma2018implicit}, and Riemannian optimization on low-rank manifolds~\cite{mishra2014fixed,keshavan2010matrix}, enable scalability in large-scale settings. Due to space limitations, we do not attempt a comprehensive review of this extensive literature, and instead provide a brief account for context of the LoRA methods.

\subsection{Learning low-rank subspaces} 
Subspace methods explain high-dimensional observations using a small number of latent components. For example, principal component analysis (PCA)~\cite{pearson1901,candes2011robust,Robust-PCA} formalizes the idea that a data ensemble concentrates near a low-dimensional linear subspace, and is closely related to the Karhunen–Lo\`eve (KL) expansion viewpoint for stochastic processes. This perspective has also been extended via kernel methods to incorporate nonlinearity~\cite{scholkopf1997kernel}. Canonical correlation analysis (CCA)~\cite{hotelling1992relations} extracts maximally correlated low-dimensional representations from two or more data sets, whereas independent component analysis (ICA)~\cite{comon1994independent,hyvarinen2000independent} identifies statistically independent latent sources. 
Dictionary learning~\cite{olshausen1996emergence} generalizes fixed-basis subspace by adaptively learning an overcomplete set of atoms from data, with which each observation can be expressed as a sparse linear combination of these atoms. Going beyond a single global subspace, subspace clustering addresses the problem of grouping data that lie in a union of multiple low-dimensional subspaces. Multidimensional scaling (MDS)~\cite{young2013multidimensional} embeds data points into a low-dimensional Euclidean space while preserving pairwise distance or dissimilarity. In contrast, local linear embedding (LLE)~\cite{tenenbaum2000global} performs nonlinear dimensionality reduction by preserving local geometry, which represents each data point as a weighted combination of its neighbors, and optimizes the subspace to maintain these reconstruction weights. 
These subspace-based techniques have impacted a wide range of applications, including face recognition~\cite{turk1991eigenfaces}, data visualization~\cite{baingana2014embedding}, multi-view~\cite{multiview-CCA} and multi-modal learning~\cite{andrew2013deep}, and blind source separation~\cite{blind-source-separation}.

Recovering a ground-truth signal from its limited and noisy linear measurements is one of the central themes in SP. When the unknown signal is a low-rank matrix $\A \in \bR^{m \times n}$, these inverse problems are known as \emph{matrix sensing} under general linear measurements~\cite{candes2006near}, or low-rank \emph{matrix completion} when the measurements are partially observed entries~\cite{recht2010guaranteed}. When the two factor matrices are non-negative, the problem is termed non-negative matrix factorization (NFM)~\cite{lee1999learning}.
Such low-rank structures are intrinsic to many real-world applications, and have been extensively exploited across areas such as phase retrieval~\cite{candes2013phaselift}, quantum state tomography~\cite{gross2010quantum}, image inpainting~\cite{jin2015annihilating}, recommender systems~\cite{koren2009matrix}, sensor network localization~\cite{ding2008sensor}, and wireless communications~\cite{kalogerias2013matrix}. 
While low rank can be accounted via convex relaxations such as nuclear-norm~\cite{recht2010guaranteed}), a scalable approach is to rely on the BM factorization~\cite{burer2003}, which optimizes over two smaller factor matrices $\X \in \bR^{m\times r}$ and $\Y \in \bR^{n\times r}$ such that $\X\Y^\top \approx \A$.
The associated factorization task has been studied from multiple perspectives, including the loss function landscape~\cite{ge2016matrix,ge2017no}, efficient algorithms with convergence guarantees~\cite{recht2010guaranteed,chi2019nonconvex}, and statistical properties such as identifiability and finite-sample error rates~\cite{recht2010guaranteed, negahban2011estimation}. It will be argued later that LoRA is structurally isomorphic to BM factorization. Thus, the rich literature on matrix sensing and completion offers a powerful analytical basis for understanding the effectiveness of LoRA.

\subsection{Learning low-rank tensors}

Another SP perspective, where low-rank structure plays a central role, is tensor modeling. While tensor algebra generalizes many familiar matrix concepts, it also introduces important differences. For instance, determining tensor rank is NP-hard in general, and unlike the matrix case, a best low-rank approximation of a higher-rank tensor may not exist. These characteristics are closely tied to the fact that there are multiple definitions of tensor rank, together with richer identifiability of factor models and more intricate optimization approaches. For comprehensive background, the reader is referred  to relevant SP tutorials such as~\cite{sidiropoulos2017tensor,kolda2009tensor}; nevertheless, the present survey assumes no extensive prior knowledge of tensors.

Two of the most widely adopted low-rank tensor representations are canonical polyadic (CP), and Tucker decompositions. The former was introduced in~\cite{hitchcock1927expression,hitchcock1928multiple}, and later rediscovered under abbreviated names such as PARAFAC, CANDECOMP, or CP form~\cite{kolda2009tensor}. It expresses a tensor as a sum of rank-one tensors, extending the familiar rank-one expansion of matrices to high-order data arrays (slabs); the formal expression will be provided later when discussing tensorized adapters. Alternatively, Tucker decompositions~\cite{tucker1964extension,de2000multilinear} can be viewed as a high-order generalization of the matrix singular value decomposition (SVD), where a tensor is represented with mode-wise factor matrices together with a smaller core tensor.

Low-rank tensor models have been widely adopted in both SP and machine learning, with applications to coding in telecommunications~\cite{sidiropoulos2002blind}, community detection in social networks~\cite{cichocki2015tensor}, and topic modeling~\cite{anandkumar2014tensor}. Classical matrix recovery tasks such as sensing and completion have also been generalized to tensor settings, addressing a range of higher-order inverse problems; see, e.g.,~\cite{kanatsoulis2019tensor,zhang2016exact,gandy2011tensor}. Methodologically, tensor low-rank structures can be promoted via convex surrogates in a manner analogous to the matrix case~\cite{gandy2011tensor}, but scalable methods typically optimize over explicit factorized representations such as CP and Tucker decompositions. The resultant nonconvex optimization solvers leverage the loss function geometry to establish statistical guarantees~\cite{dong2023fast,ge2021understanding}. Besides low rank, SP modeling approaches also account for sparsity features; see for example~\cite{PCP,Robust-PCA,low-rank-sparse}. 

Given that LoRA also imposes low-rank structure in the parameter space of LLMs for efficient fine-tuning, it naturally shares a common foundation with these SP methods. 
From an SP viewpoint, however, a central distinction is that classical SP problems are often concrete and well-defined, whereas LLMs remain considerably more ambiguous given our still limited theoretical and empirical understanding. 
Prompted by the link as well as the gap, the present survey re-examines the PEFT variants to distill a self-contained account that bridges these two fields. 
Our goal is twofold: (i) bring SP insights to design principled and well-justified fine-tuning methods for LLMs; and (ii) cross-fertilize ideas developed in modern deep learning to inspire innovative approaches to low-rank SP problems.

\section{Low-rank in fine-tuning LLMs}
\label{Sec.LoRA}
This section reviews LoRA, compares it with other PEFT methods, and highlights open-source codebases that implement LoRA and its variants. Our goal is a streamlined guide for uninitiated readers. We also demonstrate that classical matrix sensing can be viewed as LoRA applied to a single linear layer, which offers a simple testbed to bridge the two communities.

\subsection{LoRA recap}
The majority of LLM parameterizations comprise   weight matrices of their linear layers. Taking GPT-3 as an example, the model involves multiple Transformer blocks, each containing 6 linear mappings with associated weight matrices: the query $\mathbf{Q}$, key $\mathbf{K}$, value $\mathbf{V}$, and output $\mathbf{O}$ pertaining to the (self-)attention module, and two feed-forward network (FFN) weights $\mathbf{F}_1$ and $\mathbf{F}_2$ (also known as the up- and down-projection)~\cite{GPT3}; see also Figure~\ref{fig.lora}. 
Full fine-tuning updates all six matrices, whose dimensions are huge, typically in the order of millions each. Given the sheer scale of modern LLMs, this approach incurs prohibitive memory and computational costs in resource-limited settings, which scale at least linearly with model size\footnote{Exact dependence also relies on hyper-parameters such as sequence length.}. 

Motivated by the empirical observations that adapting within a low-rank subspace can retain most of the performance of full fine-tuning~\cite{hu2021lora}, LoRA freezes the pre-trained linear layers, and postulates that the task-specific weight updates are approximately low-rank. For notational simplicity, we do not distinguish individual layer types, but instead represent a pre-trained weight matrix by $\W_l \in \bR^{m_l \times n_l}$, where $l$ indexes the linear layers across Transformer blocks. With $\Delta \W_l$ denoting the additive update to the given $\W_l$, fine-tuning under this constraint of low-rank $r$ can be expressed as
\begin{equation}\label{eq.problem-low-rank}
\begin{aligned}
	\min_{\{\Delta \W_l\}_{l=1}^L} & ~~~f(\{\W_l + \Delta \W_l \}_{l=1}^L; \mathcal{D}^\mathrm{ft}) \\ 
    \text{s.t.}~~~ & \rank( \Delta \W_l) \leq r, \quad l=1,\ldots,L
\end{aligned}
\end{equation}
where $f$ is the chosen loss function over the fine-tuning data $\mathcal{D}^\mathrm{ft}$ of downstream tasks. Loss $f$ and the low-rank constraint in~\eqref{eq.problem-low-rank} are nonconvex.
Following low-rank approaches in SP, LoRA achieves parameter efficiency by adopting the BM factorization that sets $\Delta \W_l = \X_l\Y_l^\top$, where $\X_l \in \bR^{m_l \times r}$ and $\Y_l \in \bR^{n_l \times r}$~\cite{burer2003}.  For brevity, the range of $l$ and the dataset $\mathcal{D}^\mathrm{ft}$ will be omitted in the following. 
This simple reparameterization represents an $m_ln_l$-entry matrix using only $(m_l+n_l)r$ trainable parameters. Under the low-rank  $r\ll \min\{m_l,n_l\}$, the parameter count is drastically reduced. An additional practical benefit is that the rank constraint in~\eqref{eq.problem-low-rank} is implicitly enforced by construction, and yields the unconstrained optimization problem
\begin{equation}\label{eq.prob-lora}
	\min_{\X_l,\Y_l} f(\{\W_l + \X_l \Y_l^\top \}_l)
\end{equation}
which can be directly optimized using standard deep learning algorithms; see also Figure~\ref{fig.lora} for a graphical illustration.
It is also worth noting that the original LoRA in~\cite{hu2021lora} adapts merely the query and value matrices in the attention module, whereas subsequent variants often treat the choice of target layers as a flexible, user-specified design option. 

\begin{figure}[t]
	\centering
	\includegraphics[width=0.48\textwidth]{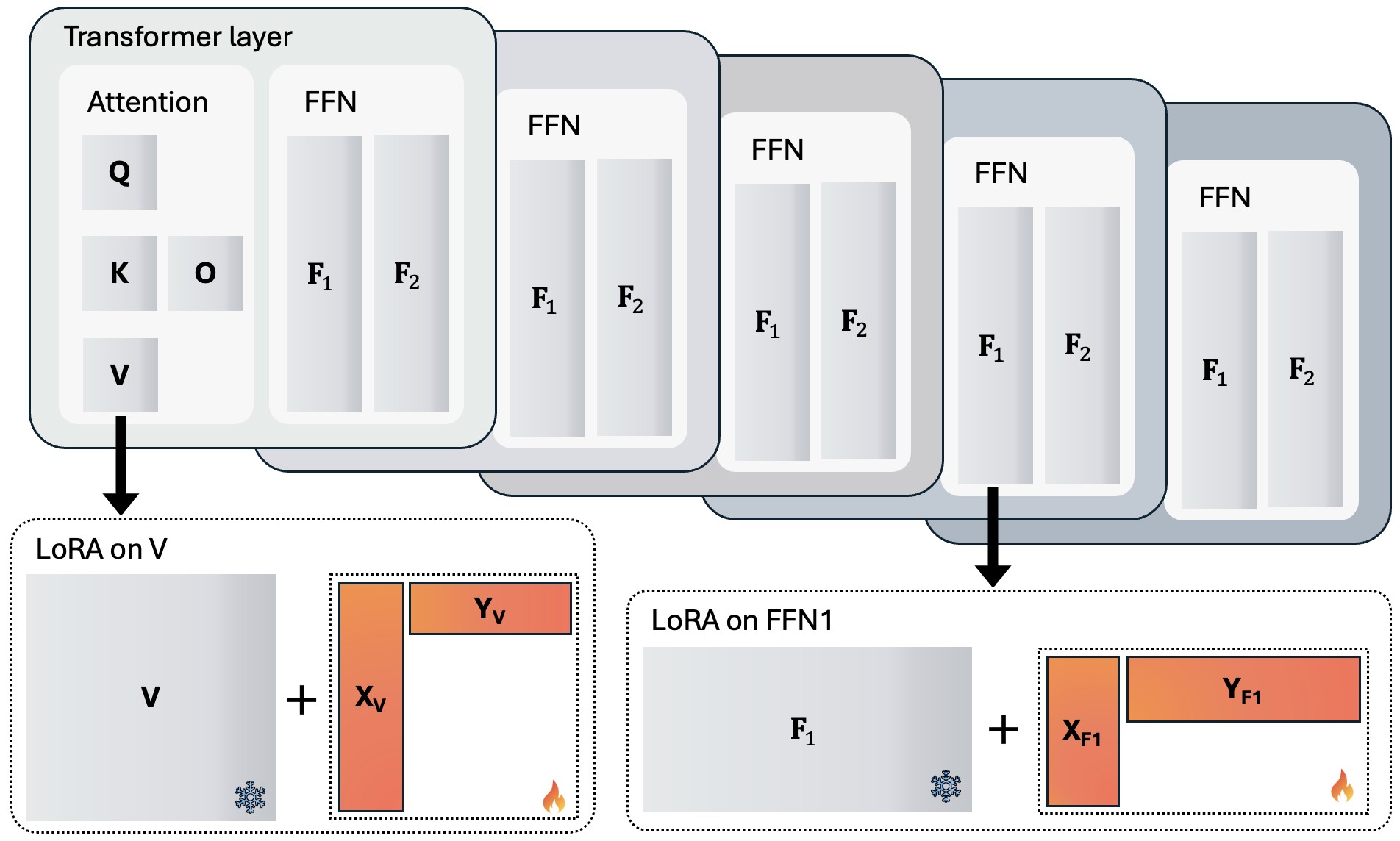}
	\caption{LoRA fine-tuning of a GPT-3 model. Grey and orange boxes are respectively frozen (snowflake icon) weights of linear layers, and trainable (fire icon) LoRA weights.}
	 \label{fig.lora}
\end{figure}

LoRA can markedly lower the required resources for fine-tuning, while only incurring a small performance drop, or even outperforming full fine-tuning when the latter overfits the dataset. 
For example,~\cite{hu2021lora} reports that fine-tuning RoBERTa-Large (335M) requires only about 0.8M LoRA parameters, and fine-tuning GPT-3 175B requires 37.7M LoRA parameters, yet achieving performance comparable to (or slightly better than) full fine-tuning on GLUE and on WikiSQL/MNLI, respectively. For GPT-3 175B, LoRA reduces the reported GPU memory requirement from 1.2 TB to 350 GB. 
More recent advances further push the limits of resource efficiency. For instance,~\cite{dettmers2024qlora} demonstrates that even a 65B-parameter model can be fine-tuned on a single 48GB GPU, substantially lowering the hardware barrier for individuals and small labs to adapt reasonably large-scale models. Beyond resource efficiency, LoRA has also enabled a wide range of applications across diverse domains, as elaborated later in~Section~\ref{Sec.applications}.

\begin{figure*}[t]


\centering
	\includegraphics[width=0.98\textwidth]{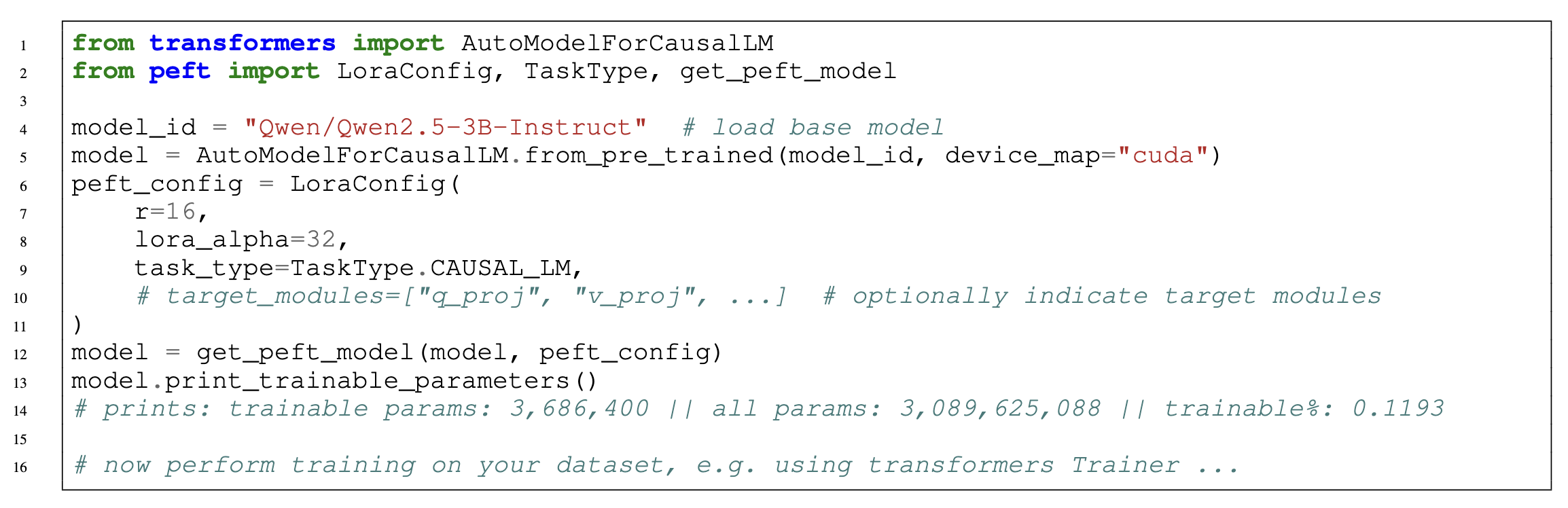}
\caption{Example code snippet of using LoRA through PEFT~\cite{peft} codebase.}
\label{fig.code}
\end{figure*}

\subsection{Comparison with other PEFT methods}

Beyond LoRA, several PEFT alternatives have emerged. The natural questions are: how does LoRA compare to these methods, and what accounts for its popularity in practice?

LLMs rely on an autoregressive process to generate responses conditioned on a given prompt. Building on this, prompt and prefix tuning adapt models by learning a small set of task-specific “soft prompts”~\cite{lester2021power,li2021prefix}. Prompt tuning prepends learnable prompt vectors to the input, while prefix tuning injects learnable prefix key/value vectors into attention calculation, steering the model’s behavior without modifying the pre-trained weights. However, because these methods are equivalent to leveraging additional tokens, they can introduce additional inference overhead, e.g., longer context and extra KV-cache. 
In contrast, LoRA introduces \emph{no extra inference} overhead when serving a single task, since the low-rank updates can be merged into the model weights prior to deployment. 
Moreover, LoRA does not rely on the autoregressive prompting interface and is therefore applicable to \emph{non-autoregressive} settings such as language-diffusion models~\cite{nie2025large}. 
Finally, prompt- and prefix-based methods typically do not directly adapt FFN layers, which are often conjectured to store factual knowledge, and are commonly targeted in knowledge editing~\cite{fang2024alphaedit}.

\begin{table*}[t]
\centering
\caption{Comparison of parameter-efficient fine-tuning approaches.}
\renewcommand{\arraystretch}{1.25}
\begin{tabular}{ccccc}
\toprule
\textbf{Approach} & \textbf{Memory} & \textbf{Performance on simple tasks} & \textbf{Performance on challenging tasks} & \textbf{Extendability} \\
 &  & \textbf{(e.g., GLUE)} & \textbf{(e.g., math, coding\dots)} &  \\
\midrule
Full parameter FT & High & Strong & Strong & Low \\
Adapter           & Moderate & Strong & Moderate & Moderate \\
BitFit            & Very low & Moderate & Weak or unknown & Moderate \\
Prefix tuning     & Low & Strong & Moderate & Moderate \\
{\textbf{LoRA}} 
                  & {\textbf{Low}} 
                  & {\textbf{Strong}} 
                  & {\textbf{Moderate -- Strong}} 
                  & {\textbf{High}} \\
\bottomrule
\end{tabular}
\label{tab:peft_comparison}
\end{table*}

Adapter-based fine-tuning~\cite{houlsby2019} and zeroth-order optimization methods (often referred to as MeZO)~\cite{malladi2023fine,zhang2024revisiting,zhang2023dpzero} are two other popular approaches to memory-efficient adaptation of LLMs. Adapter methods insert small trainable modules into an otherwise frozen backbone~\cite{houlsby2019}. At the extreme end of parameter efficiency, BitFit~\cite{zaken2022bitfit} freezes all weights and fine-tunes only the bias terms, which minimizes the number of trainable parameters but can also limit expressiveness. 

Alternatively, MeZO targets extreme memory efficiency by estimating updates without storing backpropagation activations, requiring negligible additional memory beyond the backbone model.
Compared with these approaches, LoRA offers a desirable accuracy–efficiency tradeoff, and flexibility with respect to available computing and memory. As resources increase, one can simply increase the LoRA rank $r$ to allocate more expressive capacity to the adapters. A full comparison can be found in Table \ref{tab:peft_comparison}.

\subsection{LoRA codebase}

Thanks to the open-source community, LoRA and its variants have been integrated into multiple software packages. Early libraries such as \texttt{AdapterHub} (now \texttt{adapters})\footnote{\url{https://github.com/adapter-hub/adapters}} and HuggingFace \texttt{PEFT}~\cite{peft} played a key role in making LoRA accessible, and are also being updated with recent LoRA variants. Today, LoRA is deeply integrated into the HuggingFace ecosystem, with strong support for both LLMs via \texttt{transformers} and diffusion models through \texttt{diffusers}. Tools such as \texttt{Unsloth}\footnote{\url{https://github.com/unslothai/unsloth}} and \texttt{bitsandbytes}\footnote{\url{https://github.com/bitsandbytes-foundation/bitsandbytes}} further provide GPU-friendly optimizers that make LoRA-style fine-tuning faster and more memory-efficient. Meanwhile, \texttt{vLLM}~\cite{kwon2023efficient} and \texttt{SGLang}\footnote{\url{https://github.com/sgl-project/sglang}} offer efficient inference and deployment support for LoRA-adapted models. 
Recently, Thinking Machines Lab’s \texttt{Tinker}\footnote{\url{https://thinkingmachines.ai/tinker/}} has introduced LoRA-based fine-tuning as a managed API for open-weight models, and has also showcased LoRA's effectiveness in reinforcement learning workflows for LLMs.

To illustrate how simple it is to use LoRA, we modify an example from the PEFT library~\cite{peft} in Figure~\ref{fig.code}. In particular, enabling LoRA requires only two additional commands (lines 7–13), highlighting how the open-source ecosystem has substantially lowered the barrier to entry for beginners.

\subsection{LoRA on a single linear layer}
\label{Sec.LoRA-single-layer}
Applying the BM factorization to eliminate the explicit rank constraint in~\eqref{eq.problem-low-rank} hints to a link of LoRA with  low-rank SP techniques. Next, we unveil this link by showing that the low-rank matrix sensing problem can be viewed as the simplest instance of LoRA fine-tuning. 

Recall that matrix sensing recovers an unknown signal matrix $\W_o \in \bR^{m \times n}$ from a collection of $N$ data (or measurements) $ {\cal D} := \{(\M_i, y_i)\}_{i=1}^N$, where $\M_i \in \bR^{m \times n}$ is a sensing matrix that yields observation $y_i = \tr(\M_i^\top \W_o)$. 
With $\y := [y_1, \ldots, y_N]^\top \in \bR^N$, and $\m(\W) := [\tr(\M_1^\top \W), \ldots, \tr(\M_N^\top \W)]^\top$,
signal sensing is obtained as 
\begin{equation}\label{eq:matrix-sensing}
    \hat{\W}_o = \argmin_{\W \in \bR^{m \times n}} \| \y - \m(\W) \|^2.
\end{equation}

Now revisit sensing through the lens of pre-training and fine-tuning. Suppose that the ground-truth matrix to learn during pre-training is $\W_o^{\text{pre}}$, and that the model has been fit to global optimality with $\W^{\text{pre}}=\W_o^{\text{pre}}$. In the subsequent fine-tuning stage, let $\W_o^{\text{ft}}=\W_o^{\text{pre}}+\Delta\W$ denote the perturbed new ground truth, where the perturbation $\Delta\A$ is constrained to be low-rank. Given fine-tuning data ${\cal D}^{\text{ft}}$, and viewing $\W^{\text{pre}}$ as a frozen linear layer, applying LoRA to this layer leads to
\begin{equation}\label{eq.sensing}
    \min_{\X\in\bR^{m\times r}, \Y\in\bR^{n\times r}} \frac{1}{2}\|\y^{\text{ft}} - \m(\W^{\text{pre}} + \X\Y^\top)\|^2
\end{equation}
where $\y^{\text{ft}} = \m(\W_o^{\text{ft}})$.
Upon defining $\y := \y^{\text{ft}} - \m (\W^{\text{pre}})$ $ = \m (\Delta \W_o)$, the linearity of $\m$ simplifies~\eqref{eq.sensing} to
\begin{equation}\label{eq.prob-sensing}
    \min_{\X\in\bR^{m\times r}, \Y\in\bR^{n\times r}} \frac{1}{2}\| \y - \m( \X\Y^\top) \|^2
\end{equation}
which is exactly a low-rank matrix sensing loss written in its BM-factorized form. Although this single-layer reduction is far simpler than fine-tuning a deep LLM, it already provides a clean analytical testbed along with several useful insights.

First, it offers a straightforward sanity check on the expressiveness of a LoRA variant: among the many proposed designs, it is reassuring if a method can solve the basic one-layer problem~\eqref{eq.prob-sensing} efficiently. That said, this should not be interpreted as a strict requirement, but rather an informative indicator of potential tradeoffs. Some variants are designed with different priorities in mind, such as faster training, or escalated hardware efficiency, thus may intentionally trade expressiveness for advantages in specific regimes.
Second, this problem also provides an analytically tractable setup for understanding how optimization algorithms cope with the nonconvex, globally nonsmooth loss function landscape and rich gauge invariance (see Section~\ref{Sec.gauge-invariance}) induced by LoRA's bilinear structure. Since tracking the full dynamics of LoRA fine-tuning in LLMs remains largely open, this simplified problem offers a useful proxy that retains these core optimization challenges while remaining amenable to analysis. Indeed, several methods originally developed for solving~\eqref{eq.prob-sensing} have proven effective in LoRA fine-tuning; cf. Sections~\ref{Sec.init} and~\ref{Sec.AltGD}.

\section{Model Architectures of LoRA}
\label{Sec.architectures}

\begin{figure}[t]
	\centering
	\includegraphics[width=0.48\textwidth]{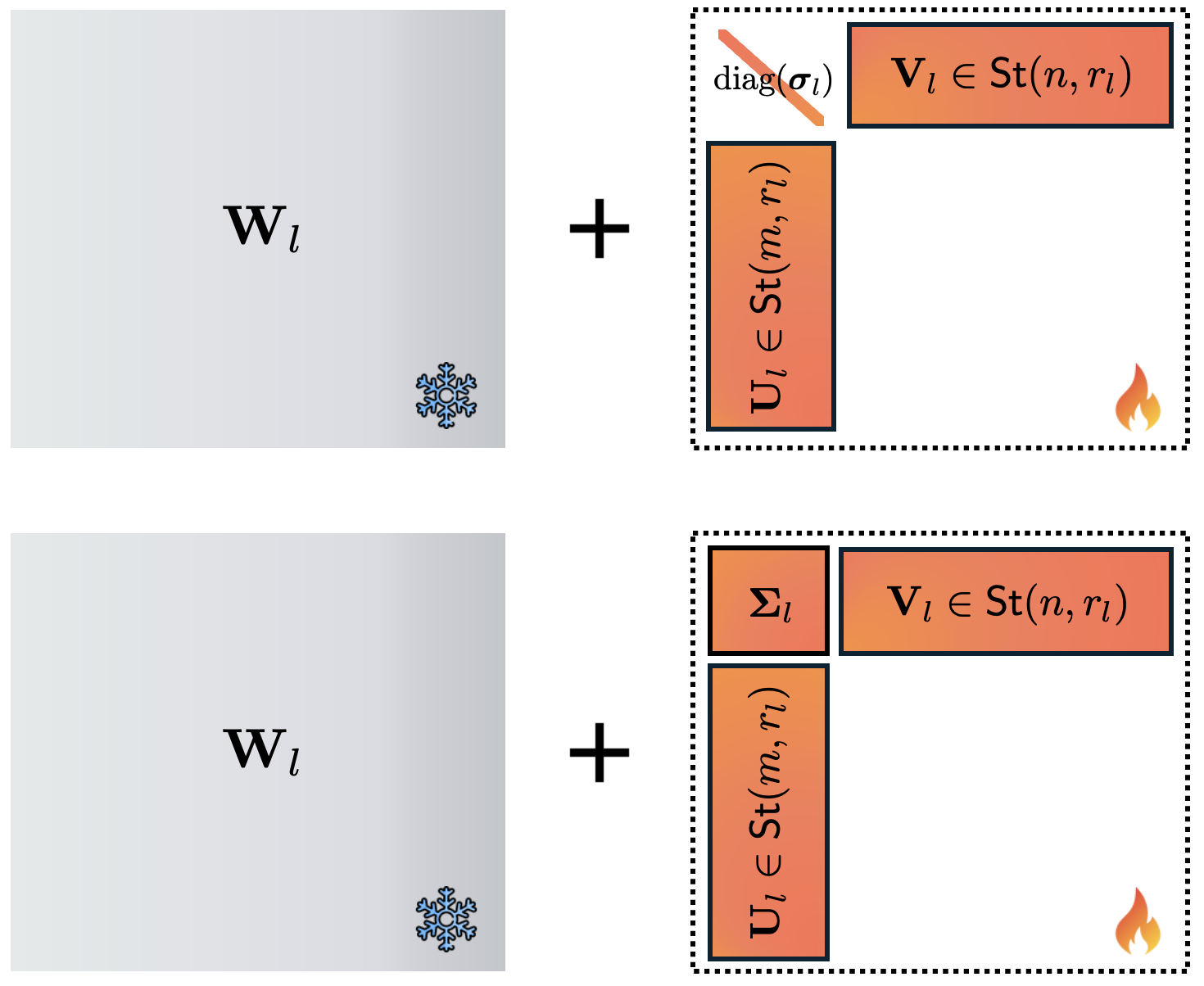}
	\caption{SVD factorization in AdaLoRA~\cite{zhang2023adaptive} (top) and PoLAR~\cite{kai2025} (bottom), where Stiefel manifold $\st(m,r):= \{ \U \in \bR^{m \times r} \mid \U^\top \U  = \I_r\}$.}
	 \label{fig.svd}
\end{figure}

Having outlined low-rank tools from SP, attention now turns to fine-tuning methods that are transferable from SP and those developed more recently. The focus is first fine-tuning itself, with broader applications deferred to later sections. Akin to conventional SP problems such as~\eqref{eq.prob-sensing}, efficiency of a fine-tuning method often relies on the interaction between i) the adapter model architecture, and ii) a paired optimization algorithm that leverages the architecture. For this reason, the ensuing section begins with architectural (model) designs beyond the BM factorization that is the cornerstone of LoRA. 

\subsection{SVD-based parameterization}
\label{Sec.SVD-parameterization}
Another parameter-efficient alternative for parameterizing $\Delta \W_l \in \bR^{m\times n}$ is the factorization provided by SVD
\begin{equation}\label{eq.svd}
    \Delta \W_l=\U_l\bfSigma_l\V_l^\top
\end{equation}
where $\U_l\in\bR^{m\times r}$ and $\V_l\in\bR^{n\times r}$ have orthonormal columns; that is, $\U_l^\top\U_l=\I_r$ and $\V_l^\top\V_l=\I_r$, or satisfy relaxed versions thereof. Relaxing the orthogonality constraints will turn out to simplify modeling and enable GPU-friendly optimization. The middle factor $\bfSigma_l\in\bR^{r\times r}$ can be chosen to be diagonal or left unconstrained depending on the target settings; these choices will become explicit shortly.

\begin{figure*}[t]
	\centering
	\begin{tabular}{ccc}
		\hspace{-0.2cm}
		\includegraphics[width=.31\textwidth]{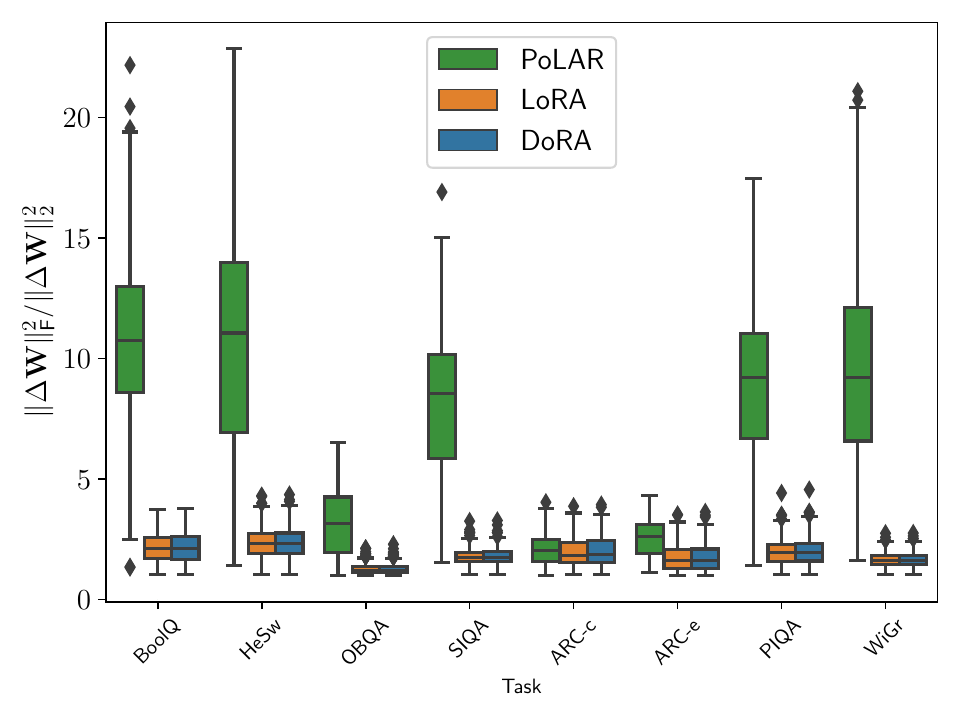}&
		\hspace{-0.2cm}
		\includegraphics[width=.32\textwidth]{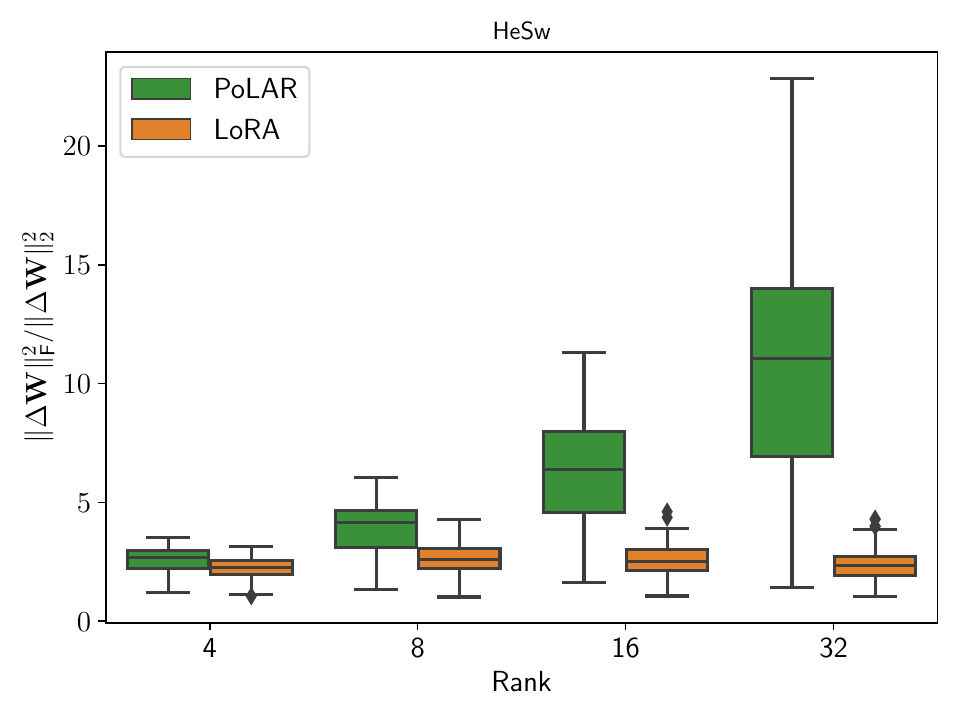}&
		\hspace{-0.2cm}
		\includegraphics[width=.32\textwidth]{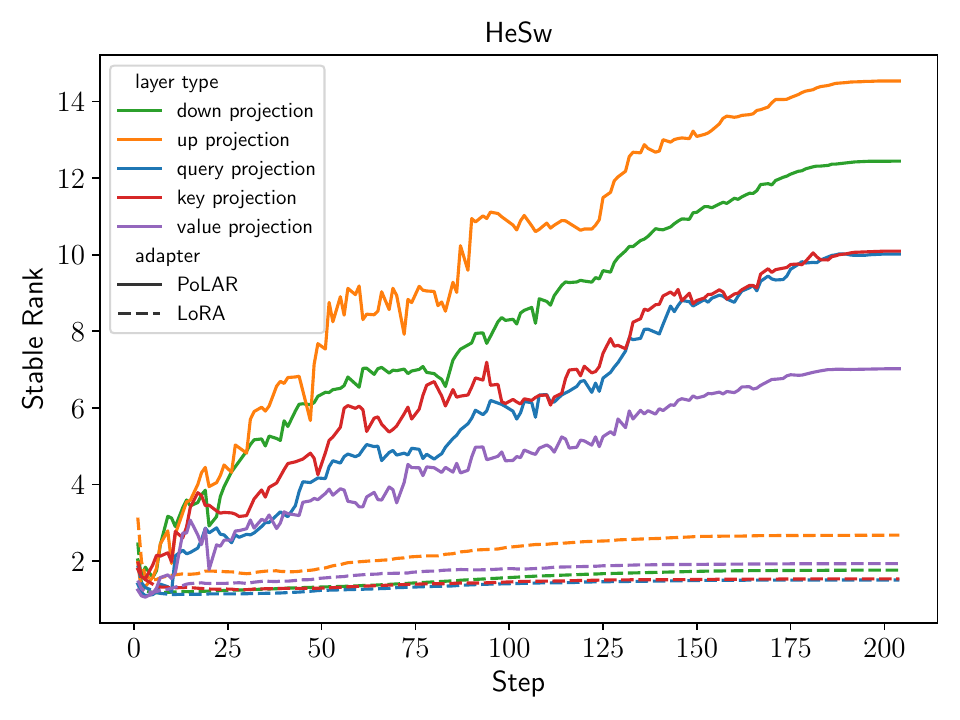}
		\\ 
	   (a) Stable rank across different datasets & (b) Stable versus algebraic rank & (c) Stable rank versus iteration
	 \end{tabular}
	\vspace{-0.2cm}
	\caption{(a) PoLAR parameterization facilitates better utilization of rank. (b)(c) Fine-tuning a LLaMA2-7B model on the HellaSwag dataset~\cite{zellers2019}.}
	\vspace{-0.3cm}
	 \label{fig.polar}
\end{figure*}

The motivation behind an SVD-type model, for the same expressiveness as BM factorization, is two-fold. First, SVD facilitates explicit growth or pruning of singular directions, and thus naturally supports task-adaptive, layer-wise rank allocation~\cite{zhang2023adaptive}.
Second, orthogonality in $\U_l$ and $\V_l$ promotes effective utilization of allocated parameters, especially when the number of optimization steps is limited~\cite{kai2025}.
A minor limitation of this parameterization is the slightly increased trainable parameters. Using a full $\bfSigma_l\in\bR^{r\times r}$ introduces $r^2$ additional parameters compared to standard LoRA (or only $r$ additional parameters if diagonal). In practice, this overhead is negligible as commonly used ranks are sufficiently small, e.g., $r=8$ or $32$. A graphical illustration of these parameterizations can be found in Figure~\ref{fig.svd}.

\textbf{Dynamic rank allocation.} 
Standard LoRA allocates the same rank budget $r$ to every layer of an LLM, yet not all layers contribute equally to downstream tasks. It thus motivates dynamic, data-dependent rank allocation across layers. For a diagonal $r_l\times r_l$ matrix $\bfSigma_l$, SVD reduces to
\[
\U_l \bfSigma_l \V_l^\top = \sum_{s=1}^{r_l} \sigma_l^s \mathbf{u}_l^s (\mathbf{v}_l^{s})^\top
\]
where $\sigma_l^s$ is the $s$-th diagonal entry of $\bfSigma_l$, and $\mathbf{u}_l^s$ and $\mathbf{v}_l^s$ are the $s$-th columns of $\U_l$ and $\V_l$, respectively. 
This admits a natural interpretation: $\sigma_l^s$ serves as an explicit importance score for the $s$-th rank-one summand $\mathbf{u}_l^s (\mathbf{v}_l^{s})^\top$. Consequently, this allows for pruning the less important components $\sigma_l^s \mathbf{u}_l^s (\mathbf{v}_l^{s})^\top$ when $\sigma_l^s$ falls below a certain threshold.
Moreover, the removal of individual subspaces becomes fully decoupled, making it possible to prune rank in a flexible manner. 
AdaLoRA~\cite{zhang2023adaptive} leverages this idea with two modifications. First, instead of enforcing exact orthogonality, it encourages approximately orthogonal $\U$ and $\V$ by adding soft penalty terms $\|\U^\top\U-\I \|_\fro^2$ and $\|\V^\top\V-\I \|_\fro^2$ to the loss. The resulting unconstrained formulation is often easier to optimize in standard deep-learning pipelines. 
Second, the importance score is enriched by incorporating its sensitivity and uncertainty of subspaces, thereby making the allocation mechanism more robust to the stochasticity of training.
Concretely, for a layer $l$ with effective rank $r_l^\star$, AdaLoRA starts from an over-parameterized rank $r>r_l^\star$, and progressively prunes less important subspaces every few fine-tuning steps, thereby producing a task-dependent rank allocation across layers.
As opposed to progressive pruning, IncreLoRA~\cite{zhang2023increlora} instead grows the singular space incrementally during fine-tuning.
Further, it is worth stressing that dynamic-rank methods are not limited to SVD-type.
SoRA~\cite{ding2023sparse} removes the orthogonality constraints on $\U$ and $\V$. DyLoRA~\cite{valipour2023dylora} proposes an analogous layer-wise rank allocation strategy under the BM form by having $\X_l\Y_l^\top=\sum_{s=1}^{r_l}\mathbf{x}_l^s(\mathbf{y}_l^{s})^\top$. However, the absence of orthogonality allows different rank-one components to interact through correlated directions. This entangles the contributions of each component, making the notion of ``importance'' ambiguous and the associated pruning criterion delicate to design. 

\textbf{Improving parameter efficiency.} 
For the SVD-type parameterization~\eqref{eq.svd}, PoLAR~\cite{kai2025} and SteLLA~\cite{li2025stella} impose respectively approximate and strict orthogonality on $\U_l$ and $\V_l$, both with an \emph{unconstrained} $\bfSigma_l$.
This unconstrained $\bfSigma_l$ can be viewed as a generalized form of weight normalization~\cite{salimans2016weight,wei2025benefits} induced by the BM factorization. Specifically, applying the polar decomposition to vanilla LoRA renders $\X_l=\U_l\bfSigma_l^1$, where $\U_l^\top\U_l=\I_r$ encodes the “direction,” and $\bfSigma_l^1\in\bR^{r\times r}$ is a positive semi-definite (PSD) matrix representing the “magnitude,” which is analogous to weight normalization. Similarly, $\Y_l=\V_l\bfSigma_l^2$ with $\V_l^\top\V_l=\I_r$ and PSD $\bfSigma_l^2$. This SVD-type model with unconstrained $\bfSigma_l$ can be then obtained by
rewriting the BM factorization as $\X_l\Y_l^\top=\U_l(\bfSigma_l^1\bfSigma_l^{2\top})\V_l^\top$, and merging the two magnitude matrices into a single one $\bfSigma_l:=(\bfSigma_l^1\bfSigma_l^{2\top})$. This clarifies why $\bfSigma_l$ is unconstrained: product of two PSD matrices need not be PSD. 

To handle the orthogonality constraints $\U_l^\top\U_l=\I_r$ and $\V_l^\top\V_l=\I_r$, SteLLA uses retraction-based Riemannian optimization (discussed in Section V). 
PoLAR instead adopts the Landing algorithm~\cite{ablin2022landing}, which adds a penalty term $\|\U_l^\top\U_l-\I_r\|_\fro^2$ (and similar for $\V_l$) to avoid the SVD operation required by retraction, thus improving the compute efficiency. 
When combined with these optimization methods, the SVD-type parameterization enables effective utilization of the allocated subspace.
Figure~\ref{fig.polar}(a) reports the \emph{stable rank}\footnote{Defined as $\|\Delta \W_l \|^2/\|\Delta \W_l\|_\fro^2$, the stable rank is a smooth surrogate of the algebraic rank that downweighs small singular values.} of ${\Delta \W_l}$ across layers of a LLaMA2-7B model fine-tuned on commonsense reasoning tasks~\cite{liu2024dora}. 
Despite using $r=32$, the stable rank is close to $1$, which demonstrates that LoRA often exploits a rank-one subspace, and its expressiveness is not fully explored.
In contrast, PoLAR achieves substantially higher stable ranks. Figure~\ref{fig.polar}(b) plots the stable rank of PoLAR and LoRA against the algebraic rank $r$, where the stable rank of the former increases with $r$. Figure~\ref{fig.polar}(c) further illustrates the fine-tuning dynamics. It is seen that PoLAR consistently maintains a larger stable rank than LoRA. These results indicate that for the same expressiveness, PoLAR more effectively utilizes the available rank budget than LoRA in practice.

\subsection{Rank-augmented parameterization} 
\label{Sec.high-rank}
\begin{figure}[t]
	\centering
	\includegraphics[width=0.48\textwidth]{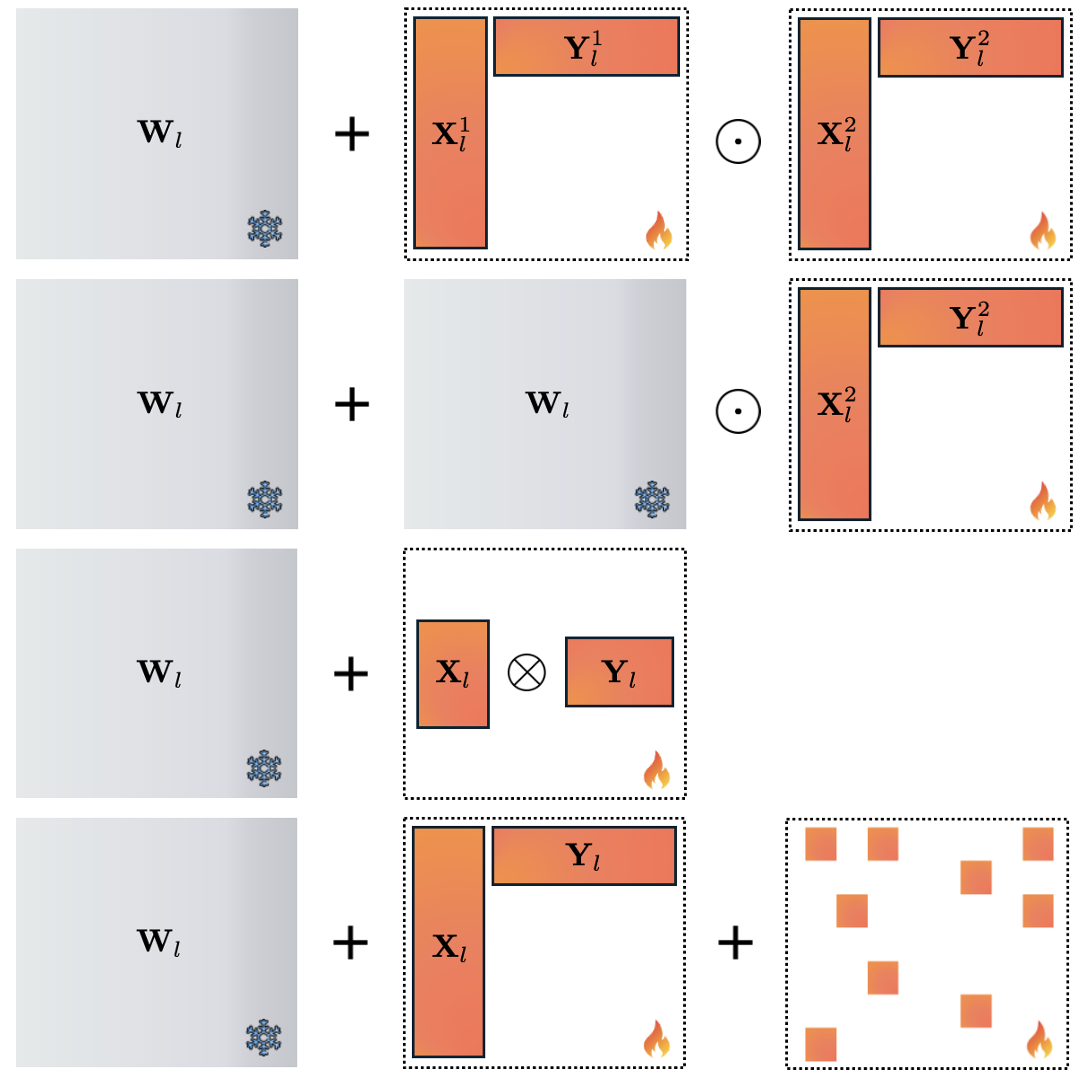}
	\caption{Rank-augmented parameterization. From up to down are FedPara~\cite{hyeon2021fedpara}, HiRA~\cite{huang2025hira}, KronA~\cite{edalati2025krona}, and RoSA~\cite{nikdan2024rosa}. }
	 \label{fig.rank-aug}
\end{figure}

While BM and SVD factorizations are efficient, their expressiveness is limited to a rank-$r$ matrix. This naturally raises the question of whether one can represent higher-rank $\Delta \W_l$ without increasing the number of trainable parameters.

\textbf{Hadamard based parameterization.} A means of increasing the rank is to leverage the  Hadamard product whose rank is submultiplicative~\cite{styan1973hadamard}. 
\begin{theorem}[\cite{million2007hadamard}]
\label{thm.Hadamard-rank}
    The Hadamard product of $\A,\B\in\bR^{m\times n}$, is 
    \[
    \A \odot \B :=
    \begin{bmatrix}
        a_{11}b_{11} & \cdots & a_{1n}b_{1n}\\
        \vdots & \ddots & \vdots\\
        a_{m1}b_{m1} & \cdots & a_{mn}b_{mn}
    \end{bmatrix}
    \]
    and its rank satisfies
    $     \; \rank(\A \odot \B) \leq \rank(\A) \rank(\B).
    $
\end{theorem}
In particular, the Hadamard product of two rank-$r$ matrices can have rank as large as $r^2$. Motivated by this observation, FedPara~\cite{hyeon2021fedpara} (or LoHA~\cite{yeh2024navigating}) has additive weight update
\begin{equation}\label{eq.fedpara}
    \Delta \W_l = (\X_l^1\Y_l^{1\top}) \odot (\X_l^2\Y_l^{2\top}) 
\end{equation}
where $\X_l^1, \X_l^2 \in \bR^{m \times r}$ and $\Y_l^1, \Y_l^2 \in \bR^{n \times r}$ are learnable LoRA adapters. 
The next proposition characterizes the expressiveness of Hadamard-based parameterization. 

\begin{proposition}[Expressiveness of FedPara
~\cite{hyeon2021fedpara}] 
\label{prop.Hadamard-express}
The parameterization $\Delta \W = (\X^1\Y^{1\top}) \odot (\X^2\Y^{2\top}) $ can express (i)  
\underline{all} matrices with rank up to $r$; and, (ii) a \underline{subset} of matrices having rank $r < \rank(\Delta \W) \le r^2$. 
\end{proposition}

A naive implementation of~\eqref{eq.fedpara} would first learn the two dense $m\times n$ matrices $\X_l^1(\Y_l^1)^\top$ and $\X_l^2(\Y_l^2)^\top$ (in both forward and backward passes), and then take their Hadamard product. 
However, this incurs an $\mathcal{O}(mn)$ memory footprint for the intermediates. A memory-efficient implementation was pointed out by~\cite{singhal2025abba}, which avoids forming any $m\times n$ intermediate via a row-wise Khatri–Rao product~\cite{khatri1968solutions}. Specifically, for $\A,\B\in\bR^{m\times n}$ written in rows
\[
    \mathbf \A=
    \begin{bmatrix}
        \bfa_1^\top\\ 
        \bfa_2^\top \\ 
        \cdots \\
        \bfa_m^\top
    \end{bmatrix}
    \quad
    \mathbf B=
    \begin{bmatrix}
        \bfb_1^\top \\
        \bfb_2^\top \\
        \cdots \\
        \bfb_m^\top
    \end{bmatrix},
\]
the row-wise Khatri–Rao product $\star$ is defined as
\[
\mathbf A \star \mathbf B
    :=
    \begin{bmatrix}
        \bfa_1^\top \otimes \bfb_1^\top \\
        \bfa_2^\top \otimes \bfb_2^\top \\
        \cdots \\
        \bfa_m^\top \otimes \bfb_m^\top 
    \end{bmatrix} \in \bR^{m \times n^2}
\]
where $\otimes$ denotes the Kronecker product
\[
    \A \otimes \B =
    \begin{bmatrix}
    a_{11}\B & a_{12}\B & \cdots & a_{1n}\B\\
    a_{21}\B & a_{22}\B & \cdots & a_{2n}\B\\
    \vdots   & \vdots   & \ddots & \vdots\\
    a_{m1}\B & a_{m2}\B & \cdots & a_{mn}\B
    \end{bmatrix}.
\]
With this notation,the update~\eqref{eq.fedpara} can be rewritten as
\begin{align*}
    \Delta \W_l
    &= (\X_l^1\Y_l^{1\top})\odot(\X_l^2\Y_l^{2\top}) \\
    &= \underbrace{(\X_l^1 \star \X_l^2)}_{m\times r^2}
        \underbrace{(\Y_l^1 \star \Y_l^2)^\top}_{r^2\times n}. 
\end{align*}
As a consequence, the forward/backward calculation can be implemented using two “thin” factors of sizes $m\times r^2$ and $n\times r^2$, thus avoiding the sizable $m\times n$ matrix. 

While FedPara is more expressive, it also poses additional optimization challenges as the adapter becomes “deeper,” involving four learnable matrices as opposed to two in standard LoRA. In general, deeper parameterizations are harder to optimize, and can exhibit slower training dynamics~\cite{shamir2019exponential}.

HiRA\footnote{A similar construction can be found in~\cite{wen2023batched} yet for different purposes; cf. Section~\ref{Sec.serving}.}~\cite{huang2025hira} aims to endow $\Delta \W_l$ with a higher rank by tying one branch of the Hadamard product to the frozen pre-trained weight, which may itself have high rank:
\begin{equation}\label{eq.hira}
    \Delta \W_l = \W_l \odot (\X_l\Y_l^\top).
\end{equation}
The expressiveness of this parameterization is presented in the next proposition. 
\begin{proposition}[Expressiveness of HiRA]
\label{prop.HiRA-express}
Given $\W$, the parameterization $\Delta \W = \W \odot (\X\Y^\top) $ can express 
\begin{enumerate}
    \item a \underline{subset} of matrices of rank up to $r$, and 
    \item a \underline{subset} of matrices whose rank is $r < \rank(\Delta \W) \le \min\{ \rank(\W) \cdot r, m, n \}$. 
\end{enumerate}
\end{proposition}
A caveat however, is that  expressiveness depends critically on $\W_l$. For $m=n$ and $\W_l=\I_m$, for example, HiRA can represent solely diagonal matrices. 

\textbf{Kronecker based parameterization.} While the rank of a Hadamard product is submultiplicative, the rank of a Kronecker product is \textit{multiplicative}, as asserted next. 
\begin{theorem}[\cite{van2000ubiquitous}]
\label{thm.Kron-rank}
    For matrices $\A \in \bR^{d_1 \times d_2}$ and $\B \in \bR^{d_3 \times d_4}$, the $d_1d_3 \times d_2 d_4$ Kronecker product matrix has rank satisfying
    $\rank(\A \otimes \B) = \rank(\A)\rank(\B)$.
    \end{theorem}
This property is exploited in KronA (also known as LoKr)~\cite{edalati2025krona,yeh2024navigating}. The adapter matrix is parameterized via
\[
    \Delta \W_l = \X_l \otimes \Y_l
\]
where $\X_l \in \bR^{d_1 \times d_2}$ and $\Y_l \in \bR^{d_3 \times d_4}$. KronA requires the pre-trained weight $\W_l \in \bR^{m \times n}$ satisfying the so-called shape constraints $m = d_1 d_4$ and $n = d_2 d_3$. This mild requirement is typically compatible with common pre-trained model weight sizes. The rank of $\Delta \W_l$ can be as large as $\min\{d_1, d_2\} \times \min\{d_3, d_4\}$. 
But as asserted next, not every matrix of desirable size can be represented as a single Kronecker product. 
\begin{proposition}
\label{prop.Kron-express}
For matrices $\X,\Y$ of rank at most $r$,  parameterization $\Delta \W = \X \otimes \Y$ can express a \underline{subset} of matrices of rank up to $r^2$. 
\end{proposition}
A known remedy is to rely on an ensemble (weighted sum) of Kronecker factors. This idea is leveraged in~\cite{sadeghi2025moka} with
\[
    \Delta \W_l = \sum_{k=1}^K \alpha_l^k (\X_l^k \otimes \Y_l^{k})
\]
to improve KronA's expressiveness. 
The coefficients $\alpha_l^k$ are learned from data alongside the factors $\X_l^k,\Y_l^k$ during fine-tuning, though they may equivalently be absorbed into \(\X_l^k\).

\textbf{Low-rank plus sparsity parameterization.}
Leveraging low-rank and sparsity has well-documented impact across SP areas, including robust PCA~\cite{Robust-PCA, candes2011robust}. 
The upshot of such properties is expressiveness:  superposition of low-rank plus sparse matrix components can represent a broad family of matrices, including full-rank ones.
A simple example for this rank augmentation is a properly scaled identity matrix as the sparse component; adding this to a low-rank component yields a full-rank matrix.
Building on this idea, RoSA~\cite{nikdan2024rosa} augments LoRA with a sparse matrix $\bfS_l \in \bR^{m \times n}$ via
\[
    \Delta \W_l =  \X_l\Y_l^\top + \bfS_l.
\]
While robust PCA optimizes the sparse term $\bfS_l$ via an $\ell_1$ regularizer of the vectorized matrix
\[
    \min_{\X_l \in \bR^{m \times r}, \Y_l \in \bR^{n \times r}, \bfS_l \in \bR^{m \times n}}  f(\X_l\Y_l^\top + \bfS_l) + \lambda \| \text{vec}(\bfS_l) \|_1
\]
this formulation does not account for parameter efficiency. During training, learning the sparsity pattern end-to-end still requires viewing $\bfS_l$ as a dense $m\times n$ matrix, which is prohibitively expensive to store at the scale of modern LLMs. 
Many practical approaches sidestep this issue by fixing the sparsity pattern a priori and optimizing only the nonzero values~\cite{nikdan2024rosa,bhardwaj2024rapid}. This is a justifiable design choice: a random sparse matrix can be full-rank, provided its density is above certain threshold, as summarized in the following result.
\begin{theorem}[\cite{litvak2022singularity}]
\label{thm.Bernoulli-rank}
An $n\times n$ matrix $\A$ with i.i.d. Bernoulli-$(p)$ entries has full rank with high probability $p \gtrsim \frac{\log n}{n}$.
\end{theorem}
Consequently, this allows for representing full-rank matrices. 
\begin{proposition}
\label{prop.low-rank+sparse}
Parameterization $\Delta \W = \X\Y^\top + \bfS$ can express
\begin{enumerate}
    \item \underline{all} matrices of rank up to $r$, and 
    \item a \underline{subset} of matrices with rank $r < \rank (\W) \le \min\{ m ,n \}$. 
\end{enumerate}
\end{proposition}
Recent works have also pursued parameterizations relying purely on sparsity; see e.g., SHiRA~\cite{bhardwaj2024rapid} that adopts
\[
    \Delta \W = \W_0 \odot \bfS
\]
where $\bfS\in\bR^{m\times n}$ is a fixed sparse mask. This sparsity pattern is a design choice and can be sampled randomly or based on importance heuristics. To collect sparsity benefits on hardware, one can design specialized compute kernels\footnote{Here “kernels” refers to GPU compute kernels (i.e., codes), not e.g., Gaussian/RBF kernels.} that can effectively leverage the sparse engines available in recent GPUs. 

\subsection{Tensorized parameterization}
\label{Sec.tensorized}

\begin{figure}[t]
	\centering
	\begin{tabular}{cc}
		\hspace{-0.2cm}
		\includegraphics[width=.25\textwidth]{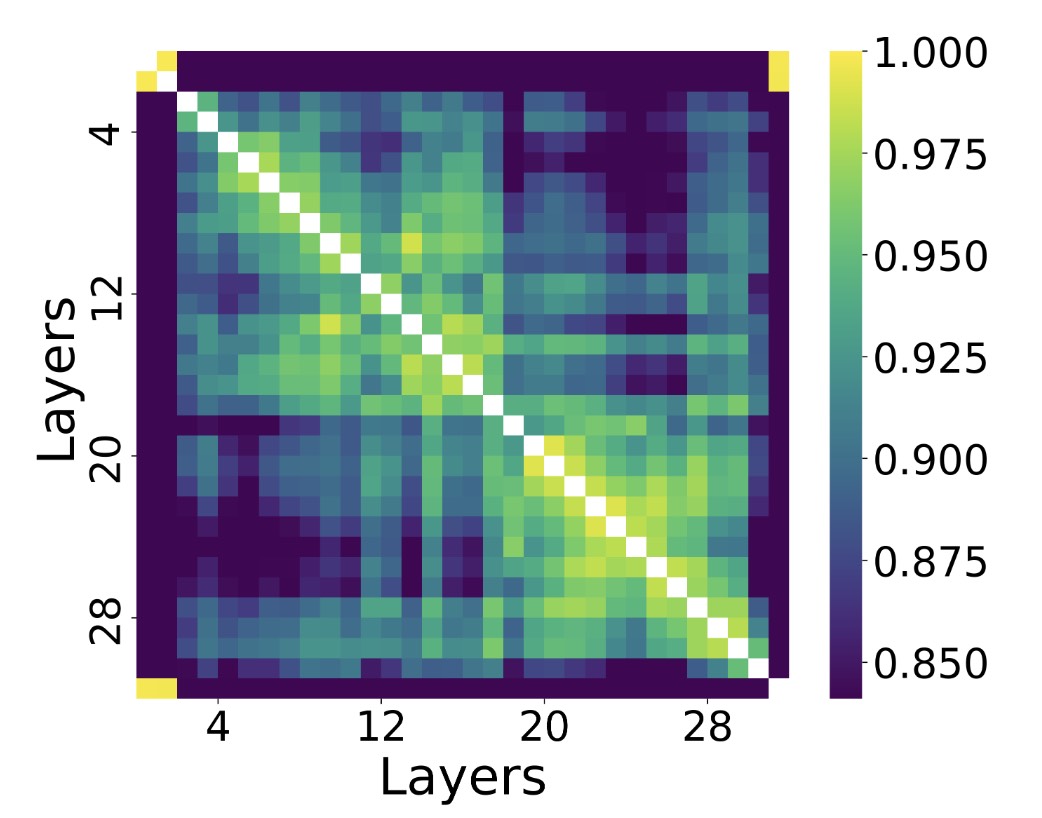}&
		\hspace{-0.2cm}
		\includegraphics[width=.24\textwidth]{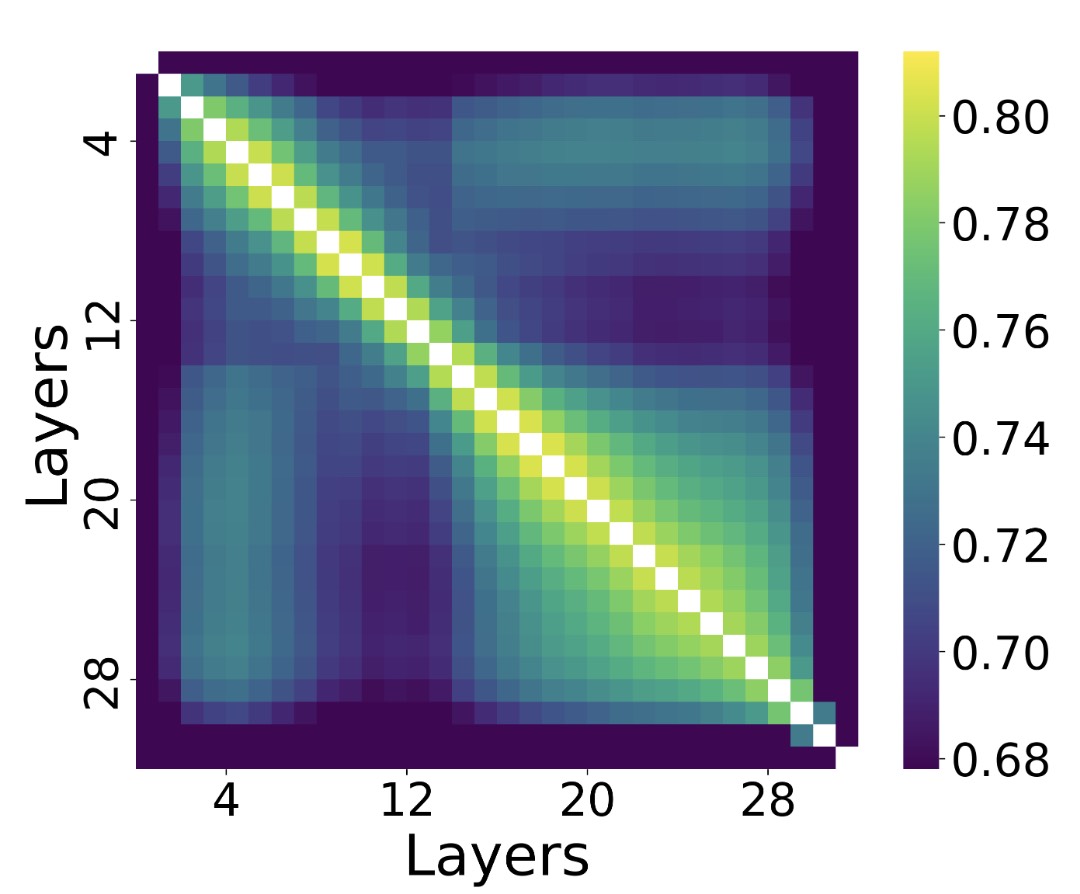}
		\\ 
	   (a) Representation similarity & (b) Weight similarity
	 \end{tabular}
	\vspace{-0.2cm}
	\caption{Visualization of correlated layers in Llama 3.1-8B-Instruct using linear centered kernel alignment (CKA) \cite{kornblith2019similarity}.
    (a) Representation similarity before the feedforward network. (b) Weight similarity of the MLP (up projection). This figure is taken from \cite{min2025docs}.}
	\vspace{-0.3cm}
	 \label{fig.layer-sim}
\end{figure}

While LoRA and the aforementioned variants improve efficiency via layer-wise low-rank updates, they typically treat adapters in different layers as unrelated matrices, and do not account for potential correlations across layers. 
However, recent studies suggest that layers in deep networks can be related, and often exhibit substantial redundancy or shared structure~\cite{min2025docs,lan2019albert,csordas2025language}; see also Figure \ref{fig.layer-sim} for an exmaple in Llama-3-Instruct.
Overlooking these inter-layer dependencies may leave parameter efficiency on the table. 

Parameters of tensor models offer a prudent approach to capturing such topological connections. Consider a collection of weight matrices with identical dimensions $\{\W_l \in \bR^{m\times n}\}_{l=1}^L$, e.g., the query projection matrices in attention blocks. During fine-tuning, their additive updates $\{\Delta \W_l \in \bR^{m\times n}\}_{l=1}^L$ can be stacked to form a three-way tensor $\Delta \bm{\mathcal{W}} \in \bR^{m \times n \times L}$. 
Notably, these weights can be reshaped into higher-order tensor models in various ways. This subsection will focus on third-order tensors for clarity, while the analysis can be readily extended to higher-order tensors. 
If $\Delta \bm{\mathcal{W}}$ exhibits low tensor rank under an appropriate definition\footnote{Computing tensor rank is generally NP-hard.}, it is natural to adopt tensor factorization schemes to parameterize the updates compactly, thereby capturing the relation across layers while reducing the number of trainable parameters.
The total parameter count can drop further from the LoRA-style scaling of ${\cal O}((m+n)rL)$ to ${\cal O}((m+n+L)r)$, where the dependence on the $L$ layers becomes additive rather than multiplicative. An overview of the tensor-based parameterization is depicted in Figure~\ref{fig.tensor}.

\begin{figure}[t]
	\centering
	\includegraphics[width=0.48\textwidth]{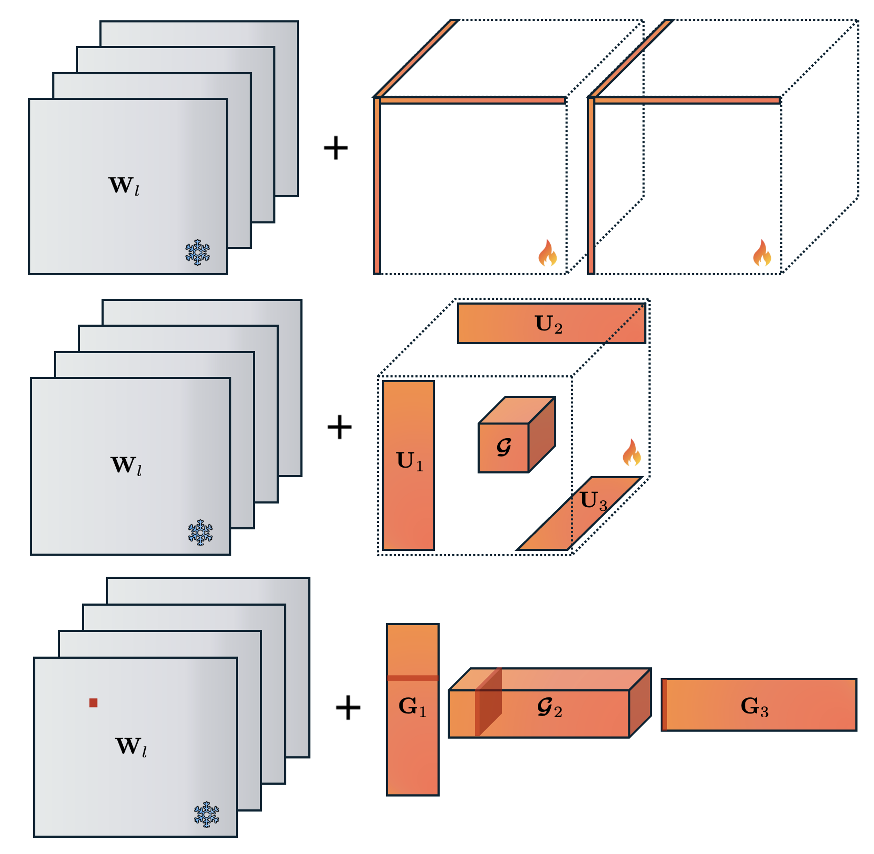}
	\caption{LoRA with~\cite{hounie2024lorta}, Tucker~\cite{bershatsky2024lotr}, and TT~\cite{yang2024loretta} parameterizations.}
	 \label{fig.tensor}
\end{figure}

\textbf{Canonical-Polyadic (CP)-based parameterization.} 
CP decomposition is a popular tensor factorization model that generalizes the notion of matrix rank to higher-order tensors. It represents an $N$-th order tensor as the superposition of rank-one summands, given by  the outer product (denoted by $\circ$) of $N$ vectors $\mathbf{\mathcal{A}} := \mathbf{s}_1 \circ \ldots \circ \mathbf{s}_N$, with $(i_1, \ldots, i_N)$-th entry $\mathbf{\mathcal{A}}[i_1,\ldots,i_N] = \prod_{n=1}^N s_{n, i_n}$. 

For a three-way additive fine-tuning update $\Delta \bm{\mathcal{W}} \in \bR^{m \times n \times L}$, CP based methods parameterize it as a rank-$r$ tensor
\[
    \Delta \bm{\mathcal{W}}= \llbracket \bfS_1,\bfS_2,\bfS_3 \rrbracket
    := \sum_{i=1}^r  \mathbf{s}_{1,i} \circ  \mathbf{s}_{2,i} \circ  \mathbf{s}_{3,i}
\]
where $\bfS_1\in\bR^{m\times r}$, $\bfS_2\in\bR^{n\times r}$, and $\bfS_3\in\bR^{L\times r}$ are factor matrices. Here, $\mathbf{s}_{k,i}$ denotes the $i$-th column of $\bfS_k,\,k=1,2,3$.
This parameterization involves $ r(L + m + n)$ parameters, offering remarkable savings compared to $Lmn$ parameters present in the full update $\Delta \bm{\mathcal{W}}$. 

In practice, the effectiveness of these approaches also depends on how weights are tensorized before applying CP factorization. For example, CaRA~\cite{veeramacheneni2025cp} for Vision Transformers (ViTs) groups the multi-head query, key, and value (QKV) projection matrices across layers to form a fourth-order tensor, while aggregating the output projection and the two FFN layers into another tensor, with CP applied to each tensor separately. Likewise, in LLM settings, LoRTA~\cite{hounie2024lorta} structures the updates of QKV and output projection matrices as a fifth-order tensor parameterized using the CP decomposition. 

Beyond parameter efficiency, a major upshot of CP in SP is its identifiability, which asserts under mild conditions that the CP decomposition of a specific tensor is unique up to permutation and scaling of components; see, e.g.,~\cite{jiang2004kruskal,sidiropoulos2017tensor}. Intriguing questions for fine-tuning is whether identifiability could yield more interpretable updates, or, be leveraged to improve the optimization landscape of tensorized adapters. 

\textbf{Tucker decomposition-based parameterization.} 
Another popular SP tool for tensor factorization is the Tucker decomposition~\cite{tucker1964extension,de2000multilinear}, which can be viewed as a higher-order counterpart of the matrix SVD: 
\[ 
    \bm{\mathcal{A}}  = \bm{\mathcal{G}} \times_1 \U_1 \times_2 \ldots \times_N \U_N 
\]
where $\bm{\mathcal{G}}$ is a smaller core tensor, and $\{ \U_n \}_{n=1}^N$ are mode-wise factor matrices (generalized “singular vectors”).
Unlike the matrix SVD where the core (singular value matrix) is strictly diagonal, the Tucker core $\bm{\mathcal{G}}$ is typically dense. Conceptually, this can be viewed as a generalization of the PoLAR magnitude–direction decomposition, where the factor matrices capture subspace directions per mode, while the core tensor depicts the multi-way “magnitudes” and couplings across modes. 
SP algorithms often impose orthogonality constraints on the factor matrices, whereas modern deep learning implementations typically relax this requirement to simplify optimization. Our discussion hereafter follows the latter.

In fine-tuning, the stacked additive weights $\Delta \bm{\mathcal{W}} \in \bR^{m \times n \times L}$ can be parameterized via Tucker decomposition as
\begin{align}
    \Delta \bm{\mathcal{W}} & = \bm{\mathcal{G}} \times_1 \U_1 \times_2 \U_2 \times_3 \U_3 \\ 
    & := \sum_{i=1}^{r_1} \sum_{j=1}^{r_2} \sum_{k=1}^{r_3} g_{ijk}  \mathbf{u}_{1, i} \circ \mathbf{u}_{2, j} \circ \mathbf{u}_{3, k} \nonumber
\end{align}
where $\U_1\in\bR^{m\times r_1}$, $\U_2\in\bR^{n\times r_2}$, and $\U_3\in\bR^{L\times r_3}$ are factor matrices; $\bm{\mathcal{G}}\in\bR^{r_1\times r_2\times r_3}$ is the core tensor; and $\times_k$ denotes the mode-$k$ tensor–matrix product.
Tucker represents an $mnL$-dimensional tensor using $r_1r_2r_3 + Lr_1 + mr_2 + nr_3$ parameters. This full three-mode factorization is also known as Tucker-3. In practice, tensorization can markedly affect parameter efficiency, as the core size $r_1r_2r_3$ grows multiplicatively with the chosen multilinear ranks. 
Fact-TK~\cite{jie2023fact} concatenates weight matrices across ViT layers (with FFN linear layers sliced to compatible sizes) to form a three-way tensor, and applies Tucker-3 factorization to parameterize the update. 
Alternatively, LoTR~\cite{bershatsky2024lotr} stacks weights from LLM layers to form a third-order tensor, but adopts a Tucker-2 variant, which leaves one mode uncompressed by setting $\U_3 = \I$ and $r_3=n$; that is, 
\begin{align}
    \Delta \bm{\mathcal{W}} = \bm{\mathcal{G}} \times_1 \U_1 \times_2 \U_2 
   ~~\Leftrightarrow ~~ \Delta \W_l = \U_1 \mathbf{G}_l \U_2^\top.  \label{eq.lotr}
\end{align}
This formulation naturally admits a layer-wise interpretation: all layers share the same $\U_1$ and $\U_2$, while each layer learns its own core slice $\mathbf{G}_l$. 
Note that if we merge $\U_l := \U_1 \mathbf{G}_l$ in~\eqref{eq.lotr}, then each layer update becomes $\Delta \W_l = \U_l \U_2^\top$. This recovers a weight-sharing variant of LoRA~\cite{renduchintala2024tied}, where the right projection matrix is tied for (shared by) all layers. 

\begin{table*}
  \caption{A SUMMARY of efficient architecture variants of LoRA. }
  \centering
  \resizebox{2\columnwidth}{!}{
  \renewcommand{\arraystretch}{1.25}
  \begin{tabular}{c c c c c c }
    \toprule
    Method & reference & Architecture & Notes & SP relation & \# parameters 
    \\
    \midrule
    LoRA &~\cite{hu2021lora} & $\Delta \W_l = \X_l\Y_l^\top$ & N/A & BM factorization & $(m+n)r L$ 
    \\
    \midrule
    AdaLoRA &~\cite{zhang2023adaptive} & $\Delta \W_l = \U_l\diag(\bm \sigma_l)\V_l^\top$ & $\U_l^\top \U_l \approx \I $, $\V_l^\top \V_l \approx \I $  & SVD &  $\leq (m+n + 1)r L$   \\
    PoLAR/SteLLA &~\cite{kai2025,li2025stella} & $\U_l \bfSigma_l\V_l^\top$ & $\U_l^\top \U_l \approx \I $, $\V_l^\top \V_l \approx \I $, $\bfSigma_l \in \bR^{r \times r}$ & SVD & $ (m+n + r)r L$ 
    \\
    \midrule
    RoSA & $\cite{nikdan2024rosa} $ & $\Delta \W_l = \X_l\Y_l^\top + \mathbf{S}_l$ &  $\mathbf{S}_l \in \bR^{m\times n}$ yet sparse & Robust PCA & $[(m+n)r + s] L $ 
    \\
    FedPara & \cite{hyeon2021fedpara} & $ \Delta \W_l = (\X_l^1\Y_l^{1\top}) \odot (\X_l^2\Y_l^{2\top}) $ & N/A & Hadamard product & $2(m+n)rL$   
    \\
    HiRA  &~\cite{huang2025hira} & $ \Delta \W_l = \W_l \odot (\X_l\Y_l^{\top}) $ & N/A & Hadamard product & $(m+n)rL$  
    \\ 
    KronA &~\cite{huang2025hira} & $ \Delta \W_l = \X_l \otimes \Y_l $, $\X_l \in \bR^{d_1 \times d_2}$, $\Y_l \in \bR^{d_3 \times d_4}$ & $m = d_1 d_2$, $n = d_3 d_4$ & Kronecker product & $(d_1d_2 + d_3 d_4)L$   
    \\
    
    \midrule
    CaRA/LoRTA &~\cite{veeramacheneni2025cp,hounie2024lorta} & $ \Delta \bm{\mathcal{W}} = \sum_{i=1}^r  \mathbf{s}_{1,i} \circ  \mathbf{s}_{2,i} \circ  \mathbf{s}_{3,i}$ & N/A & CP & $mr + nr + Lr$ 
    \\
    LoTR &~\cite{bershatsky2024lotr} & $\Delta \bm{\mathcal{W}}  = \bm{\mathcal{G}} \times_1 \U_1 \times_2 \U_2$ & weight sharing across layers & Tucker-2 & $mr + nr + r^2 L$ 
    \\
	FacT-TK &~\cite{jie2023fact} & $\Delta \bm{\mathcal{W}} = \bm{\mathcal{G}} \times_1 \U_1 \times_2 \U_2 \times_3 \U_3$ & N/A & Tucker-3 & $mr + nr + Lr + r^3$  \\
    FacT-TT/MetaTT &~\cite{jie2023fact,lopez2025metatt} & $ \Delta \bm{\mathcal{W}}[i,j,k] = \mathbf{G}_1[i,:] \bm{\mathcal{G}}_2[:, j, :]  \mathbf{G}_3[:,k] $  & $\mathbf{G}_1\in\bR^{m\times r}$, $\bm{\mathcal{G}}_2\in\bR^{r\times n\times r}$, and $\mathbf{G}_3\in\bR^{r\times L}$ & TT & $mr + nr^2 + Lr$ \\
    \midrule
    DoRA &~\cite{liu2024dora} & $\W_l + \Delta \W_l = \mathbf{m}_l  \frac{\W_l + \X_l \Y_l^\top}{\| \W_l + \X_l \Y_l^\top\|_c}$ & N/A & Rotation & $(m+n)rL + nL$  \\
    Vera  &~\cite{kopiczko2023vera} & $\Delta \W_l = \diag(\mathbf{a}_l) \A  \diag(\mathbf{b}_l) \B^\top $ & $\mathbf{A} \in \bR^{m \times r}$ and $\mathbf{B}\in \bR^{m \times r}$ shared &  Parameter sharing & $(m + r)L$ 
    \\
    NoLA  &~\cite{koohpayegani2023nola} &  $\Delta \W_l = \big(\sum_{k=1}^K \alpha_l^k \A^k \big) \big(\sum_{k=1}^K \beta_l^k \B^k \big)$ &  Dictionaries: $\{\A^k \in \bR^{m \times r}\}_{k=1}^K$ and $\{\B^k \in \bR^{n \times r}\}_{k=1}^K$  & Parameter sharing & $K(m+n)r + 2KL$ 
    \\
    \bottomrule
  \end{tabular}
  }
  \label{tab.para}
\end{table*}

\textbf{Tensor train (TT)-based parameterization.}
TT decomposition represents a high-order tensor as a sequence or ``train,'' of low-order cores. This enables a parsimonious parameterization whose complexity grows linearly (rather than exponentially as in Tucker) with the tensor order. 
Formally, given an $N$-th order tensor $\bm{\mathcal{A}} \in\bR^{I_1\times \cdots \times I_N}$, TT expresses each entry as a product of matrices
\[
    \bm{\mathcal{A}}[i_1,\dots,i_N] = \bm{\mathcal{G}}_1[:,i_1,:] \bm{\mathcal{G}}_2[:,i_2,:] \cdots \bm{\mathcal{G}}_N[:,i_N,:]
\]
where each core $\bm{\mathcal{G}}_k \in \bR^{r_{k-1}\times I_k \times r_k}$, and the matrix slice $\bm{\mathcal{G}}_k[:, i_k, :]\in\bR^{r_{k-1}\times r_k}$. 
Integers $r_0,r_1,\dots,r_N$ are the TT ranks, with boundary conditions $r_0=r_N=1$. 
For fine-tuning, the stacked additive weights form a three-way tensor $\Delta \bm{\mathcal{W}} \in \bR^{m\times n\times L}$ and rely on a three-core TT factorization to obtain
\[
    \Delta \bm{\mathcal{W}}[i,j,k] = \mathbf{G}_1[i,:] \bm{\mathcal{G}}_2[:, j, :] \mathbf{G}_3[:,k] 
\]
where $\mathbf{G}_1\in\bR^{m\times r_1}$, $\bm{\mathcal{G}}_2\in\bR^{r_1\times n\times r_2}$, and $\mathbf{G}_3\in\bR^{r_2\times L}$ (recall that $r_0=r_3=1$). 
TT involves $r_1 m + r_1 r_2 n + r_2 L$ parameters, but its parameter count is sensitive to the permutation of modes; e.g., swapping the second and third modes changes intermediate rank requirements, and thus the total parameter count.

In practice, implementation strategies vary by how the weights are tensorized. Fact-TT~\cite{jie2023fact} concatenates the weights per layer of a ViT to form a three-way tensor, and applies TT decomposition for parameterization. In contrast, LoRETTA~\cite{yang2024loretta} reshapes the weight matrices of \textit{a single layer} to form higher-order tensors, and uses TT cores to parameterize the associated updates. MetaTT~\cite{lopez2025metatt} further extends this approach by incorporating both the layer dimension and the intra-layer dimensions into a single high-order TT parameterization.

\textbf{Comparison.}
Next, a brief comparison of these tensor-based methods is provided. In terms of parameter count, CP and TT are often more compact than Tucker. Specifically, for an $N$-way tensor, their trainable parameters typically scale roughly linearly with order $N$, whereas Tucker includes a core tensor whose size grows multiplicatively with $N$.
While efficient, TT is often less interpretable than CP or Tucker. Furthermore, because TT parameterization involves a long chain of matrix products, it can be viewed as a generalization of deep linear networks. Consequently, it may introduce additional optimization challenges as $N$ grows~\cite{arora2018}. For instance, the iteration complexity of gradient descent can grow exponentially with depth (and thus with tensor order) in certain regimes~\cite{shamir2019exponential}. 
For fine-tuning, all three parameterizations appear promising empirically, and in many reported settings they can match matrix-based adapters while using fewer trainable parameters. However, empirical evidence is still fragmented across tasks, models, and tensorization choices. Developing principled guidelines for, e.g., which particular decomposition is preferable under specific hardware or data regimes, remains an important 
uncharted territory. 

\subsection{Other SP-related parameterizations}

Beyond structured low-rank matrices and tensors, several variants also connect naturally to SP through the lenses of parameter sharing and rotation-based parameterizations. Parameter sharing has long played a central role in SP, as exemplified by dictionary learning and sparse coding \cite{papyan2017working}, whereas rotations and orthogonal transforms are classical tools for representing signals in more efficient and structured coordinate systems \cite{wang2012introduction}.

\textbf{Parameter sharing.} 
Sharing parameters across layers has been leveraged in deep learning to reduce the number of trainable parameters, thereby improving memory efficiency; see e.g.,~\cite{lan2019albert}. In the fine-tuning context, we have already discussed LoTR~\cite{bershatsky2024lotr} through the lens of Tucker-2 parameterization. Additional examples include
VeRA~\cite{kopiczko2023vera} that introduces a vector-based variant of LoRA using the per-layer update 
\[
    \Delta \W_l = \diag(\mathbf{a}_l) \A  \diag(\mathbf{b}_l) \B^\top
\]
where $\A\in\bR^{m\times r}$ and $\B\in\bR^{n\times r}$ are fixed Gaussian random matrices shared across all layers. The only trainable parameters are the per-layer scaling vectors $\mathbf{a}_l\in\bR^m$ and $\mathbf{b}_l\in\bR^r$.
In comparison, NoLA~\cite{koohpayegani2023nola} employs a fixed dictionary of random matrices $\{\A^k \in \bR^{m \times r}\}_{k=1}^K$ and $\{\B^k \in \bR^{n \times r}\}_{k=1}^K$ shared across layers, and learns the layer-specific mixing weights $\bm{\alpha}_l \in\bR^K $ and $\bm{\beta}_l\in\bR^K$, to arrive at the layer-wise update 
\begin{align}
    \Delta \W_l = \Big(\sum_{k=1}^K \alpha_l^k \X^k \Big) \Big(\sum_{k=1}^K \beta_l^k \Y^k \Big). 
\end{align}
While storing the dictionary can incur large memory overhead, this can be mitigated by relying on structured matrices, e.g., Fourier bases, which can be generated on-the-fly~\cite{gao2024}.

\textbf{Rotation-based fine-tuning.} 
Aside from additive adapters, an alternative is to model fine-tuning updates as multiplicative $\W^{\text{ft}}=\W^{\text{pre}}\bfP$ with some $\bfP\in\bR^{n\times n}$. 
Different from additive updates, this choice reflects a distinct tradeoff. On the upside, the induced weight increment $\W^{\text{ft}}-\W^{\text{pre}} = \W^{\text{pre}} (\mathbf{P} - \mathbf{I}_n )$ can have high rank. 
On the downside, $\W^{\text{ft}}$ is constrained to lie in the column space of $\W^{\text{pre}}$, which can rule out updates that additive LoRA can represent.
Further, parameter efficiency can be acquired by restricting $\bfP$ with additional structure, such as a block-diagonal form. If one further imposes (approximate) orthogonality, this leads to orthogonal fine-tuning methods including OFT~\cite{qiu2023controlling}, BOFT~\cite{liu2023parameter}, ETHER~\cite{bini2024ether}, and HRA~\cite{HRA}. These approaches are popular in text-to-image generation, and can be interpreted as rotating the neurons, which are just the columns of $\W^{\text{pre}}$.
They differ mainly in how they constrain and parameterize $\bfP$, offering different balance points between expressiveness and parameter efficiency. For example, OFT uses a block-diagonal parameterization of $\mathbf{P}$, whereas BOFT is motivated by the Cooley–Tukey fast Fourier transform to represent a dense $\mathbf{P}$ as a product of sparse matrices via butterfly structures.

Another way to express $\mathbf{P}$ implicitly is through a magnitude–direction decomposition of neurons (i.e., column of $ \W_l + \Delta \W_l$), as in DoRA~\cite{liu2024dora}. While DoRA was originally presented through the lens of (vectorized) weight normalization, a widely used technique in deep learning~\cite{salimans2016weight}, it can also be naturally viewed within the rotation framework. Concretely, DoRA parameterizes the update as
\[
    \W_l + \Delta \W_l = \mathbf{m}_l  \frac{\W_l + \X_l \Y_l^\top}{\| \W_l + \X_l \Y_l^\top\|_c}
\]
where $\|\cdot\|_c$ denotes the column-wise $\ell_2$ norms of a matrix, with multiplication and division understood columnwise. Under this parameterization, the directional component is captured implicitly by normalizing each column of $\W_l + \X_l \Y_l^\top$, while $\mathbf{m}_l$ controls the corresponding scaling of each neuron.

\subsection{A summary of efficient parameterizations}
Let us consider a case study to briefly compare the parameterization approaches discussed so far. Specifically, consider a network with linear layers and associated weights $\{\W_l \in \bR^{m\times n}\}_{l=1}^L$. This setup slightly idealizes the Transformer architectures, where attention and multilayer perceptron (MLP) projections typically differ in dimensions. Nonetheless, it suffices to highlight the main ideas, and extend these results to  heterogeneous shapes requires only minor algebraic adjustments. For an apples-to-apples comparison, all rank-related hyperparameters are standardized as $r$ in both matrix- and tensor-based parameterizations.
The latter simply stacks the layer weights $\{\W_l\in\bR^{m\times n}\}_{l=1}^L$ to form a third-order tensor, and does not consider alternative choices such as reshaping into higher-order tensors. 
Different parameterizations are summarized in Table~\ref{tab.para}. A head-to-head empirical comparison is deliberately omitted, as the performance of these methods is highly sensitive to the optimization strategy employed. Indeed, many parameterizations can be further improved when paired with an optimizer tailored to their underlying structure. In the next section, we review recent optimization advances that explicitly account for these adapter architectures.

\section{Efficient optimization of LoRA}
\label{Sec.opt}
Beyond adapter models, the optimization strategy is another crucial ingredient for efficient fine-tuning. An example is depicted in Figure~\ref{fig.polar}(c): without an optimizer that effectively learns the adapter architecture, LoRA often underutilizes the available rank budget and results in limited expressiveness. 

While optimization theory for LoRA in full LLMs remains far from complete, it is already clear that off-the-shelf smooth optimization methods do not always translate cleanly selected parameterizations. In fact, even the one-layer testbed in~\eqref{eq.prob-sensing} can exhibit an unfavorable landscape, which lacks the “nice” properties present in standard analyses of smooth optimization. Furthermore, LoRA introduces additional architectural structures, such as parametric symmetries (gauge invariances), that generic optimizers do not exploit.

A common strategy for improving LoRA optimization is to first understand and explicitly characterize these overlooked properties on simplified testbeds; cf. \eqref{eq.prob-sensing}. This serves as a principled sanity check: if an algorithm can hardly cope with a simplest setup, it is less likely to robustly carry over to large scale fine tuning. Let us first outline challenges and structural opportunities that one can leverage for efficient optimization.

\subsection{The simplest testbed problem} 
Recall from Section~\ref{Sec.LoRA-single-layer} that fine-tuning a linear layer is intimately related to classical matrix sensing. To streamline the discussion, let us focus on the population form of~\eqref{eq.prob-sensing} in the limit $N\to\infty$. Under mild assumptions, the objective reduces to the matrix factorization problem
\begin{align}\label{eq.prob-mf}
    \min_{\X \in \bR^{m \times r}, \Y \in \bR^{n \times r}} f(\X, \Y):= 	\frac{1}{2}\| \A - \X\Y^\top \|	_\fro^2.
\end{align}
Although the global optimum in~\eqref{eq.prob-mf} is characterized by the Eckart–Young–Mirsky theorem via a truncated SVD of $\A$, the associated optimization landscape is non-trivial. Moreover,~\eqref{eq.prob-mf} already captures several structural features shared with LoRA fine-tuning that are underexploited by standard optimization approaches.

\textbf{The loss landscape.}
The bilinear parameterization \(\X\Y^\top\) introduces nonconvexity, and can violate standard regularity conditions presumed in smooth optimization analyses. Generally speaking, attractive global properties, including smoothness or Polyak–{\L}ojasiewicz (PL) inequality, no longer hold for $f$. This can be readily illustrated by a simple example.
\begin{example}\label{example.simple}
    Consider the simplest instance of~\eqref{eq.prob-mf}, where $m=n=r = 1$, and $\A = 1$; that is,
    \[
        f(x, y) = \frac{1}{2}\| xy - 1 \|^2.
    \]
\end{example}
First, it is straightforward to see that $(x,y)=(0,0)$ is a saddle point, confirming nonconvexity. The presence of saddle points precludes the global PL condition (gradient dominance) $\|\nabla f\|^2 \ge c\,(f-f^\star)$ for global minimum $f^*$, and constant $c > 0$. To see that $f$ is not globally smooth, check the Hessian 
\begin{equation}\label{eq.example_hessian}
    \mathbf{H}_f(x,y)=
    \begin{bmatrix}
        y^2 & 2xy - 1\\
        2xy - 1 & x^2
    \end{bmatrix}. 
\end{equation}
Along the curve $xy=\tfrac{1}{2}$, the largest  eigenvalue $\lambda_{\max}(\mathbf{H}_f)=\max\{x^2,y^2\}$. Since $x$ or $y$ can be arbitrarily large, the curvature is unbounded, verifying that $f$ is not globally smooth. 

These unfavorable properties also persist in the LoRA fine-tuning of LLMs. In particular, the induced objectives are nonconvex and lack global smoothness, with curvature varying substantially across the parameter space. In higher dimensions, the loss landscape becomes even richer, exhibiting more intricate saddle structures and stronger curvature variability, which in turn poses remarkable optimization challenges. Despite these difficulties, \eqref{eq.prob-mf} and LoRA fine-tuning share certain favorable topological features. For instance, \eqref{eq.prob-mf} is known to admit no spurious local minima; i.e., all local minima are globally optimal. Likewise, LoRA has also been shown free of spurious local minima in the neural tangent kernel (NTK) regime~\cite{jang2024lora}, though little is known outside this regime. 

\textbf{Architectural properties.}
An appealing optimization algorithm for LoRA~\eqref{eq.prob-lora}, and its simplified variant~\eqref{eq.prob-mf}, should not only cope with nonconvexity (e.g., by escaping saddles), but also maintain stability in the absence of global smoothness. In the prototype problem~\eqref{eq.prob-mf}, we call an algorithm \emph{efficient} if it 
reaches a global optimum within polynomial time.
Classical theories, especially for local convergence, have been thoroughly developed; see the survey~\cite{chi2019nonconvex} and references therein. Beyond these well-studied aspects, there are also architecture-induced levers that can be explicitly exploited, as  discussed next.

Among the structural issues presented in the following subsections, the most salient one is the parametric symmetry, also known as \emph{gauge invariance}. Both~\eqref{eq.prob-lora} and~\eqref{eq.prob-mf} exhibit this property in their parameterizations. Using Example~\ref{example.simple} for illustration, points such as $(x_1,y_1)=(\tfrac{1}{2},1)$ and $(x_2,y_2)=(\tfrac{1}{8},4)$ are gauge-equivalent: they produce identical loss values, yet induce distinct local geometries, e.g., different Hessians~\eqref{eq.example_hessian}. This discrepancy in curvature can lead to noticeably different optimization behaviors.

For ease in exposition, we confine our discussion to an easy case, where the dimension $r$ of $\X$ and $\Y$ is chosen identical to  the target $\rank (\A) := r_A$. 
In this setups, a zero training loss is achievable\footnote{In real-world fine-tuning, $0$ training loss is often observed when applying LoRA to 7B models on e.g., the GLUE benchmark.}. For later use, the condition number of $\A $ is $\kappa:= \sigma_{\max} (\A) / \sigma_{\min} (\A) $. The layer index $l$ is dropped hereafter for brevity.

\subsection{Initialization of LoRA}
\label{Sec.init}

The vanilla LoRA initializes $\X_0 \sim {\cal N}(0, \Theta(1/m))$ and $\Y_0 = \mathbf{0}$ for $\forall l$. This ensures $\W_l + \X_0 \Y_0^\top = \W_l$, meaning that fine-tuning initiates exactly from the pre-trained weights. While it is also possible to initialize $\X_0 = \mathbf{0}$ and $\Y_0 \sim {\cal N}(0, \Theta(1/r))$\footnote{The default ``fan-in'' mode normalizes variance of neural network weights based on the number of input units (fan-in) to a layer.},~\cite{hayou2024impact} shows that the vanilla LoRA initialization enables larger learning rates, and leads to more stable feature learning in the regime of infinite network width. 

But how does initialization affect optimization behavior? In classical smooth nonconvex optimization, initialization is analyzed mainly through its distance to an optimum or a stationary point: starting closer generally reduces the required iterations for convergence~\cite{nesterov2013introductory}. However, because the objectives in~\eqref{eq.prob-lora} and~\eqref{eq.prob-mf} are nonconvex and not globally smooth, initialization plays a critical role in a more complicated way than the simple distance-to-solution. 
In particular, a class of ``subspace aligned'' initializations have proven useful for both~\eqref{eq.prob-mf} and~\eqref{eq.prob-lora}.

\begin{figure}[t]
	\centering
	\includegraphics[width=0.4\textwidth]{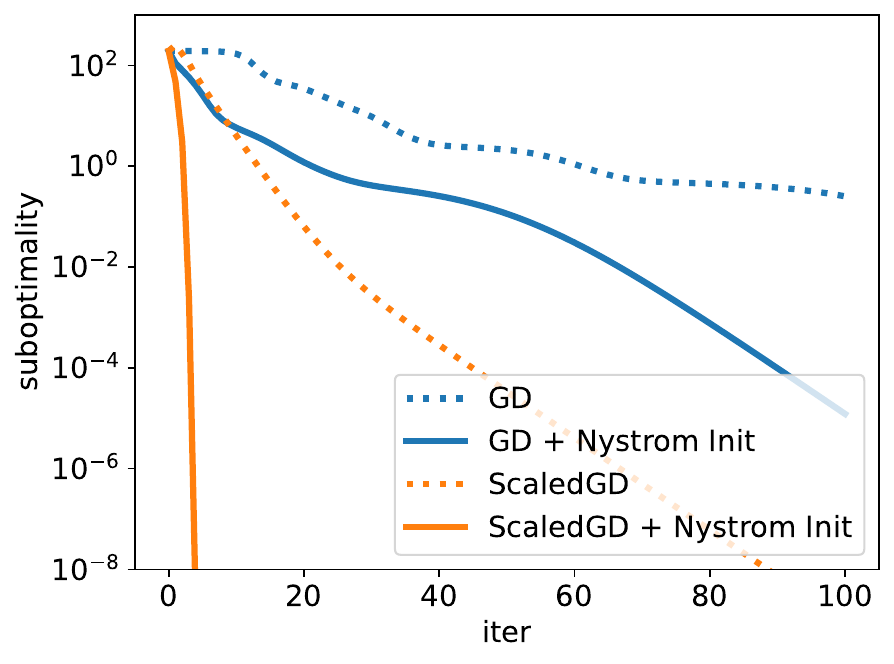}
	\caption{Convergence of~\eqref{eq.prob-mf} with random Gaussian and Nystr\"om initializations.}
	 \label{fig.init}
\end{figure}

For the one-layer adaptation in~\eqref{eq.prob-mf}, recent work~\cite{li2025nora,xu2024provable} proves that the initialization scheme can have up to an \emph{exponential} effect on the convergence rate with a given optimization algorithm. 
To be specific, a Nystr\"om sketch~\cite{frangella2023randomized,tropp2017fixed} is applied for initialization
\begin{align}\label{eq.nystrom_init}
    \X_0  = \A \bfPhi, ~~ \Y_0 = \mathbf{0}
\end{align}
where $\bfPhi \in \bR^{r \times r}$ is a Gaussian random matrix. 

We first highlight a representative case, where initialization can have a pronounced impact. Consider the optimization algorithm \emph{ScaledGD}~\cite{tong2021}, which is elaborated later in Section~\ref{Sec.gauge-invariance}. With step size $\eta > 0$, its iterates (indexed by $t$) are  
\begin{subequations}\label{eq.asym-iter}
	\begin{align}
		& \X_{t+1} = \X_t - \eta (\X_t \Y_t^\top - \A) \Y_t (\Y_t^\top \Y_t)^{-1}  \\
		& \Y_{t+1} = \Y_t - \eta (\X_t \Y_t^\top - \A)^\top \X_t (\X_t^\top \X_t)^{-1}.
	\end{align}
\end{subequations}
\begin{theorem}[Impact of initialization for ScaledGD~\cite{jia2023, li2025nora}, informal]\label{thm.scaledgd}
    With random initialization $\X_0 \sim {\cal N}(0, \sigma_x^2)$ and $\Y_0 \sim {\cal N}(0, \sigma_y^2)$ for some sufficiently small constants $\sigma_x$ and $\sigma_y$, ScaledGD with proper step sizes applied to~\eqref{eq.prob-mf} requires ${\cal O}(\log(1/\epsilon))$ iterations to reach global convergence $f(\X_t, \Y_t)\leq \epsilon$. However, ScaledGD under Nystrom initialization~\eqref{eq.nystrom_init} achieves the same error in ${\cal O}( \log\log(1/\epsilon))$ iterations.
\end{theorem}

Similarly,~\cite{xu2024provable} establishes benefits of Nystr\"om initialization for vanilla gradient descent (GD), which takes the form
\begin{align*}
    \X_{t+1} & = \X_t - \eta (\X_t \Y_t^\top - \A) \Y_t \\
    \Y_{t+1} & = \Y_t - \eta (\X_t \Y_t^\top - \A)^\top \X_t.
\end{align*}

\begin{theorem}[Impact of initialization for GD~\cite{ye2021, xu2024provable}, informal] \label{thm.gd}
    With random initialization $\X_0 \sim {\cal N}(0, \sigma_x^2)$ and $\Y_0 \sim {\cal N}(0, \sigma_y^2)$ for some sufficiently small constants $\sigma_x$ and $\sigma_y$, GD with proper step sizes applied to~\eqref{eq.prob-mf} requires ${\cal O}(\kappa^4 \log(1/\epsilon))$ iterations to reach global convergence $f(\X_t, \Y_t)\leq \epsilon$. However, GD under Nystrom initialization~\eqref{eq.nystrom_init} achieves the same error in ${\cal O}( \kappa^2\log(1/\epsilon))$ iterations.
\end{theorem}

A numerical comparison is provided in Figure~\ref{fig.init}. The proofs are more involved, and we refer readers to the cited papers for details. The intuition behind acceleration is that the sketch initialization $\X_0=\A\bfPhi$ captures explicitly the column space (and implicitly the row space) of $\A$. Initializing the optimization with the optimal subspace, the algorithm bypasses the initial search phase for subspace alignment, thereby substantially accelerating convergence. 

Moving back to the LoRA fine-tuning, these theoretical insights suggest that an effective LoRA initialization should ideally capture the column and row spaces of the optimal update $\Delta \W^\star$. However, the optimal subspace associated with $\Delta \W^\star$ is unknown a priori. Motivated by the empirical observation that well-performing updates $\Delta \W$ often exhibit substantial subspace overlap with the pre-trained weights $\W^{\rm pre}$, a practical alternative is to use $\W^{\rm pre}$ as a proxy to extract directional information for initialization~\cite{meng2024pissa,buyukakyuz2024olora,li2025nora}. Concretely, PiSSA~\cite{meng2024pissa} initializes $\X_0,\Y_0$ using the top-$r$ singular directions of $\W_0$ via truncated SVD, while OLoRA~\cite{buyukakyuz2024olora} utilizes a top-$r$ QR decomposition for a similar goal. MiLoRA~\cite{wang2025milora} instead leverages the bottom-$r$ singular directions of $\W_0$, and NoRA~\cite{li2025nora} employs a Nystr\"om sketch to sample these directions. Finally, LoRA-GA~\cite{wang2024lora-ga} and LoRA-One~\cite{zhang2025lora} pursue subspace alignment through gradient-based initialization signals.

\subsection{Optimization via asymmetric roles of $\X$ and $\Y$}
\label{Sec.AltGD}
Another standard approach in SP for solving~\eqref{eq.prob-mf} (or~\eqref{eq.prob-sensing}) is alternating gradient descent (AltGD), sometimes discussed under the broader umbrella of alternating least squares (ALS). AltGD alternates between updating $\X_t$ and $\Y_t$ per iteration
\begin{subequations}\label{eq.alt-gd}
\begin{align}
    \X_{t+1} & = \X_t - \eta_x (\X_t \Y_t^\top - \A) \Y_t \\
    \Y_{t+1} & = \Y_t - \eta_y (\X_{t+1} \Y_t^\top - \A)^\top \X_{t+1}
\end{align}
\end{subequations}
where $\eta_x > 0$ and $\eta_y > 0$ are step sizes. 
Note that the gradients for $\X_t$ and $\Y_t$ are evaluated at different points. For~\eqref{eq.prob-mf}, it is known that such method converges linearly.
\begin{theorem}[Convergence of AltGD~\cite{ward2023}, informal]
    Consider AltGD in~\eqref{eq.alt-gd} initialized with  $\X_0  = \A \bfPhi$ and $\Y_0 = \bfPsi$, where $\bfPhi$ and $\bfPsi$ are Gaussian random matrices of proper sizes. With proper step sizes, AltGD converges to global minima $f(\X_t,\Y_t) \leq \epsilon$ in ${\cal O}(\kappa^2 \log(1/\epsilon))$ iterations. 
\end{theorem}

Compared to GD with Nystr\"om initialization in Theorem~\ref{thm.gd}, AltGD achieves a comparable iteration complexity even when the initialization only captures the column space of $\A$ (rather than both its column and row spaces). This suggests a practical advantage by treating the two factors asymmetrically. This asymmetry has been extended to ScaledGD~\eqref{eq.asym-iter}. In particular, asymmetric ScaledGD variants~\cite{jia2023,liu2025efficient} show that updates~\eqref{eq.alt-gd} enable larger step size for faster convergence.

Because~\eqref{eq.alt-gd} requires two backpropagation passes per iteration, AltGD and related alternating schemes are not yet widely adopted in practical LoRA fine-tuning, while the underlying principle of exploiting the asymmetric roles of $\X$ and $\Y$ has been broadly recognized. For example, LoRA+~\cite{hayou2024lora+} advocates for assigning different learning rates to $\X$ and $\Y$ to reflect their asymmetric roles under standard LoRA initialization, which can improve stability and facilitate feature learning. More broadly, several works explicitly study this asymmetry. For instance,~\cite{zhu2024} shows that, under suitable assumptions, fixing a random $\X$ and optimizing only $\Y$ can lead to improved generalization guarantees.

\subsection{Gauge-invariant optimization}
\label{Sec.gauge-invariance}
Standard optimization methods often ignore the ``invariance'' inherent in LoRA parameterization. Consider a bilinear pair $(\X,\Y)$. The training objective depends on $(\X,\Y)$ solely through their outer product $\X\Y^\top$, which is invariant to invertible linear transformations; that is, for any invertible $\Q\in\bR^{r\times r}$, the reparameterization $\tilde {\X}=\X\Q$ and ${\tilde \Y}=\Y\Q^{-\top}$ leaves the product unchanged: $\tilde{\X} \tilde{\Y}^\top = \X \Y^\top$. Thus, $(\X,\Y)$ and $(\tilde{\X},\tilde{\Y})$ attain the same objective value. In group theory parlance, this property is known as \emph{gauge invariance} ~\cite{mishra2014fixed}. 

Crucially, while the objective value is invariant, the \emph{local geometry} of the loss landscape at gauge-equivalent points can differ significantly. To see this, revisit Example~\ref{example.simple} with $f(x,y)=\frac{1}{2}(xy-1)^2$, and compare two equivalent parameterizations $(x,y)=(2,2)$ and $(\tilde x,\tilde y)=(1,4)$. It can be seen that the Hessian at these two points are
\begin{align}
    \mathbf{H}_f(2,2)=
    \begin{bmatrix}
        4 & 3\\
        3 & 4
    \end{bmatrix};~~
    \mathbf{H}_f(1,4)=
    \begin{bmatrix}
        16 & 3\\
        3 & 1
    \end{bmatrix}. 
\end{align}
The largest eigenvalues of these Hessians are different ($7$ versus $ \sim 19$), indicating a discrepancy of local smoothness. Beyond loss landscapes, the optimization behaviors can also differ considerably along the gauge orbit. 

\begin{example}
After one GD update with stepsize $\eta$, initializing at $(x,y)=(2,2)$ yields $(x^+,y^+)=(2-6\eta,2-6\eta)$, whereas initialization $(\tilde x,\tilde y)=(1,4)$ gives $(\tilde x^+,\tilde y^+)=(1-12\eta,4-3\eta)$. Since $\tilde x^+\tilde y^+ \neq x^+y^+$, these gauge-equivalent points can result in different objective changes and optimization trajectories.
\end{example}

Given that gauge-equivalent points can exhibit markedly disparate local geometries and optimization behaviors, it is natural to ask whether it is possible to systematically leverage these differences for faster convergence.

\begin{figure}[t]
	\centering
	\includegraphics[width=0.5\textwidth]{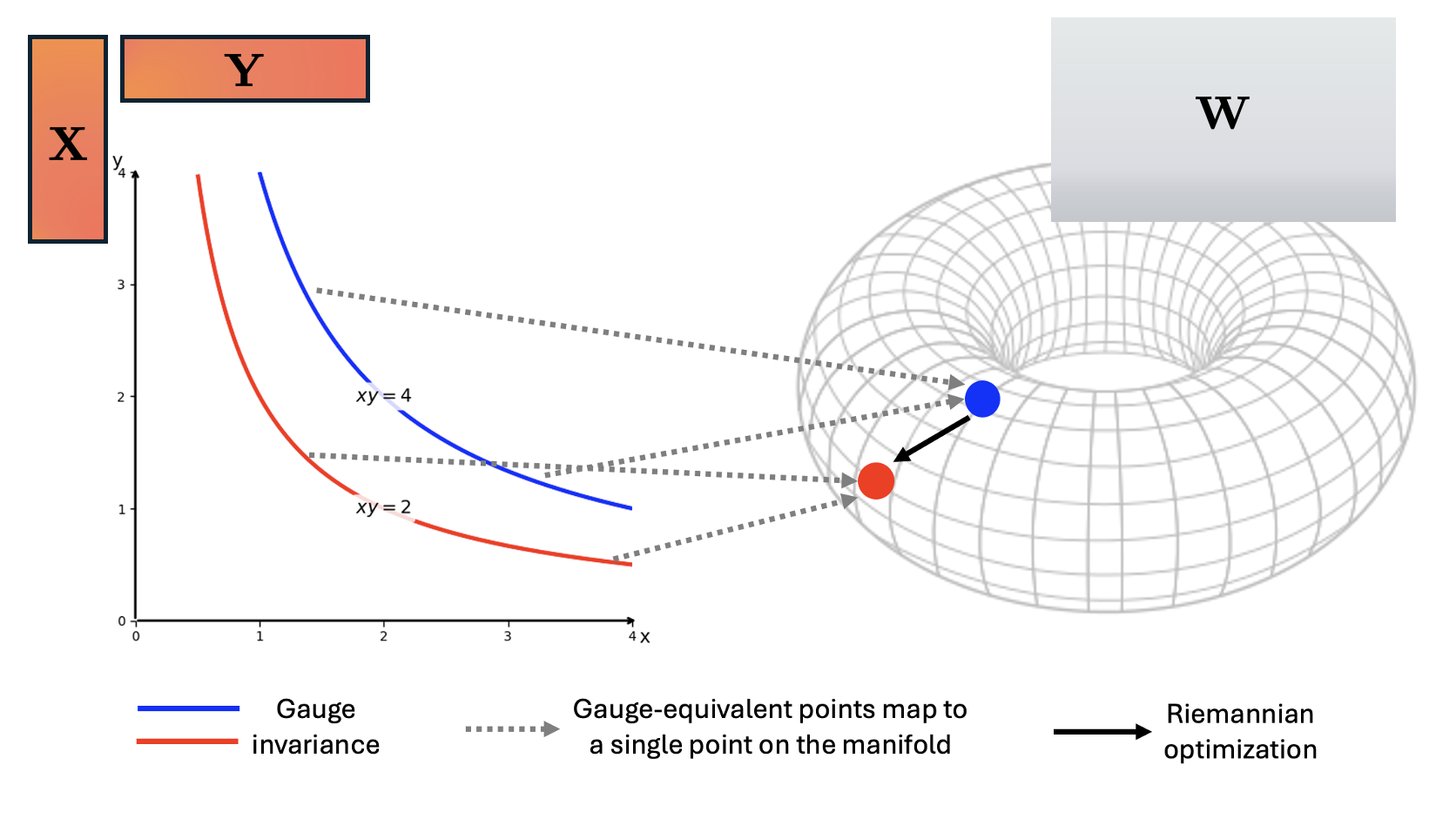}
	\caption{Minimizing $f(x,y)= \frac{1}{2} (xy-1)^2$ . The manifold shown on the right is for visualization purposes only. }
	 \label{fig.reflora}
\end{figure}

\textbf{A general framework.}
The most direct way to avoid gauge symmetry is to abandon the factorized variables $\X$ and $\Y$ and instead optimize over the product $\W = \X\Y^\top$. This is because all factor pairs related by the transformation $(\X\Q, \Y\Q^{-\top})$, with $\Q \in \mathsf{GL}(r)$, produce exactly the same $\W$. 
Note that $\mathsf{GL}(r)$ is the general linear group of degree $r$, i.e., the set of all $r \times r$ invertible matrices. 
In the scalar setting of Example~\ref{example.simple}, the entire hyperbola $\{(x,y)| xy=4\}$ can be viewed as one point in the $\W$ space, and $\{(x,y)| xy=2\}$ as another.
So from the perspective of $\W$, the whole gauge orbit collapses to a single point, and the ambiguity disappears.

However, explicitly forming and updating $\W$ is prohibitively expensive in the settings of interest. This can be overcome by exploiting more advanced optimization techniques under the mild assumption that $\W$ belongs to the fixed-rank manifold
$\mathcal{M}_r^{m\times n} := \{ \W \in \mathbb{R}^{m\times n} \mid \rank(\W)=r \}$
which admits the quotient representation $
    \mathcal{M}_r := \big(\bR_\ast^{m\times r} \times \bR^{n\times r}_\ast\big)\big/\mathsf{GL}(r)$.
Here, $/$ denotes the quotient under the natural group of $\mathsf{GL}(r)$ and $\mathbb{R}_*^{m\times r}$ and $\mathbb{R}_*^{n\times r}$ denote matrices with full column rank. 
The latter expression indicates that rather than optimizing over the full matrix $\W$, we can continue working with the factorized variables $(\X ,\Y) \in \bR_\ast^{m\times r} \times \bR^{n\times r}_\ast$, provided that the optimization procedure handles the quotient structure appropriately \cite{absil2008optimization,mishra2014fixed,boumal2023introduction}. This yields a memory-efficient update that is, in essence, equivalent to optimizing over $\W$ itself. A graphical illustration of this idea is shown in Figure~\ref{fig.reflora}.

A useful preliminary observation is that $\mathbb{R}_*^{m\times r}$ (and similarly $\mathbb{R}_*^{n\times r}$) is itself a manifold, namely the fixed-rank manifold $\mathcal{M}_r^{m\times r}$. Although optimization over $\X$ and $\Y$ naturally invites a Riemannian treatment, this manifold is, fortunately, simple enough that Riemannian gradient descent almost always reduces to vanilla gradient descent \cite{absil2008optimization,boumal2023introduction}.

With this background, handling gauge invariance requires only two additional ingredients beyond standard GD. The first is to define a \textit{gauge-invariant metric}, so that equivalent points $(\X\Q, \Y\Q^{-\top}), \forall \Q \in \mathsf{GL}(r)$ share the same geometric structure in their tangent spaces. This metric induces the associated Riemannian gradient $\Rgrad f$ that is used to perform update. The second, which can be optional, is to decompose $\Rgrad f$ into \textit{vertical} and \textit{horizontal} components. The vertical component describes motion within an equivalence class, while the horizontal component captures motion that genuinely changes the underlying point on the quotient. For instance, in Example~\ref{example.simple}, vertical motion corresponds to moving along the curve $xy=4$. Projecting away the vertical component can improve optimization efficiency, as it avoids updates along redundant move along the gauge orbits.

Next, we make these ideas concrete by presenting several methods that fall under this unified framework.

\textbf{ScaledGD.} 
Scaled gradient descent (ScaledGD) is among the most widely used approaches for matrix sensing and completion~\cite{tong2021,zhang2021preconditioned,xu2023}. While it admits several interpretations, we focus exclusively on its formulation as Riemannian optimization equipped with a specific gauge-invariant Riemannian metric. To be specific, for two tangent vectors $(\bm{\xi}_\X,\bm{\xi}_\Y)$ and $(\bm{\varphi}_\X,\bm{\varphi}_\Y)$ at $(\X,\Y)\in\bR_*^{m\times r}\times\bR_*^{n\times r}$, the metric is defined as
\begin{align*}
    g_{(\X,\Y)}& \big((\bm{\xi}_\X,\bm{\xi}_\Y),(\bm{\varphi}_\X,\bm{\varphi}_\Y)\big) \\
    & := \langle \bm{\xi}_\X (\Y^\top \Y), \bm{\varphi}_\X \rangle_\fro + \langle \bm{\xi}_\Y (\X^\top \X), \bm{\varphi}_\Y \rangle_\fro. 
\end{align*}
We refer readers to Appendix~\ref{Apdx.ScaledGD} for a detailed derivation of the gauge-invariance of this metric.
Under this metric, the Riemannian gradients at $(\X,\Y)$, also derived in detail in Appendix~\ref{Apdx.ScaledGD}, are
\begin{align*}
    \Rgrad_{\X} f &= \nabla_{\X} f(\X,\Y)(\Y^\top \Y)^{-1}, \\
    \Rgrad_{\Y} f &= \nabla_{\Y} f(\X,\Y)(\X^\top \X)^{-1}.
\end{align*}
ScaledGD then updates $(\X,\Y)$ by taking a step along the negative Riemannian gradient, i.e.,
\begin{align*}
    \X_{t+1} &= \X_t - \eta \Rgrad_{\X} f(\X_t,\Y_t), \\
    \Y_{t+1} &= \Y_t - \eta \Rgrad_{\Y} f(\X_t,\Y_t).
\end{align*}
Next, we illustrate how ScaledGD copes with gauge invariance through Example~\ref{example.simple}.
\begin{example}\label{example.scaledgd}
   Consider applying ScaledGD to the loss in Example~\ref{example.simple}. For any $(x,y)$ satisfying $xy=4$, the one-step ScaledGD update $(x^+,y^+)$ renders
    \begin{align*}
        x^+y^+ = 4 - 6\eta + \frac{9}{4}\eta^2 .
    \end{align*}
    This confirms that, under ScaledGD, all gauge-equivalent points are updated coherently to a new equivalence class after one iteration, independent of the chosen factorization.
\end{example}

It is worth noting that ScaledGD does not decompose the Riemannian gradient into vertical and horizontal components, and therefore its update does not fully eliminate potential drift along gauge directions. While explicitly projecting onto the horizontal space eliminates this drift, it typically requires solving a Sylvester equation~\cite{mishra2014fixed}, which is computationally expensive. Nevertheless, ScaledGD already proves highly effective for the one-layer LoRA problem~\eqref{eq.prob-mf}. In particular, the convergence rate in Theorem~\ref{thm.scaledgd} does depend on the condition number $\kappa$, which accounts for its popularity in ill-conditioned matrix factorization and sensing tasks.

In the context of LoRA fine-tuning, ScaledGD-style gauge-aware optimization has been adopted in~\cite{zhang2024riemannian}, where it is further shown to promote stable feature learning. LoRA-Pro~\cite{wang2025lorapro} advances this approach by explicitly solving a Sylvester equation to project the Riemannian gradient onto the horizontal space, which can speed up convergence at the cost of extra per-iteration computations.

\textbf{RefLoRA.} Another method leveraging above viewpoint to tackle gauge invariance is refactored low-rank adaptation (RefLoRA)~\cite{zhang2025reflora}. In contrast to ScaledGD, it relies on a distinct gauge-invariant Riemannian metric to directly handle the gauge drift. For a point $(\X, \Y) \in\bR_\ast^{m\times r}\times\bR_\ast^{n\times r}$, define their Gram matrices as $\bfP_\X :=\X^\top \X$ and $\bfP_\Y:=\Y^\top \Y$. RefLoRA constructs a positive definite matrix
\[
    \bfS := \bfP_\X^{-1/2}\Big(\bfP_\X^{1/2} \bfP_\Y \bfP_\X^{1/2}\Big)^{1/2}\bfP_\X^{-1/2}
\]
which is also known as the (matrix) geometric mean of $\bfP_\X^{-1}$ and $\bfP_\Y$. Using $\bfS$, the gauge-invariant metric is defined as
\begin{equation}\label{eq.reflora-metric}
    \begin{aligned}
    g_{(\X,\Y)} & \big((\bm{\xi}_\X,\bm{\xi}_\Y), (\bm{\varphi}_\X,\bm{\varphi}_\Y)\big) \\
    & := \langle \bm{\xi}_\X \bfS, \bm{\varphi}_\X \rangle_\fro + \langle \bm{\xi}_\Y \bfS^{-1}, \bm{\varphi}_\Y \rangle_\fro
\end{aligned}
\end{equation}
where $(\bm{\xi}_\X,\bm{\xi}_\Y)$ and $(\bm{\varphi}_\X,\bm{\varphi}_\Y)$ are tangential to the manifold at $(\X, \Y)$. With the Riemannian gradient derived in Appendix~\ref{Apdx.reflora}, RefLoRA updates are given by
\begin{subequations}\label{eq.reflora}
    \begin{align}
        \X_{t+1} & = \X_t - \eta \nabla_{\X} f(\X_t,\Y_t) \bfS_t^{-1}, \\
        \Y_{t+1} & = \Y_t - \eta \nabla_{\Y} f(\X_t,\Y_t) \bfS_t.
    \end{align}
\end{subequations}
The main benefit of this metric is that it eliminates the vertical component in a computationally efficient manner, and the Riemannian gradient thus lies entirely in the horizontal space; see Appendix~\ref{Apdx.reflora} for a proof. As a consequence, RefLoRA not only enjoys efficient update equivalence across gauge-related parameterizations, but also yields a steeper decrease in the loss; see numerical tests in Figure~\ref{fig.reflora}, and the following example.

\begin{figure}[t]
	\centering
	\includegraphics[width=0.4\textwidth]{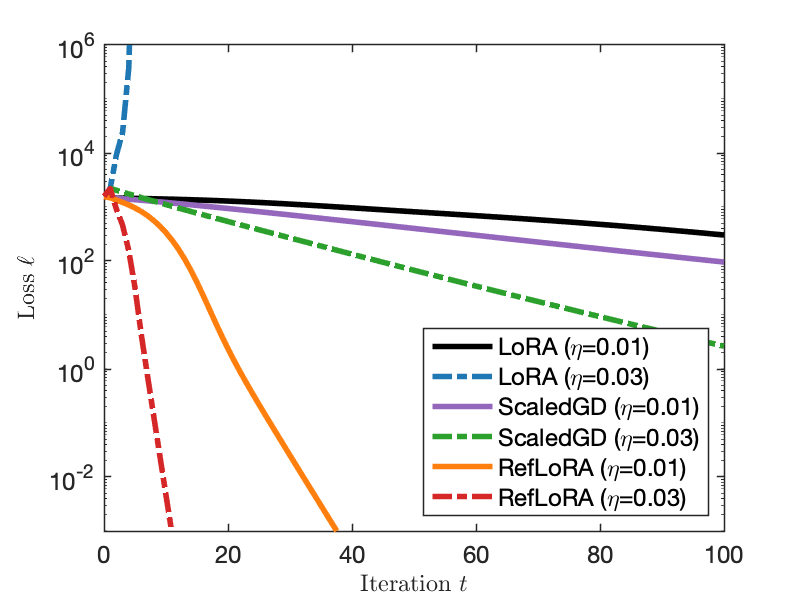}
	\caption{Convergence comparison of LoRA~\cite{hu2021lora}, ScaledGD~\cite{tong2021}, and RefLoRA~\cite{zhang2025reflora} for low-rank matrix factorization~\eqref{eq.prob-mf}.}
	 \label{fig.reflora-converge}
\end{figure}

\begin{example}\label{example.reflora}
    To apply RefLoRA to the loss in Example~\ref{example.simple}, for any $(x,y)$ satisfying $xy=4$, the one-step RefLoRA update $(x^+,y^+)$ renders
    \begin{align*}
        x^+y^+ = (2-6\eta)^2.
    \end{align*}
    Similar to ScaledGD in Example~\ref{example.scaledgd}, RefLoRA maps all gauge-equivalent points with $xy=4$ to the same equivalence class after one step. Notwithstanding, when using the same (typically small) step size $\eta$ as ScaledGD, RefLoRA yields a noticeably larger decrease in the objective. Letting for instance $\eta=1/3$, ScaledGD gives $f(x^+,y^+) =  2.25$, while RefLoRA renders $f(x^+,y^+) = 0$. 
\end{example}

In fact, RefLoRA admits an alternative interpretation: among all pairs in the equivalence class $[(\X, \Y)]$ (i.e., all factorizations representing the same product $\X\Y^\top$), it selects a particular one such that one GD step from this point yields the steepest descent in the objective value~\cite{zhang2025reflora}. This optimal pair $\tilde{\X} = \X \bfS^{1/2}$ and $\tilde{\Y} = \Y \bfS^{-1/2}$ is ``balanced'' in terms of their Gram matrices $\tilde{\X}^\top \tilde{\X} = \tilde{\Y}^\top \tilde{\Y}$. Furthermore, performing Riemannian optimization~\eqref{eq.reflora} with the gauge-invariant metric~\eqref{eq.reflora-metric} is equivalent to standard GD on this balanced pair $(\tilde{\X}, \tilde{\Y})$ in Euclidean space. In Example~\ref{example.reflora}, this corresponds to performing GD along the trajectory of $x=y$. 

When applied to LLM fine-tuning, RefLoRA bypasses the need for solving the Sylvester equation required by LoRA-Pro~\cite{wang2025lorapro}, and thus achieves faster per-iteration runtime. To grasp its efficiency, recall that the Gram matrices $\bfP_\X$ and $\bfP_\Y$ are of size $r \times r$, and hence computing $\bfS$ incurs merely $\mathcal{O}(m+n+r)r$ overhead, which is much smaller than the $\mathcal{O}(m^2 r)$ cost of solving a Sylvester equation.

Note that, although the above approaches are highly effective in practice, they incur additional computational cost, typically requiring at least an inversion of a small $r \times r$ matrix. To shift the tradeoff toward greater computational efficiency, one may instead account for only a subset of the full gauge invariance. A classical choice is scale invariance, corresponding to the equivalence class $\{(c\X, \Y/c) \mid c \in \mathbb{R}_*\}$. Under this weaker symmetry, the matrix $\bfS_t$ in the RefLoRA update \eqref{eq.reflora} can be replaced by
$\frac{\|\Y_t\|_\fro}{\|\X_t\|_\fro} \I$, yielding a cheaper update \cite{zhang2025reflora}.

There are also other approaches that explicitly target gauge invariance in LoRA optimization. For example, LoRA-RITE~\cite{yen2025lorarite} factors out the “magnitude” components of $\X_t$ and $\Y_t$ via polar decomposition, thereby focusing updates on the subspace directions. ScaLoRA~\cite{scalora} copes with gauge invariance under the HiRA~\cite{huang2025hira} setting. 
Moreover,~\cite{li2024implicit} shows that for scale-invariant two-factor problems (including rank-1 LoRA), sharpness-aware minimization (SAM)~\cite{foret2020sharpness,li2024enhancing} implicitly reduces the imbalance $\big| \|\X|_\fro^2-\|\Y\|_\fro^2 \big|$. While this does not eliminate the full $\mathsf{GL}(r)$ gauge invariance, it suggests that part of SAM’s benefit may stem from implicitly regularizing a subset of gauge degrees of freedom.

\subsection{Optimization of SVD-type parameterization}
Previous subsections focused on advanced optimization schemes for the standard bilinear (BM) parameterization used in vanilla LoRA. We now turn to optimization schemes tailored to other architectures discussed in Section~\ref{Sec.architectures}.
This section focuses on SVD-type parameterizations, using PoLAR~\cite{kai2025} as a representative example. Applying this parameterization to the testbed problem~\eqref{eq.prob-mf} yields
\begin{align}\label{eq.polar-mf}
	\min_{\U, \V, \bfTheta} & \frac{1}{2} \| \U \bfTheta \V^\top - \A \|_\fro^2 \\
    \text{s.t.}~~~ & \U \in \st{(m, r)}, \V \in \st{(n, r)}	, \bfTheta \in \bR^{r \times r} \nonumber
\end{align}
where $\st(m,r):= \{ \U \in \bR^{m \times r} \mid \U^\top \U  = \I_r\}$ stands for the Stiefel-manifold; cf. Section~\ref{Sec.SVD-parameterization}. 
Problem~\eqref{eq.polar-mf} also offers an additional explanation for leaving $\bfTheta$ unconstrained. Indeed, forcing $\bfTheta$ to be diagonal can slow down convergence, because it introduces non-strict saddles, whose Hessian has both $0$ and negative eigenvalues, making them difficult to escape for first-order optimization algorithms~\cite{levin2024}.

As far as convergence, rather than working with quotient manifold as in the previous subsection,~\cite{kai2025} views the constraint sets as embedded submanifolds of Euclidean space, and applies retraction-based Riemannian optimization. Let $\E_t$ and $\F_t$ denote the Riemannian gradients with respect to $\X_t$ and $\Y_t$, respectively. Using the polar retraction, the Riemannian gradient descent (RGD) updates take the form
\begin{subequations}\label{eq.polar-updates}
\begin{align}
    \X_{t+1} &= (\X_t - \eta \E_t) (\I_r + \eta^2 \E_t^\top \E_t )^{-1/2} \\
    \Y_{t+1} &= (\Y_t - \eta \F_t) (\I_r + \eta^2 \F_t^\top \F_t )^{-1/2}.
\end{align}
In addition, $\bfTheta$ is updated by a standard GD step
\begin{align}\label{eq.update-theta}
	\bfTheta_t =  \bfTheta_{t-1} - \gamma \nabla_{\bfTheta} f(\X_{t+1}, \bfTheta_t, \Y_{t+1}).
\end{align}
\end{subequations}

Under this alternating update between “direction” $(\X_t, \Y_t)$ and “magnitude”  $\bfTheta_t$,~\cite{kai2025} shows that the required iterations for convergence scales inversely with the rank $r$.

\begin{theorem}[\cite{kai2025}, informal]\label{thm.polar}
    If $m > n $ and $r > r_A$, the updates~\eqref{eq.polar-updates} for~\eqref{eq.polar-mf} converges to $\frac{1}{2} \| \X_T \bfTheta_T \Y_T^\top - \A \|_\fro^2 \leq \epsilon$ within ${\cal O} \big( \frac{  r^4  r_A^2  \kappa^4 }{(r - r_A)^8}  + \frac{ r^2 r_A \kappa^4 }{ (r - r_A)^4} \log\frac{1}{\epsilon} \big)$ iterations.
\end{theorem}
Recent work by~\cite{wei2025benefits} further shows that, in the matrix-sensing setting of~\eqref{eq.prob-sensing}, increasing the rank $r$ not only accelerates optimization, but can also reduce sample complexity: the number of measurements required to recover $\A$ decreases proportionally to $1/r$. 
In practical LLM fine-tuning however, PoLAR does not strictly follow the alternating scheme in~\eqref{eq.polar-updates}, and it replaces explicit retractions with retraction-free updates, e.g., the landing family of methods~\cite{ablin2022landing,gao2022landing,schechtman2023landing}, to improve GPU efficiency. Nonetheless, Theorem~\ref{thm.polar} offers a plausible explanation for why PoLAR tends to benefit from larger $r$.

Apart from PoLAR, the optimization of many other matrix-based parameterizations (e.g., FedPara or KronA) is still not well understood. These architectures can introduce extra symmetries in addition to the standard gauge invariances, which may greatly affect optimization dynamics. We conjecture that explicitly analyzing these architectural properties holds strong potential for further improving downstream performance.

\subsection{Optimization for tensor-based parameterization}
In general, advanced optimization for tensor-based adapters remains a largely underexplored frontier. We highlight several promising directions and open questions below. 

The SP community has developed a rich algorithmic toolbox for (approximate) tensor decomposition. Many tensor problems are traditionally tackled with alternating schemes, most notably alternating least squares (ALS)~\cite{zachariah2012alternating}, block coordinate descent (BCD)~\cite{schizas2012covariance}, and splitting algorithms such as alternating direction method of multipliers (ADMM)~\cite{Giannakis2016}. In contrast, tensorized-LoRA adapters have not yet fully embraced these paradigms, plausibly due to the prohibitive cost of performing multiple backpropagation passes per iteration at the scale of LLMs. How to adapt the analytical insights and practical strengths of these methods to modern fine-tuning pipelines remains an open challenge.

Turning to gradient-based methods in SP, existing theoretical understanding of tensor factorizations is still limited relative to the matrix case. For CP factorization,~\cite{ge2021understanding} shows that a slightly modified gradient flow (GF, the continuous-time counterpart of GD) can solve certain orthogonally decomposable tensor problems. Regarding Tucker parameterizations,~\cite{dong2023fast,liu2024low} have developed ScaledGD variants that aim to handle the richer gauge invariance therein, while~\cite{xia2019polynomial} leverages optimization over the Grassmann manifold to exploit the Tucker-type structure. As for TT, the local convergence of Riemannian gradient descent is established in~\cite{cai2022provable}. At a global level, optimization with TT parameterizations is closely related to deep linear networks, and may therefore suffer from similar depth-related bottlenecks as the tensor order grows~\cite{shamir2019exponential}. Recent progress~\cite{zhang2026ancre} alleviates such bottlenecks in deep linear networks, and may thus also offer useful insights here.
So far, a limited number of these optimization tools have been adopted in practical tensorized adapters for fine-tuning LLMs, leaving substantial room for further innovation.

Last but not least, a critical open direction is to understand whether the parameter savings from exploiting cross-layer structure (see Section~\ref{Sec.tensorized}) come with hidden costs. This concern originates from SP: many tensor analogues of standard matrix concepts, including spectral norm and singular values, are NP-hard to compute~\cite{hillar2013most}. Additionally, tensor problems often exhibit statistical–computational tradeoffs. For example, in spiked tensor PCA~\cite{arous2019landscape}, gradient-based methods can require a substantially higher signal-to-noise ratio (SNR) than the information-theoretic threshold. It remains unclear whether analogous optimization bottlenecks arise for tensorized adapters in LLM fine-tuning, or whether these classical concerns in SP carry over to LLM settings, considering that fine-tuning only aims for good suboptimal solutions rather than the exact one.

\subsection{Summary of efficient optimization}
While the preceding discussion focuses on strategies that can be translated into practical algorithms, a growing body of work aims to uncover the fundamental principles governing LoRA’s success and identify broader optimization trends.
For instance, it is analyzed in~\cite{zeng2023expressive} that for fully connected networks LoRA can be highly expressive when the rank $r$ is sufficiently large. From an optimization and generalization perspective,~\cite{jang2024lora} shows that when $r$ grows with the square root of data size, LoRA in the NTK regime avoids spurious local minima, and enjoys favorable generalization guarantees.

For complementary perspectives centered on implicit regularization of optimization dynamics of GD/GF, cf.,~\eqref{eq.prob-mf} and~\eqref{eq.prob-sensing}~\cite{du2018,jiang2023,wei2026}. In this context, regularization refers to the biases induced by the optimization procedure itself, even in the absence of any explicit regularizer. While the underlying mathematical analyses are often highly technical, the core challenge remains practical: how to reliably and empirically exploit them in real-world fine-tuning regimes, including choosing architectures and optimizers to induce desirable implicit biases, remains an open question.

Lastly, note that most existing optimization techniques are developed for the BM (or vanilla LoRA) parameterization. Theory and practical algorithms for alternative parameterizations, such as SVD-type, high-rank, as well as tensor-based parameterizations, remain comparatively underexplored. We expect substantial room for further progress in these directions. 

\section{Broader scope of LoRA} \label{Sec.applications}

Having discussed adapter models and tailored optimization strategies, this section deals with LoRA thrusts beyond the standard supervised fine-tuning on downstream tasks.

\subsection{Broader fine-tuning with LoRA}

\begin{figure}[t]
	\centering
	\includegraphics[width=0.48\textwidth]{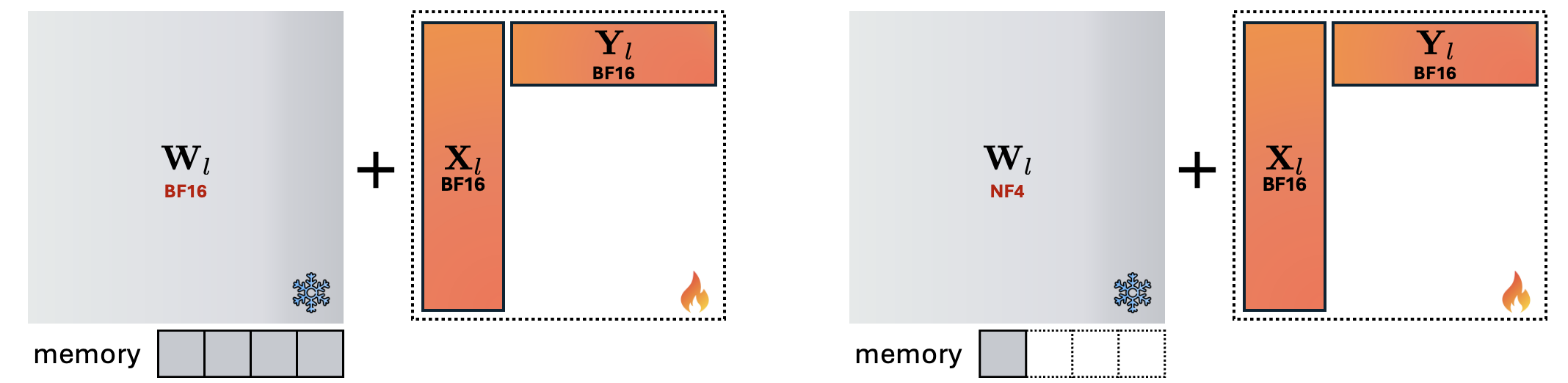}
	\caption{A comparison of LoRA (left) and QLoRA (right).}
	 \label{fig.qlora}
\end{figure}

LoRA improves LLM performance by adapting a pre-trained model to learn a target data distribution for downstream tasks. By interpreting ``downstream tasks'' and ``data distribution'' broadly, it already reveals a richer ecosystem of applications. 
 
\textbf{LoRA accounting for quantization error.}
LLMs are typically pre-trained and fine-tuned in FP16 or BF16 numerical representations. However, serving models at these precisions may not be scalable in cloud settings, especially under millions of concurrent user requests. Moreover, autoregressive generation requires repeated forward passes, which can lead to substantial end-to-end latency.
A standard remedy is quantization, which runs the model at a lower precision (e.g., FP8/INT8 or 4-bit) to reduce memory footprint and improve inference throughput. However, quantization error can degrade downstream performance, that is noticeable at 4-bit precision. For instance, applying GPTQ 4-bit quantization~\cite{frantar2022gptq} to LLaMA-7B leads to roughly a 1-point performance drop in the 5-shot MMLU benchmark~\cite{hendrycks2021measuring}.

LoRA has been adopted as a lightweight approach to compensate for quantization error. QLoRA~\cite{dettmers2024qlora} for example, enables 4-bit quantization (NF4) with minimal performance loss. During fine-tuning, the quantized base model ${\cal Q}(\W_l)$, and the LoRA factors $\{\X_l, \Y_l\}_l$ are learned so that the low-rank updates jointly account for quantization error and adapt the model to the downstream data
\begin{align}\label{eq.prob-qlora}
    \min_{\X_l,\Y_l} f(\{ {\cal Q}(\W_l) + \X_l \Y_l^\top \}_l).
\end{align}
Here ${\cal Q}(\W_l)$ denotes the frozen quantized pre-trained weights, while $\X_l$ and $\Y_l$ are trained and stored in FP16/BF16. 
This pipeline offers two benefits: i) the LoRA updates can compensate for quantization-induced errors; and ii) keeping the base model quantized alleviates the memory bottleneck at fine-tuning; see Figure~\ref{fig.qlora} for an illustration.

LoftQ~\cite{li2023loftq} aims at quantization-aware fine-tuning by constructing a quantized base model whose quantization error is more amenable with a low-rank correction. Specifically, \emph{before} fine-tuning, it solves the per-layer $l$ optimization 
\begin{align} \label{eq.loftq}
    \min_{{\cal Q}(\W_l), \X_l, \Y_l} \| \W_l - {\cal Q}(\W_l) - \X_l\Y_l^\top \|_{\fro}^2.
\end{align}
To handle the discrete nature of quantization, LoftQ alternates between updating the quantized weight ${\cal Q}(\W_l)$ and refitting the low-rank factors $(\X_l, \Y_l)$.
Letting $\Q_{l,t}$ denote the quantized model weight at iteration $t$, LoftQ update takes the form
\begin{align*}
    \Q_{l,t} & = {\cal Q}(\W_l - \X_{l,t-1}\Y_{l,t-1}^\top) \\
    \X_{l,t}, \Y_{l,t} & \in \argmin_{\X ,\Y} \|  \W_l - \Q_{l,t} - \X_l\Y_l^\top\|_\fro^2.
\end{align*}
The resulting $\{ {\cal Q}(\W_l), \X_l, \Y_l \}$ are then used to initialize QLoRA fine-tuning in~\eqref{eq.prob-qlora}.

IR-QLoRA~\cite{qin2024accurate} improves quantized fine-tuning from a complementary angle. Rather than calibrating quantization by minimizing an $\ell_2$ reconstruction loss in~\eqref{eq.loftq}, it quantizes a pre-trained LLM by minimizing an information-based loss. It further introduces an additional, largely parameter-free feature path via pooling/averaging input features. Together, these components stabilize ultra-low-bit ($<4$-bit) quantized fine-tuning.
QA-LoRA~\cite{xu2023qa} targets quantized deployment. Naively merging ${\cal Q}(\W)+\X\Y^\top$ produces an FP16/BF16-weighted model, which cannot be directly deployed in low-bit form. To address this, QA-LoRA advocates a pooling-based LoRA variant that enables low-bit (e.g., INT4) merging of the adapted weights while preserving accuracy.

\textbf{LoRA for long-context adaptation.} 
LLMs are often pre-trained with a relatively short context window, typically 2K to 8K tokens. For example, LLaMA2-7B supports a 4K context length~\cite{llama2}. However, many applications, such as long-document understanding, retrieval over large corpora, and multi-turn agents, require context lengths far beyond the pre-training limit~\cite{bai2024longbench,shinn2023reflexion}. Viewing this extension of context length as a distribution shift, one can apply fine-tuning to mitigate performance degradation. While the core remedy involves efficient attention mechanisms, which go beyond the scope of this paper, LoRA has been used as a lightweight tool to reduce the error induced by context extension. Empirically, LongLoRA \cite{chen2023longlora} reports extending LLaMA2-7B from 4K to 100K context with the aid of LoRA.

\subsection{Pre-training LLMs with LoRA}
While LoRA was originally developed for PEFT, recent work has also utilized it for pre-training to lower hardware barriers. Since LoRA parameterizes updates in a rank-$r$ subspace, a central challenge is how to recover the high-rank weight updates for pre-training with low-rank matrices $\Delta \W_l$.

One idea is to combine low rank with sparsity to obtain effectively higher-rank updates, in the spirit of the low-rank-plus-sparse models discussed in Section \ref{Sec.high-rank}. SLTrain and LOST~\cite{han2024sltrain,li2025lost} fix the sparsity pattern before training to avoid the cost of storing a full-sized dense matrix. While SLTrain samples the nonzero entries at random, LOST enforces a randomly chosen column-wise sparsity pattern, and inserts an activation function between the two LoRA factors to account for nonlinearity. 

Another research line learns a sequence of low-rank updates $\{\Delta \W^{s}_{\ell}\}$. Although each $\W^{s}_{\ell}$ is low-rank, their accumulation $\sum_s \Delta \W^{s}_{\ell}$ can be a high-rank matrix. 
This idea is exploited in ReLoRA~\cite{lialin2023relora}, which repeatedly applies LoRA, merges the update into the base weights, and restarts a new low-rank adapter during pre-training. 
CoLA~\cite{xia2024chain} further links this mechanism to Frank Wolfe method or greedy subspace learning~\cite{jaggi2013revisiting,li2021momentum}, where training proceeds through a sequence of subspaces, and each LoRA stage can be viewed as approximately solving an FW subproblem. 
Notably, the idea of learning low-rank structure in a successive, stage-wise manner also has deep roots in SP, for example via deflation~\cite{ge2021understanding}. Related viewpoints also appear in incremental learning~\cite{li2020towards} and, in certain nonconvex settings, saddle-to-saddle dynamics~\cite{jacot2021saddle,wei2026}.

There are also approaches that rely on more advanced optimization to learn purely low-rank parameterizations. Motivated by the low-rank lottery ticket hypothesis~\cite{schotthofer2022low}, which suggests that a suitably chosen low-rank network can match the performance of a large dense model, LORO~\cite{mo2025parameter} replaces the $m\times n$ linear weight matrix $\W_l$ with a low-rank factorization $\X_l\Y_l^\top$, and trains the factors directly through Riemannian optimization. It is observed that this low-rank model can match the performance of full-parameter pre-training for a LLaMA-1B model, while providing memory and runtime savings. 

A remotely related line of research that leverages low-rank structure for pre-training is GaLore~\cite{zhao2024galore}. Unlike low-rank parameterizations, GaLore keeps the model weights intact, but maintains a low-rank representation of the training dynamics. Concretely, it projects gradients into a low-rank subspace, performs the optimizer update (including momentum and second-moment states) in that compressed space, and then lifts the resulting update back to the original parameter space. By storing and updating the compressed gradients and optimizer states, GaLore achieves substantial memory savings.

\subsection{Serving LLMs with LoRA}
\label{Sec.serving}
LoRA can also benefit LLM serving. When deploying a model for a single downstream task, the adapted weights $\W_l+\X_l\Y_l^\top$ can be pre-merged into a single matrix. As a result, serving a single LoRA-adapted model incurs no additional inference latency, as opposed to other PEFT methods such as prefix tuning~\cite{li2021prefix}.

However, LoRA’s serving advantage is most highlighted when managing multiple downstream tasks. Consider the multi-domain serving setup where the users are from $N$ domains, e.g., coding, mathematics, and health. A straightforward solution would be fine-tuning $N$ separate LLMs, one per domain, which requires storing $N$ full models in memory for serving. This can be prohibitive when $N$ or the base model size grows. In contrast, LoRA allows for storing a single shared base model alongside $N$ lightweight adapters, each specialized to a particular domain. The total storage requirement boils down to $N S_a + S_b$, where $S_a$ and $S_b$ are the sizes of the base model and LoRA adapter; typically $S_a \ll S_b$. Thus, LoRA enables multi-domain serving in a highly memory-efficient manner.
Building upon this, several works have further improved the efficiency of multi-LoRA serving from various perspectives. SLoRA~\cite{sheng2024slora} leverages paging techniques to optimize memory access for faster inference. LoRA-Inlaid~\cite{xia2024efficient} combines LoRA with quantization to efficiently serve multiple adapters. More recently, EdgeLoRA~\cite{shen2025edgelora} demonstrates LoRA-based serving in edge scenarios with tight resource constraints. 

These benefits are not limited to BM-based LoRA parameterizations. For an SVD-style adapter $\Delta \W_l=\U_l\bfSigma_l\V_l^\top$, one can merge $\U_l \bfSigma_l$ into a single matrix, making it equivalent to serving with standard LoRA. For tensor-based parameterizations, efficient serving remains less explored. Nevertheless, in the worst case, one can always ``unpack'' the tensor along the layer dimension and obtain $L$ per-layer low-rank updates. 
Lastly, for low-rank-plus-sparsity parameterizations, while the low-rank component can be  efficiently deployed, the sparse component requires leveraging specialized hardware, such as NVIDIA’s sparse engine (with $N:M$ sparsity), to accelerate computation, which remains an open area of research.

Another important application of LoRA is batched serving, where $K$ users simultaneously request a shared base model but $K$ different LoRA adapters $\Delta\W_l^k=\X_l^k(\Y_l^k)^\top$ for their respective tasks. Let the input to the $l$-th layer be $\mathbf{z}_l^k \in \bR^{m}$. Batching concatenates inputs to from matrix $\mathbf{Z}_l :=[\mathbf{z}_l^1, \ldots, \mathbf{z}_l^K]^\top$, so that a single general matrix multiply (GeMM) can process all users' requests in parallel, avoiding $K$ separate kernel launches and improving throughput by amortizing overhead. With standard LoRA, batching applies cleanly to the pre-trained weight $\W_l$, but user-specific adapters cannot be fused into the same batched GeMM. Consequently, the forward pass becomes
\[
    \underbrace{\mathbf{Z}_l\W_l}_{\text{batching}} + \text{row-stack}\big[\underbrace{(\mathbf{z}_l^1)^\top \X_l^1(\Y_l^1)^\top, \ldots, (\mathbf{z}_l^K)^\top \X_l^K(\Y_l^K)^\top}_{\text{cannot be fused}} \big]
\]
where $\text{row-stack}[\cdot]$ stacks $K$ row vectors into a matrix.
FastLoRA~\cite{wen2023batched}, which predates HiRA from a different motivation, addresses this issue by reparameterizing adaptation with~\eqref{eq.hira} so that the forward pass becomes
\begin{align}\label{eq.fast-lora}
    & \mathbf{Z}_l\W_l + \\
    & \big[(\mathbf{z}_l^1)^\top (\W_l \circ \X_l^1(\Y_l^1)^\top), \ldots, (\mathbf{z}_l^K)^\top (\W_l \circ \X_l^K(\Y_l^K)^\top )\big]. \nonumber
\end{align}
This form admits a GEMM-friendly batch implementation; see also Appendix~\ref{Apdx.fastlora}. The tradeoff is that FastLoRA incurs extra $\mathcal{O}(mnr)$ computations compared to standard LoRA, making it most advantageous in the large-$K$ regime, where throughput gains from improved batching outweigh the extra arithmetic.

\subsection{Reinforcement Learning with LoRA}
At a high level, LLM alignment refers to post-training procedures, such as supervised fine-tuning and preference-based optimization, that steer a pretrained model toward behavior preferred by humans.
Recent evidence suggests that LoRA is surprisingly competitive for the post-training and alignment of LLMs via reinforcement learning (RL). For instance, TiNA~\cite{wang2025tina} demonstrates that applying LoRA-based RL to a 1.5B parameter base model yields substantial reasoning improvements at a small cost of approximately \$9 USD. Similarly,~\cite{sun2023exploring} investigated LoRA-based proximal policy optimization (PPO) alignment for LLaMA-7B. It finds that, LoRA-based RL from human feedback (RLHF) can match or even slightly outperform full fine-tuning while drastically reducing memory consumption. Recently,~\cite{schulman2025lora} reports that LoRA performs comparably to full-parameter updates in RL settings even with minimal ranks. Through an information-theoretic lens it is argued that RL updates may inherently require relatively low effective capacity, which can be readily provided by lightweight LoRA adapters.

The preceding four subsections collectively demonstrate that LoRA is a versatile tool applicable throughout the entire lifecycle of an LLM, from initial pre-training to downstream fine-tuning, alignment, and deployment/serving. 
By offering memory efficiency at every stage of the pipeline, LoRA effectively democratizes large-model development and deployment by lowering hardware barriers.

\subsection{Concept injection and merging with LoRA}

\begin{figure}[t]
	\centering
	\includegraphics[width=0.48\textwidth]{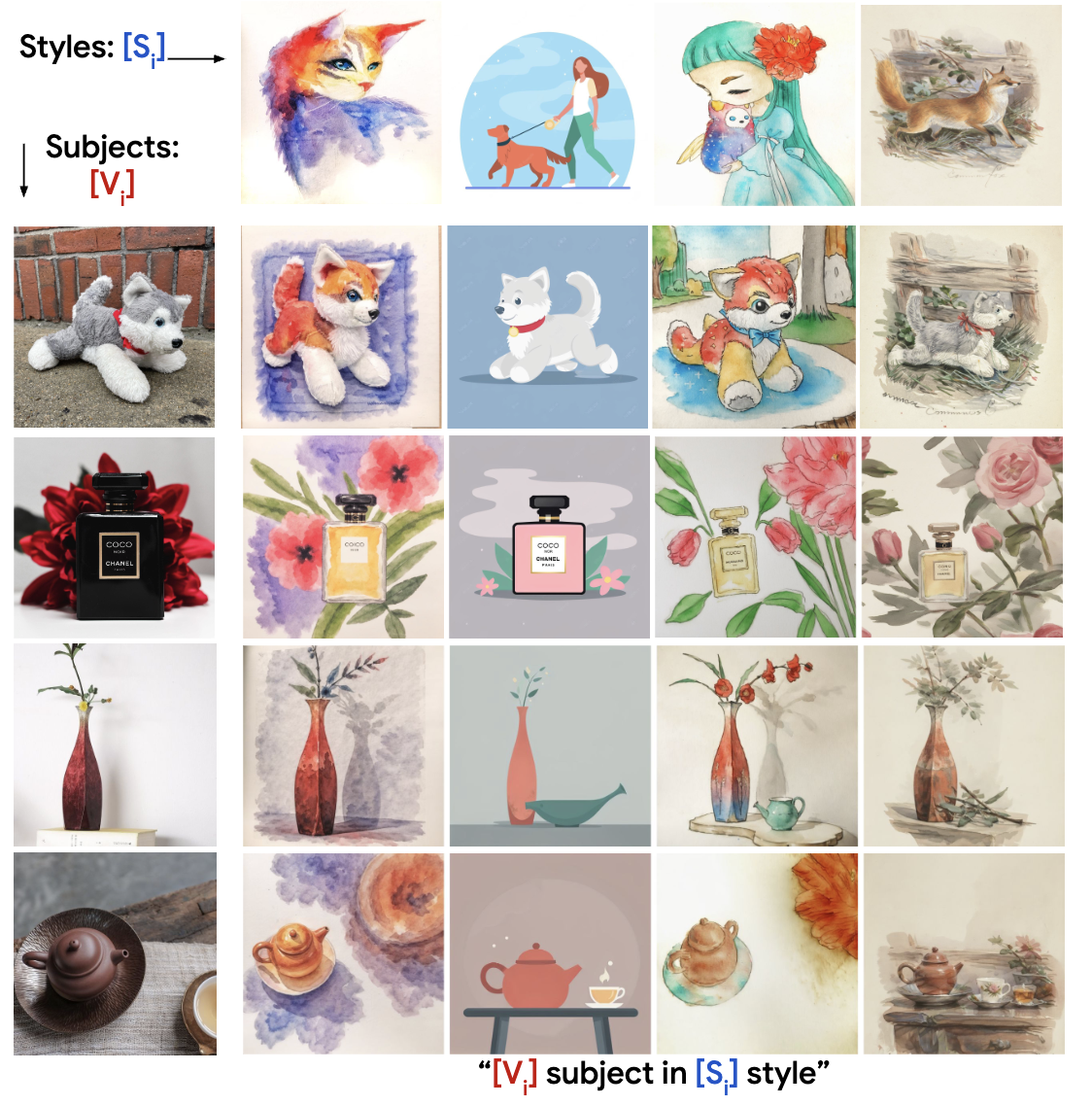}
	\caption{Mix of multiple (two here) concepts. Figure taken from ZipLoRA~\cite{shah2024ziplora}. }
	 \label{fig.moe-LoRA}
\end{figure}

While the methodologies discussed are broadly applicable, this subsection focuses specifically on text-to-image generalization for ease of visualization. 
The additive nature of LoRA is often interpreted as a mechanism for encoding ``add-on knowledge.''  This paradigm has proven highly effective in personalized text-to-image generation, where a diffusion model is adapted to generate images of novel subjects, such as a user’s own dog, using only a handful of (typically $\leq 10$) reference examples~\cite{ruiz2023dreambooth}. Since the pre-trained model lacks prior exposure to these specific instances, attaching and fine-tuning a LoRA module provide a compact, modular representation of the new ``concept.'' 

A natural application extending this interpretation of additive concept encoder, is to combine $N$ separately trained LoRAs, say $\{\X_l^n,\Y_l^n\}_{n=1}^N$, to enable joint generation of multiple learned concepts. A representative example is subject–style fusion, where individual LoRAs encoding subject identity and artistic style are ``merged'' to produce diverse stylized outputs; see Figure~\ref{fig.moe-LoRA} from ZipLoRA~\cite{shah2024ziplora}.
For such multi-concept generation, the central challenge lies in how to combine them to achieve high-fidelity generation. Interestingly, a simple linear composition,
$\Delta \W_l=\sum_n w_l^n \X_l^n(\Y_l^n)^\top$,
is often sufficient for high-quality results in many tasks~\cite{huang2023lorahub,han2023svdiff}. 
Although the composition is linear at the weight level, the cumulative effect remains nonlinear due to the network depth.
The mixing coefficients $w_l^n$ are optimized in a layer-wise manner to handle potential semantic conflicts and context dependence of concepts. For instance, the word ``football'' may refer to soccer or American football depending on whether the desired context is European or American. 

These conflict-resolution strategies draw clear parallels to the established literature on ensemble learning~\cite{shen2019random,lu2020ensemble,lu2023surrogate}, which also has a long-standing history in SP, as well as to multi-task learning~\cite{sener2018multi}.
For example, MoLE~\cite{wu2024mixture} adopts a mixture-of-experts (MoE) framework, employing a gating network to dynamically predict the mixture coefficients $\{w_l^n\}$. 
Alternatively,~\cite{gu2023mix} determines $\{w_l^n\}$ by minimizing the discrepancy between outputs of the merged model and the individual adapters on a calibration set. Concretely, let $\Delta \W_l^n = \X_l^n(\Y_l^n)^\top$, and define the merged update $\Delta \W_l(\mathbf{w}_l) := \sum_{n=1}^N w_l^n \Delta \W_l^n$ and calibration dataset $\mathcal{D}$. The weights are optimized via
\begin{align*}
    \min_{{\mathbf{w}_l}} \mathbb{E}_{\x \in {\cal D}} \sum_{n=1}^N
    \| \big(\W_l + \Delta \W_l(\mathbf{w}_l) \big) \x - \big( \W_l + \Delta \W^n \big) \x \|^2.
\end{align*}
Other work has investigated learning mixing coefficients via contrastive objectives~\cite{simsar2025loraclr}, token- or region-wise weighting schemes~\cite{yang2024lora}, or auxiliary hypernetworks that generate the weights on the fly~\cite{shenaj2025lora}.

\begin{figure*}[t]
	\centering
	\includegraphics[width=0.9\textwidth]{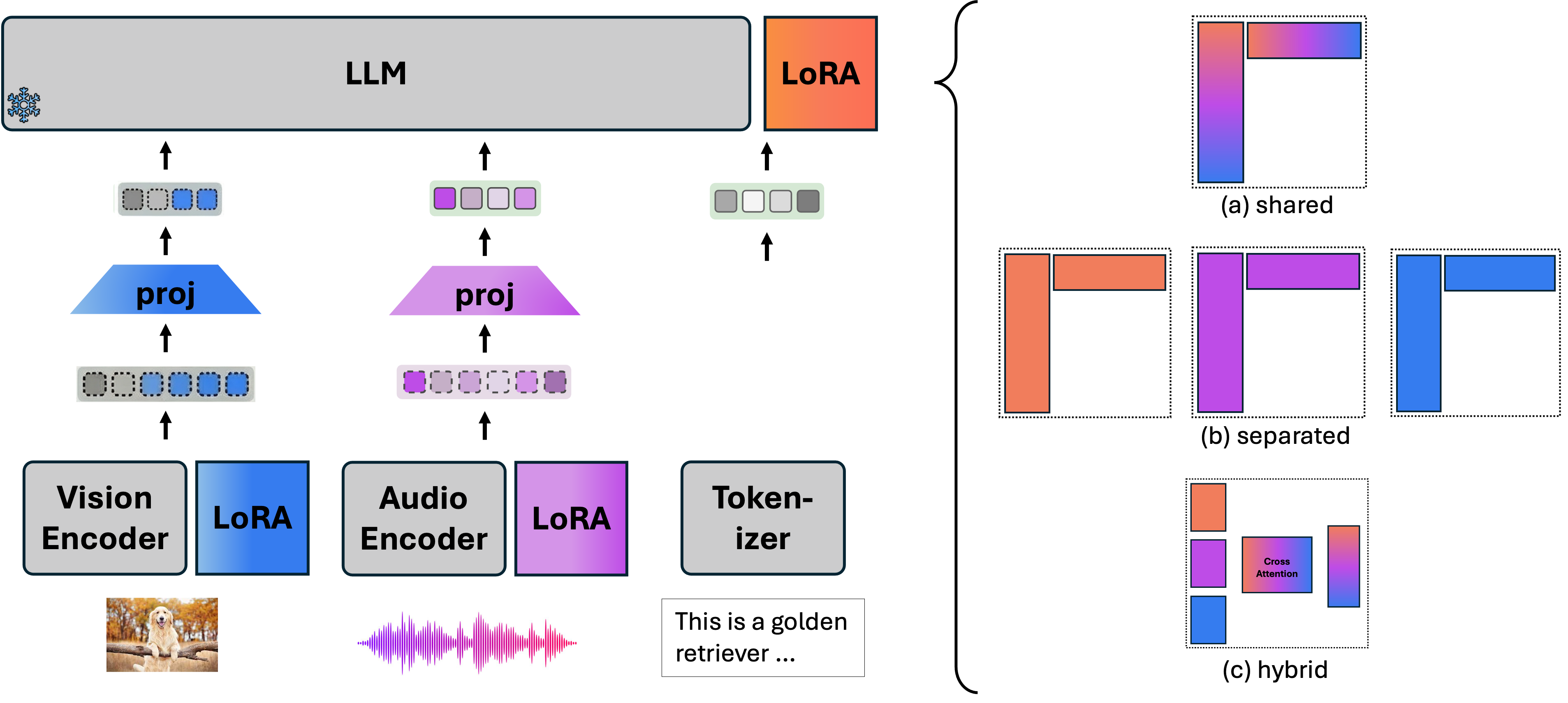}
	\caption{LoRA for fine-tuning multi-modal models.}
	 \label{fig.multi-modality}
\end{figure*}

\subsection{Responsible fine-tuning with LoRA}
As LLMs continue to raise noticeable societal concerns, it is natural to ask whether LoRA fine-tuning mitigates or exacerbates issues including safety, privacy, and bias.

\textbf{LoRA for uncertainty quantification.}
Uncertainty quantification is indispensable for safety-critical scenarios such as AI-assisted medical diagnosis and legal reasoning. While LLMs often exhibit reasonable calibration during pre-training~\cite{kadavath2022language}, meaning that the model’s stated confidence roughly matches its actual correctness frequency,
fine-tuning on domain-specific tasks often makes the model over-confident, especially when the dataset~\cite{jiang2021calibrating,tian2023just} is small.
This motivates the pursuit of Bayesian approaches that offer model uncertainty quantification. 
Laplace-LoRA~\cite{yang2023bayesian} first performs standard LoRA fine-tuning to obtain a maximum a posteriori (MAP) estimate the LoRA parameters, and then uses a Laplace approximation to the posterior to assess uncertainty. For scalability, LoRA is viewed as a two-layer linear network, allowing Kronecker-factored (KFAC) block approximation for the Fisher information matrix~\cite{martens2015optimizing}. An additional low-rank approximation is introduced to keep the KFAC matrices tractable.
BLoB~\cite{wang2024blob} performs Bayesian adaptation by jointly learning the mean and covariance of a low-rank variational distribution via backpropagation. By treating $\X_l$ as a variational random variable, it optimizes the resulting variational objective end-to-end. 
In addition, ScalaBL~\cite{samplawski2025scalable} and C-LoRA~\cite{rahmati2025c} shift the Bayesian/variational treatment to a much smaller set of parameters. In particular, they leverage SVD-type models, and only follow a Bayesian approach for the singular value $\bfSigma_l \in \bR^{r \times r}$, to reduce variational overhead, and thus enhance scalability. 

\textbf{LoRA for safety-aware fine-tuning.}
There are two primary directions along this line of research. One utilizes LoRA to probe the fragility of alignment mechanisms. For example,~\cite{lermen2023lora} shows that adversarial LoRA or QLoRA fine-tuning can substantially erode the refusal behavior of aligned chat models.
The other direction leverages LoRA for safety fine-tuning. Safe LoRA~\cite{hsu2024safe} presumes that alignment fine-tuning identifies ``safety directions'' in weight space. It constrains LoRA weights to stay within those directions $\V_l := \W_l^{\text{aligned}}-\W_l^{\text{unaligned}}$. As an example, the aligned model could be LLaMA-3.1-Chat, and the unaligned model is the corresponding base LLaMA-3.1. The associated safety projection matrix is
\begin{align}
    \mathbf{C}_l = \V_l(\V_l^\top \V_l)^{-1}\V_l^\top .
\end{align}
Given a LoRA update $\Delta \W_l=\X_l\Y_l^\top$, Safe LoRA constructs the projected update $\Delta \hat{\W}_l := \mathbf{C}_l\Delta \W_l$. The original update $\Delta \W_l$ is replaced by $\Delta \hat{\W}_l$ whenever the projected update is sufficiently misaligned with the original update; that is, $\langle \Delta \hat{\W}_l,\Delta \W_l\rangle \leq \tau \|\Delta \hat{\W}_l\|_\fro \|\Delta \W_l \|_\fro$
for a threshold $\tau$. Subsequent work~\cite{li2025salora} replaces $\mathbf{C}_l$ with a data-dependent formulation.
Despite these promising results, there remains room for improvement in safety-aware PEFT. For example, establishing theoretical robustness guarantees would provide stronger evidence that such methods improve safety.

\textbf{LoRA for privacy-aware fine-tuning.}
It is shown in~\cite{ran2025lora} that LoRA fine-tuning does not inherently prevent privacy leakage, and hence membership inference attacks remain effective. To address this issue, differential privacy (DP)–based optimizers inject carefully calibrated noise into gradient updates, providing a principled way to limit such leakage. Notably,~\cite{yu2021differentially} suggests that DP optimizers such as DP-SGD can offer better privacy–utility tradeoffs when applied to PEFT than full-parameter update.
Building on this,~\cite{sun2024improving} observes that LoRA’s bilinear parameterization can cause DP noise to interact destructively with the two-factor updates. To mitigate this interference, the $\X_l$ factors are frozen per layer, and only the remaining components are trained. Together, these observations suggest that gauge invariance discussed in Section~\ref{Sec.gauge-invariance} may be a natural complement to private LoRA, potentially improving both training stability and model utility.

\subsection{Applications to multi-modal LLMs}
Multi-modal foundation models often recruit an LLM as the backbone, mapping vision or audio inputs into ``soft prompts'' or prefix-like embeddings via an encoder plus a learnable projection module. LoRA can be applied to jointly adapt the modality encoders and the LLM to improve cross-modal alignment. For instance,~\cite{cappellazzo2025large} fine-tunes video encoders in tandem with the LLM to improve representation parity between the two components; see Figure~\ref{fig.multi-modality} for an example.

There are several design choices for fine-tuning the backbone LLM. A straightforward strategy is to treat tokens from different modalities uniformly, and apply LoRA in the standard modality-agnostic way~\cite{cappellazzo2025large}. Conversely, Phi-4-multi-modal~\cite{abouelenin2025phi} leverages modality-specific LoRA modules coupled with dedicated routing mechanisms. This allows the model to switch between different modes, such as text-only, vision-plus-text, or audio-plus-text, while minimizing the burdensome cross-modal interference. MoKA~\cite{wei2025moka} proposes a hybrid approach that jointly models unimodal and cross-modal adaptation. It assigns each modality a unique LoRA factor $\{\X_l^m\}_m$, mixes modality-specific features via cross-attention, and subsequently applies a shared $\Y_l$ factor across modalities; see Figure~\ref{fig.multi-modality}. It remains unclear whether fully shared, modality-specific, or hybrid parameterizations offer the desired tradeoff in practice. The answer may depend on the modality mix, data regime, and target use case, leaving substantial room for further research towards principled multi-modal adapter designs. 

\subsection{Additional directions}
As LoRA has gained widespread popularity and vast collections of LoRA checkpoints have become publicly available~\cite{zhao2024retrieval} (e.g., on HuggingFace), LoRA weights are increasingly being viewed as a new form of ``data.'' Inspired by generative modeling of images, recent work explores meta-models capable of directly generating LoRA weights. Thanks to the relatively low parameter count of LoRA adapters, several studies have employed diffusion-based generators to synthesize LoRA updates~\cite{wu2024difflora,shao2025context}. Further,~\cite{putterman2024learning} incorporates the gauge invariance inherent in LoRA’s bilinear parameterization model and generation procedure. 
While these approaches are intriguing, they have so far been studied for task-specific personalization, and have not yet been demonstrated at larger scales of general-purpose foundation models. The generated adapters typically underperform those obtained via direct fine-tuning. Nevertheless, these preliminary results are encouraging: they achieve non-trivial performance across several settings and point toward toward faster, more reusable, and potentially ``amortized'' model adaptation.

\section{Research Outlook}
The adapter architectures, optimization strategies, and broader scopes reviewed thus far indicate that substantial opportunities for research remain. While specific challenges have been noted throughout this survey, we wrap up by highlighting key areas where synergies between the SP and deep learning (DL) communities could lead to impactful advances.

\subsection{From SP to DL}
The SP community has developed a rich toolkit for leveraging low-rank structure. A number of SP tools can have high potential in the context of LLM fine-tuning too. 

\textbf{SP-transferred architectures.}
Architectures with provable efficiency in classical SP settings may offer significant merits when transferred to LLM adaptation.
One example we have already seen is the SVD-type parameterizations discussed earlier in Section~\ref{Sec.SVD-parameterization}. While they are equally expressive as the BM factorization, SVD parameterizations can exhibit more favorable optimization behavior on matrix sensing~\eqref{eq.prob-sensing}; that is, one-layer LoRA.
Moreover, several alternative parameterizations that are well-known in SP deserve exploration.
For instance,~\cite{ma2023over} shows that in the rank-one case, a tensorized lift of~\eqref{eq.prob-sensing} can lead to a more benign loss function landscape. In regimes where the restricted isometry property (RIP) fails and BM exhibits spurious local minima, the lifted formulation can turn these spurious minima into strict saddles. In the tensor regime, hierarchical Tucker decomposition~\cite{hackbusch2009new} is also known to improve the efficiency of standard Tucker by reducing the parameters in the core.
These examples raise an open question: do these architectures with established benefits in classical SP retain their advantages in the LLM scale? 

\textbf{SP-inspired optimization.} 
Alternating schemes such as ALS and splitting methods including ADMM have been successful in SP, yet they have seen limited adoption in LoRA fine-tuning. The major bottleneck may be that they require multiple backward passes at LLM scale. While a few recent works have begun to revisit such ideas in deployment-oriented settings \cite{ma2026salaad}, translating them into memory-efficient fine-tuning remains a challenge.
Another pressing direction is to better understand optimization for tensor-based parameterizations. Although tensorized adapters can be more compact than matrix-based ones, it is less clear whether they introduce implicit costs. Tensor methods in SP are known to have computational-versus-statistical efficiency tradeoffs~\cite{dong2023fast}. Whether analogous fundamental tradeoffs emerge with tensorized adapters in LLM fine-tuning remains a largely uncharted territory. 

\textbf{Understanding fine-tuning dynamics.}
Understanding learning dynamics can reveal additional levers for efficient and reliable fine-tuning, and several SP tools may be particularly useful in this regard. For example, it is well known that nuclear-norm regularization admits an equivalent BM variational form with Frobenius penalties~\cite{recht2010guaranteed,jang2024lora} through
\[
    \|\Delta \W \|_* = \min_{\X\Y^\top=\Delta \W} \frac{1}{2}\big(\|\X\|_\fro^2+\|\Y\|_\fro^2\big).
\]
This relationship suggests that weight decay on LoRA factors implicitly biases the induced update $\Delta \W$ toward a smaller nuclear norm. This offers a lens for interpreting why weight decay can stabilize training or improve generalization in practice. It remains to be seen how far this analogy carries beyond BM-style factorizations, e.g., to SVD-type and tensorized parameterizations, and whether comparable variational characterizations exist to explain their implicit biases.

\textbf{Sparsity with memory-efficient parameterization?}
Low-rank plus sparsity is a foundational modeling recipe in SP. However, current ``sparse'' LoRA variants typically fix the sparsity pattern in advance, as learning the sparse support end-to-end often necessitates applying an $\ell_1$ regularizer on a full matrix $\bfS_l \in \bR^{m \times n}$, which is difficult at LLM scale to reconcile with strict memory budgets. 
A key open problem is how to make the \emph{sparsity pattern} learnable without ever materializing a dense $\bfS_l\in\bR^{m\times n}$.

\textbf{Broader scopes.} 
Beyond optimization and parameterizations, information theory and coding theory can provide insightful guidance for quantization and compression, leading to more principled designs. Moreover, subspace methods such as CCA and ICA could offer a fresh perspective on multi-view and multi-modal adaptation. For example, CCA could be used to learn low-dimensional shared representations, while ICA could identify disentangled, modality-specific latent factors to minimize cross-modal interference during joint fine-tuning.

\subsection{From DL back to SP}
The innovations currently driving LLM fine-tuning are, in turn, proving useful for SP. Recent architectural and optimization paradigms, though motivated by LLM adaptation, are theoretically validated on classical SP testbeds such as matrix sensing and matrix completion. As a result, a number of these parameterizations and algorithmic principles can be transferred back to traditional low-rank recovery problems with minimal modification.
A salient example is the study of Hadamard decomposition~\cite{ciaperoni2024hadamard}, which is closely related to FedPara~\cite{hyeon2021fedpara} and HiRA~\cite{huang2025hira}. The goal is to approximate a matrix as the Hadamard product of low-rank matrices
\[
   \min_{\X_1, \X_2, \Y_1, \Y_2 } \| \X_1 \Y_1^\top \circ \X_2 \Y_2^\top - \W  \|_\fro^2.
\]
This parameterization can be more compact than the standard BM factorization. Gradient-based solvers for this problem are developed in~\cite{ciaperoni2024hadamard,wertz2025efficient}. Notably, the resulting optimization inherits additional symmetries beyond the usual factorization invariance within per-pair $(\X_i,\Y_i), i = \{1,2 \}$, and across Hadamard product, which further enriches the geometry and poses new challenges and opportunities for algorithm design.

Another compelling direction involves explicitly accounting for symmetries in classical SP problems. Recent work in DL has highlighted that making such symmetries explicit can yield simpler and more efficient algorithms, with examples already discussed in Section~\ref{Sec.gauge-invariance}. 
The rich factorization and recovery tasks in SP also exhibit rich—symmetry groups that the optimizer must implicitly navigate. Interestingly, for several SP problems even within the same gauge-equivalence class, not all parameterizations are ``equally good'' from an optimization standpoint. For instance,~\cite{ge2017no} demonstrates that balanced factorizations $\|\X\|_\fro \approx\|\Y\|_\fro$ enjoy more benign landscape properties, such as the absence of spurious local minima under suitable conditions. 
Moreover, flatter minima have been linked to better generalization in matrix recovery~\cite{ding2024flat}. These insights motivate designing algorithms that explicitly respect the underlying symmetries and, among equivalent parameterizations, preferentially steer iterates toward well-conditioned representatives that offer favorable behaviors.

\section{Conclusion}
In this survey, a comprehensive overview of low-rank adaptation (LoRA) for large models is provided  to connect it with well-established SP tools. By reframing LoRA as a modern rendition of classical low-rank models, this survey highlighted how architectural choices and optimization algorithms drive the success of LoRA. Our discussion has emphasized that the relationship between SP and LoRA is not an one-way transfer of knowledge. While SP offers a useful toolbox and rigorous analysis to improve and justify fine-tuning, the scale and complexity of foundation models introduce novel challenges that push the boundaries of classical SP. 

As the field moves toward more sophisticated practical settings, the synergy between these two disciplines will be vital. Future research should not only continue bringing SP insights to bear on large-scale model adaptation, but also leverage the innovations developed for LoRA to inspire novel approaches to fundamental SP problems, such as matrix sensing and tensor completion. Ultimately, this cross-fertilization will yield a more principled and efficient framework for artificial general intelligence. 

\bibliographystyle{IEEEtranS}
\bibliography{datactr}

\appendix

\section{Trash bin}
\subsection{Proof of Proposition~\ref{prop.Hadamard-express}}
\begin{proof}
First, we demonstrate that $\Delta \W$ can express any matrix of rank up to $r$. Let $\X^2 = [\mathbf{1}_m, \mathbf{0}, \dots, \mathbf{0}]$ and $\Y^2 = [\mathbf{1}_n, \mathbf{0}, \dots, \mathbf{0}]$, where $\mathbf{1}$ denotes the all-ones vector of the specified dimension. This yields $\Delta \W = (\X^1 \Y^{1\top}) \odot (\X^2 \Y^{2\top}) = \X^1 \Y^{1\top}$. This recovers the BM factorization used in vanilla LoRA, which spans all matrices of rank at most $r$.

Second, we evaluate the expressiveness for higher-rank matrices. Using Theorem~\ref{thm.Hadamard-rank}, it is clear that $\rank(\Delta \W) \le r^2$. Recall that the dimension of a rank-$r$ matrix manifold in $\mathbb{R}^{m \times n}$ is $r(m+n-r)$. As $\Delta \W$ is parameterized by two manifolds of rank at most $r$, it has no more than $2r(m+n-r)$ degrees of freedom (DoF). The dimension of a rank-$r^2$ manifold may exceed this DoF via
\begin{equation*}
r^2(m+n-r^2) > 2r(m+n-r). 
\end{equation*}
As a consequence, the parameterization cannot be surjective. That says, $\Delta \W$ can express only a subset of matrices of rank within $(r, r^2]$. 
\end{proof}

\subsection{Proof of Proposition~\ref{prop.HiRA-express}}
\begin{proof}
From Theorem~\ref{thm.Hadamard-rank}, the upper bound of rank is $r_{\max} := \min\{ \rank(\W) \cdot r,m,n \}$. To prove the proposition, it suffices to show that there exists certain $\W$, such that $\Delta \W = \W \odot (\X \Y^\top)$ can only cover a subset of matrices with rank in $(0, r]$ and $(r, r_{\max}]$. A simple example is $\W = \I$, which results in diagonal $\Delta \W$. 
\end{proof}

\subsection{Proof of Proposition~\ref{prop.Kron-express}}
\begin{proof}
From Thoerem~\ref{thm.Kron-rank}, the upper bound of $\rank(\A \otimes \B)$ is $r^2$. Next, we construct a counterexample to show that exists certain $\Delta \W$ which cannot be expressed as $\A \otimes \B$. Specifically, consider $d_1 = d_2 = d_3 = d_4 = 2$. It follows that
\[
\Delta \W = \begin{bmatrix}
    a_{11}\B & a_{12}\B \\
    a_{21}\B & a_{22}\B \\
    \end{bmatrix}.
\]
Hence the block structure cannot represent matrices such as
\[
\begin{bmatrix}
    1 & 0 & 0 & 0 \\
    0 & 0 & 0 & 0 \\
    0 & 0 & 0 & 1 \\
    0 & 0 & 0 & 0 \\
\end{bmatrix}
\]
which suggests the Kronecker product can only express a subset of matrices of rank up to $r^2$. 
\end{proof}

\subsection{Proof of Proposition~\ref{prop.low-rank+sparse}}
\begin{proof}
First, setting $\bfS = \bf0$ recovers LoRA, thus enabling parameterizing all matrices of rank no more than $r$. Second, using Theorem~\ref{thm.Bernoulli-rank} leads to the upper bound
\begin{align*}
    \rank(\Delta \W) &\le \min\{ \rank(\X\Y^\top) + \rank(\bfS), m, n \} \\
    &\le \min\{ m ,n \}
\end{align*}
which concludes the proof. 
\end{proof}

\subsection{ScaledGD derivation}
\label{Apdx.ScaledGD}

First, we show that ScaledGD's metric is gauge-invariant. Consider two gauge-equivalent points $(\X, \Y)$ and $(\tilde{\X}, \tilde{\Y}) = (\X\Q, \Y\Q^{-\top}),\,\Q \in \mathsf{GL}(r)$. Let $(\bm{\xi}_\X, \bm{\xi}_\Y), (\bm{\varphi}_\X,\bm{\varphi}_\Y) \in \mathcal{T}_{\X,\Y}$ be any two points on the tangent space at $(\X, \Y)$. By pushforward, the corresponding gauge-invariant points on the tangent space $\mathcal{T}_{\tilde{\X},\tilde{\Y}}$ are $(\tilde{\bm{\xi}}_\X, \tilde{\bm{\xi}}_\Y ) = (\bm{\xi}_\X \Q, \bm{\xi}_\Y \Q^{-\top})$ and $(\tilde{\bm{\varphi}}_\X, \tilde{\bm{\varphi}}_\Y) = (\bm{\varphi}_\X \Q, \bm{\varphi}_\Y \Q^{-\top})$. Using the definition of ScaledGD's metric, it follows that
\begin{align*}
    g_{(\tilde{\X},\tilde{\Y})}& \big((\tilde{\bm{\xi}}_\X,\tilde{\bm{\xi}}_\Y),(\tilde{\bm{\varphi}}_\X,\tilde{\bm{\varphi}}_\Y)\big) \\
    & = \langle \tilde{\bm{\xi}}_\X (\tilde{\Y}^\top \tilde{\Y}), \tilde{\bm{\varphi}}_\X \rangle_\fro + \langle \tilde{\bm{\xi}}_\Y (\tilde{\X}^\top \tilde{\X}), \tilde{\bm{\varphi}}_\Y \rangle_\fro \\
    &= \langle \bm{\xi}_\X \Q (\Q^{-1} \Y^\top \Y \Q^{-\top}), \bm{\varphi}_\X \Q \rangle_\fro + \\
    &\qquad \langle \bm{\xi}_\Y \Q^{-\top} (\Q^{\top} \X^\top \X \Q), \bm{\varphi}_\Y \Q^{-\top} \rangle_\fro \\
    &= \langle \bm{\xi}_\X (\Y^\top \Y ), \bm{\varphi}_\X \rangle_\fro + \langle \bm{\xi}_\Y (\X^\top \X), \bm{\varphi}_\Y \rangle_\fro \\
    &= g_{(\X,\Y)} \big((\bm{\xi}_\X, \bm{\xi}_\Y), (\bm{\varphi}_\X, \bm{\varphi}_\Y)\big)
\end{align*}
which proves the gauge-invariance of the metric. 

Next, we drive the Riemannian gradient under the metric. 
Let \(F(\X,\Y) := f(\X, \Y)\) be the lifted objective on the total space
\(\bR^{n\times r}\times \bR^{m\times r}\).
For conciseness, denote the Euclidean (Frobenius) gradients by
\[
\G_\X := \nabla_\X F(\X, \Y), \qquad \G_\Y := \nabla_\Y F(\X,\Y).
\]
Next, we derive the Riemannian gradient from the Riemannian metric
\begin{align*}
    g_{(\X,\Y)}& \big((\bm{\xi}_\X,\bm{\xi}_\Y),(\bm{\varphi}_\X,\bm{\varphi}_\Y)\big) \\
    & := \langle \bm{\xi}_\X (\Y^\top \Y), \bm{\varphi}_\X \rangle_\fro + \langle \bm{\xi}_\Y (\X^\top \X), \bm{\varphi}_\Y \rangle_\fro. 
\end{align*}

The Riemannian gradient \(\Rgrad F(\X,\Y) = (\Rgrad_\X F, \Rgrad_\Y F)\)
is defined as the unique tangent vector satisfying, for any tangent direction \((\bm{\varphi}_\X,\bm{\varphi}_\Y)\), that
\begin{align}\label{eq:riem-grad-def}
g_{(\X,\Y)}&\big((\Rgrad_\X F,\Rgrad_\Y F),(\bm{\varphi}_\X,\bm{\varphi}_\Y)\big) \nonumber \\
&= \langle \G_\X,\bm{\varphi}_\X \rangle_\fro + \langle \G_\Y,\bm{\varphi}_\Y \rangle_\fro. 
\end{align}

Substituting the definition of \(g_{(\X, \Y)}\) into~\eqref{eq:riem-grad-def} yields
\begin{align}\label{eq:riem-grad-plug}
\langle \Rgrad_\X F \, \Y^\top \Y, \bm{\varphi}_\X\rangle_\fro +& \langle \Rgrad_\Y F\, \X^\top \X, \bm{\varphi}_\Y \rangle_\fro \nonumber  \\
&= \langle \G_\X, \bm{\varphi}_\X \rangle_\fro + \langle \G_\Y, \bm{\varphi}_\Y \rangle_\fro. 
\end{align}

Since \(\ \bm{\varphi}_\X \) and \( \bm{\varphi}_\Y \) are arbitrary,~\eqref{eq:riem-grad-plug} implies
\[
\Rgrad_\X F\,(\Y^\top \Y) = \G_\X,
\qquad
\Rgrad_\Y F\,(\X^\top \X) = \G_\Y .
\]
As \( \X \) and \( \Y \) have full column rank so that \( \X^\top \X \) and \( \Y^\top \Y \) are invertible, it follows
\[
\boxed{
\Rgrad_\X F = \G_\X (\Y^\top \Y)^{-1},
\qquad
\Rgrad_\Y F = \G_\Y (\X^\top \X)^{-1}.
}
\]

\subsection{RefLoRA derivation}
\label{Apdx.reflora}
Similar to ScaledGD, we first illustrate that the metric of RefLoRA is gauge-invariant. Define geometric mean $\tilde{\bfS} := \tilde{\bfP}_\X^{-1/2} \Big( \tilde{\bfP}_\X^{1/2} \tilde{\bfP}_\Y \tilde{\bfP}_\X^{1/2} \Big)^{1/2} \tilde{\bfP}_\X^{-1/2}$ where $\tilde{\bfP}_\X := \tilde{\X}^\top \tilde{\X}$ and $\tilde{\bfP}_\Y := \tilde{\Y}^\top \tilde{\Y}$. 
Using $(\tilde{\X}, \tilde{\Y}) = (\X\Q, \Y\Q^{-\top}),\,\Q \in \mathsf{GL}(r)$, we have
\begin{align*}
    \big[ \Q^{-\top} (\tilde{\bfP}_\X \tilde{\bfP}_\Y)^{1/2} \Q^\top \big]^2
    &= \Q^{-\top} \tilde{\bfP}_\X \tilde{\bfP}_\Y \Q^\top \\
    &= \bfP_\X \bfP_\Y \\
    \Rightarrow (\tilde{\bfP}_\X \tilde{\bfP}_\Y)^{1/2} &= \Q^\top (\bfP_\X \bfP_\Y)^{1/2} \Q^{-\top}
\end{align*}
Then, it follows from~\cite[Lemma 7]{zhang2025reflora} that
\begin{align}
\label{eq.relation-S-Stilde}
    \tilde{\bfS} &= \tilde{\bfP}_\X^{-1} (\tilde{\bfP}_\X \tilde{\bfP}_\Y)^{1/2} \nonumber \\
    &= \tilde{\bfP}_\X^{-1} \Q^\top (\bfP_\X \bfP_\Y)^{1/2} \Q^{-\top} \nonumber \\
    &= \Q^{-1} \bfP_\X^{-1} (\bfP_\X \bfP_\Y)^{1/2} \Q^{-\top} \nonumber \\
    &= \Q^{-1} \bfS \Q^{-\top}.
\end{align}
Then by definition, the metric of RefLoRA satisfies
\begin{align*}
    g_{(\tilde{\X},\tilde{\Y})}& \big((\tilde{\bm{\xi}}_\X,\tilde{\bm{\xi}}_\Y),(\tilde{\bm{\varphi}}_\X,\tilde{\bm{\varphi}}_\Y)\big) \\
    & = \langle \tilde{\bm{\xi}}_\X \tilde{\bfS}, \tilde{\bm{\varphi}}_\X \rangle_\fro + \langle \tilde{\bm{\xi}}_\Y \tilde{\bfS}^{-1}, \tilde{\bm{\varphi}}_\Y \rangle_\fro \\
    &\overset{(a)}{=} \langle \bm{\xi}_\X \Q (\Q^{-1} \bfS \Q^{-\top}), \bm{\varphi}_\X \Q \rangle_\fro + \\
    &\qquad \langle \bm{\xi}_\Y \Q^{-\top} (\Q^\top \bfS^{-1} \Q), \bm{\varphi}_\Y \Q^{-\top} \rangle_\fro \\
    &= \langle \bm{\xi}_\X \bfS, \bm{\varphi}_\X \rangle_\fro + \langle \bm{\xi}_\Y \bfS^{-1}, \bm{\varphi}_\Y \rangle_\fro \\
    &= g_{(\X,\Y)} \big((\bm{\xi}_\X, \bm{\xi}_\Y), (\bm{\varphi}_\X, \bm{\varphi}_\Y)\big)
\end{align*}
where $(a)$ utilizes~\eqref{eq.relation-S-Stilde}. This thus proves the gauge-invariance of RefLoRA's metric. 

Next, we show that the Riemannian metric~\eqref{eq.reflora-metric} leads to RefLoRA updates~\eqref{eq.reflora}. By definition~\eqref{eq:riem-grad-def}, the Riemannian gradient $(\Rgrad_\X F, \Rgrad_\Y F)$ should satisfy
\begin{align*}
\langle \Rgrad_\X F \, \bfS, \bm{\xi}_\X\rangle_\fro +& \langle \Rgrad_\Y F\, \bfS^{-1}, \bm{\varphi}_\Y \rangle_\fro \nonumber  \\
&= \langle \G_\X, \bm{\varphi}_\X \rangle_\fro + \langle \G_\Y, \bm{\varphi}_\Y \rangle_\fro. 
\end{align*}
for any $(\bm{\varphi}_\X, \bm{\varphi}_\Y) \in \bR^{n\times r}\times \bR^{m\times r}$. 
The Riemannian gradient thus satisfies
\[
\Rgrad_\X F\,\bfS = \G_\X,
\qquad
\Rgrad_\Y F\,\bfS^{-1} = \G_\Y.
\]
It then follows that
\[
\boxed{
\Rgrad_\X F = \G_\X \bfS^{-1},
\qquad
\Rgrad_\Y F = \G_\Y \bfS.
}
\]

Next, we prove that RefLoRA's Riemmanian metric~\eqref{eq.reflora-metric} guarantees the update of $(\X, \Y)$ is always on the horizontal space, and has no gauge drift. 

First, the vertical space at $(\X, \Y)$ is $\mathcal{T}_{(\X, \Y)}^\mathrm{vert} = \{ (\X \U, -\Y \U^\top) \mid \U \in \bR^{r \times r} \}$, which can be verified via
\begin{equation*}
    (\X + \epsilon \X \U) (\Y - \epsilon \Y \U^\top)^\top = \X\Y^\top + \mathcal{O} (\epsilon^2). 
\end{equation*}
For update $(\Delta \X, \Delta \Y)$ to be purely horizontal, we must have
\begin{equation*}
    g_{(\X, \Y)} ((\X \U, -\Y \U^\top), (\Delta \X, \Delta \Y)) = 0,~\forall \U \in \bR^{r \times r}.
\end{equation*}
Plugging in the update $\Delta \X = -\eta \Rgrad_\X F = -\eta \G_\X \bfS^{-1}$ and $\Delta \Y = -\eta \Rgrad_\Y F = -\eta \G_\Y \bfS$ gives
\begin{align*}
    &g_{(\X, \Y)} ((\X \U, -\Y \U^\top), (\Delta \X, \Delta \Y)) \\
    &= \langle \X \U \bfS, -\eta \G_\X \bfS^{-1} \rangle_\mathrm{F} + \langle - \Y \U^\top \bfS^{-1}, -\eta \G_\Y \bfS \rangle_\mathrm{F} \\
    &\overset{(a)}{=} \langle \X \U \bfS, -\eta \G_\W \Y \bfS^{-1} \rangle_\mathrm{F} + \langle - \Y \U^\top \bfS^{-1}, -\eta \G_\W^\top \X \bfS \rangle_\mathrm{F} \\
    &= 0
\end{align*}
where $(a)$ comes from the chain rule with $\G_\W := \nabla_\W F(\X, \Y)$. This demonstrates that the gauge component vanishes, thus completing the proof.

\subsection{FastLoRA for batched serving}
\label{Apdx.fastlora}

Recall from~\eqref{eq.fast-lora} that the first term $\mathbf{Z}_l \W_l$ already allows batching. Therefore we only consider the implementation of the second term in~\eqref{eq.fast-lora}. For notational simplicity, the layer index $l$ is omitted. Consider the output of $k$-th adapter
\begin{align*}
    \mathbf{z}^k\big(\W \circ (\X^k\Y^{k\top}) \big)
    &= \mathbf{z}^k \left( \W \circ \sum_{j=1}^r \x^k_j \y_j^{k\top}\right) \\
    & = \sum_{j=1}^r \mathbf{z}^k\big(\W \circ \x_j^k \y_j^{k\top}\big) \\ 
    & = \sum_{j=1}^r \mathbf{z}^k\big(\diag(\x_j^k) \W \circ  \diag(\y_j^k)\big) \\
    & = \sum_{j=1}^r \big( \mathbf{z}^k \circ \x_j^k \big) \W \circ  \diag(\y_j^k).
\end{align*}
This is batching friendly, as one can simply stack all $k$ related vectors together.

Let $\mathbf{Z} =[\mathbf{z}^{1\top}, \ldots, \mathbf{z}^{K\top}]^\top \in\bR^{K\times m}$, and
for each $j\in[r]$ stack the $j$-th columns across users as
\begin{align*}
    \tilde{\X}^{j}=\begin{bmatrix}\x_j^{1\top}\\ \vdots\\ \x_j^{K\top}\end{bmatrix}\in\bR^{K\times m},
    \qquad
    \tilde{\Y}^{j}=\begin{bmatrix}\y_j^{1\top}\\ \vdots\\ \y_j^{K\top}\end{bmatrix}\in\bR^{K\times n}.
\end{align*}
Then the $K$ outputs can be efficiently computed in batched form via
\begin{align*}
    \begin{bmatrix}
        \mathbf{z}^1(\W \circ \X^1\Y^{1\top})\\
        \vdots\\
        \mathbf{z}^K(\W \circ \X^r \Y^{r\top})
    \end{bmatrix}
    &= \sum_{j=1}^r \Big( \big(( \mathbf{Z}\circ \tilde{\X}^{j})\W \big)\circ \tilde{\Y}^{j} \Big). 
\end{align*}

\end{document}